\documentclass[a4paper]{article}

\usepackage{mystylefile}
\usepackage{natbib}
\usepackage{subfigure}
\usepackage{times}
\usepackage{color}
\usepackage{multirow}
\usepackage{algpseudocode}
\usepackage{bm}
\usepackage{xr}
\usepackage{url}

\usepackage{epstopdf}
\usepackage[np,autolanguage]{numprint}
\pdfoutput=1

\newcommand{\rkhs}{\mathcal{H}}
\newcommand{\HS}{\mathrm{HS}}
\newcommand{\intd}{\mathrm{d}}

\newcommand{\argmin}{\mathop{\rm argmin}\limits}
\newcommand{\Sig}{\mathbf{\Sigma}}

\newcommand{\Haus}{\mathrm{Haus}}
\newcommand{\KDE}{\mathrm{KDE}}

\newcommand{\LS}{\mathrm{LS}}
\newcommand{\X}{\mathcal{X}}
\newcommand{\esssup}{\mathop{\rm ess~sup}}

\def\g{{\bm g}}
\def\x{{\bm x}}
\def\y{{\bm y}}
\def\z{{\bm z}}
\def\<{\langle}
\def\>{\rangle}

\newcommand{\BlackBox}{\rule{1.5ex}{1.5ex}}  
\newenvironment{proof}{\par\noindent{\bf Proof\ }}{\hfill\BlackBox\\[2mm]}
 
\newtheorem{theorem}{Theorem}
\newtheorem{lemma}[theorem]{Lemma} 
\newtheorem{proposition}[theorem]{Proposition} 

\newtheorem{remark*}[theorem]{Remark}

\title{Mode-Seeking Clustering and Density Ridge Estimation\\ via Direct
Estimation of Density-Derivative-Ratios}
\author{
Hiroaki Sasaki \\ hsasaki@is.naist.jp \\
Graduate School of Information Science, \\
Nara Institute of Science and Technology, Nara, Japan \\\vspace{-1mm} \\
Takafumi Kanamori \\ kanamori@c.titech.ac.jp \\
Department of Mathematical and Computing Science, \\
Tokyo Institute of Technology, Tokyo, Japan\\
Center for Advanced Intelligence Project, \\
RIKEN, Tokyo, Japan \\\vspace{-1mm} \\
Aapo Hyv{\"a}rinen \\ a.hyvarinen@ucl.ac.uk \\
Gatsby Computational Neuroscience Unit, \\
University College London, London, United Kingdom \\
Department of Computer Science, \\ 
University of Helsinki, Helsinki, Finland\\
Canadian Institute for Advanced Research\\ \vspace{-1mm} \\
Gang Niu \\ 
gang@ms.k.u-tokyo.ac.jp \\
Graduate School of Frontier Sciences, \\ The University of Tokyo,
Chiba, Japan \\ Center for Advanced Intelligence Project, \\
RIKEN, Tokyo, Japan \\\vspace{-1mm} \\
Masashi Sugiyama \\ 
sugi@k.u-tokyo.ac.jp \\
Center for Advanced Intelligence Project, \\
RIKEN, Tokyo, Japan \\
Graduate School of Frontier Sciences, \\ The University of Tokyo,
Chiba, Japan
}
\date{}
\begin{document}
\maketitle
\begin{abstract}
 Modes and ridges of the probability density function behind observed
 data are useful geometric features. Mode-seeking clustering assigns
 cluster labels by associating data samples with the nearest modes, and
 estimation of density ridges enables us to find lower-dimensional
 structures hidden in data. A key technical challenge both in
 mode-seeking clustering and density ridge estimation is accurate
 estimation of the ratios of the first- and second-order density
 derivatives to the density. A naive approach takes a three-step
 approach of first estimating the data density, then computing its
 derivatives, and finally taking their ratios. However, this three-step
 approach can be unreliable because a good density estimator does not
 necessarily mean a good density derivative estimator, and division by
 the estimated density could significantly magnify the estimation
 error. To cope with these problems, we propose a novel estimator for
 the \emph{density-derivative-ratios}. The proposed estimator does not
 involve density estimation, but rather \emph{directly} approximates the
 ratios of density derivatives of any order. Moreover, we establish a
 convergence rate of the proposed estimator. Based on the proposed
 estimator, novel methods both for mode-seeking clustering and density
 ridge estimation are developed, and the respective convergence rates to
 the mode and ridge of the underlying density are also
 established. Finally, we experimentally demonstrate that the developed
 methods significantly outperform existing methods, particularly for
 relatively high-dimensional data.
\end{abstract}
 \section{Introduction} 
 Characterizing the probability density function underlying observed
 data is a fundamental problem in machine learning. One approach is to
 consider geometric properties of the density such as modes and
 ridges. Estimation of such geometric properties is a challenging task,
 yet offers a variety of applications~\citep{wasserman2018topological}.
  
 The \emph{modes} (i.e., local maxima) of probability density functions
 have received much attention over the years. A motivation of estimating
 the modes classically appeared in the seminal work on kernel density
 estimation~\citep{parzen1962estimation}.  More recently, the modes of
 density functions for random curves have been used in functional data
 analysis~\citep{gasser1998nonparametric}. Furthermore, in supervised
 learning, modal regression associates input variables with the modes of
 the conditional density function of the output variable, and enables us
 to simultaneously capture multiple functional relationships between the
 input and
 output~\citep{sager1982maximum,carreira2000reconstruction,carreira2001continuous,einbeck2006modelling,chen2016nonparametric,sasaki2016modal}.
 One of the most natural applications is clustering. \emph{Mean shift
 clustering} (MS) makes use of the modes of the estimated density
 function~\citep{fukunaga1975estimation,cheng1995mean,comaniciu2002mean}:
 MS initially regards all data samples as candidates for cluster
 centers, and then iteratively updates them toward the nearest modes of
 the estimated density by gradient ascent (Fig.\ref{fig:MS}). Finally,
 the data samples which converge to the same mode are assigned the same
 cluster label. Unlike standard clustering methods such as k-means
 clustering~\citep{BerkeleySymp:MacQueen:1967} and mixture-model-based
 clustering~\citep{melnykov2010finite}, the notable advantage is that
 the number of clusters is automatically determined according to the
 number of detected modes. MS has been applied to a wide range of tasks
 such as image
 segmentation~\citep{comaniciu2002mean,tao2007color,wang2004image} and
 object tracking~\citep{collins2003mean,comaniciu2000real}. (See also a
 recent review article by~\citet{carreira2015review})

 \begin{figure}[t]
  \begin{center}
   \subfigure[Iteration=0]{\includegraphics[width=0.245\textwidth,clip]{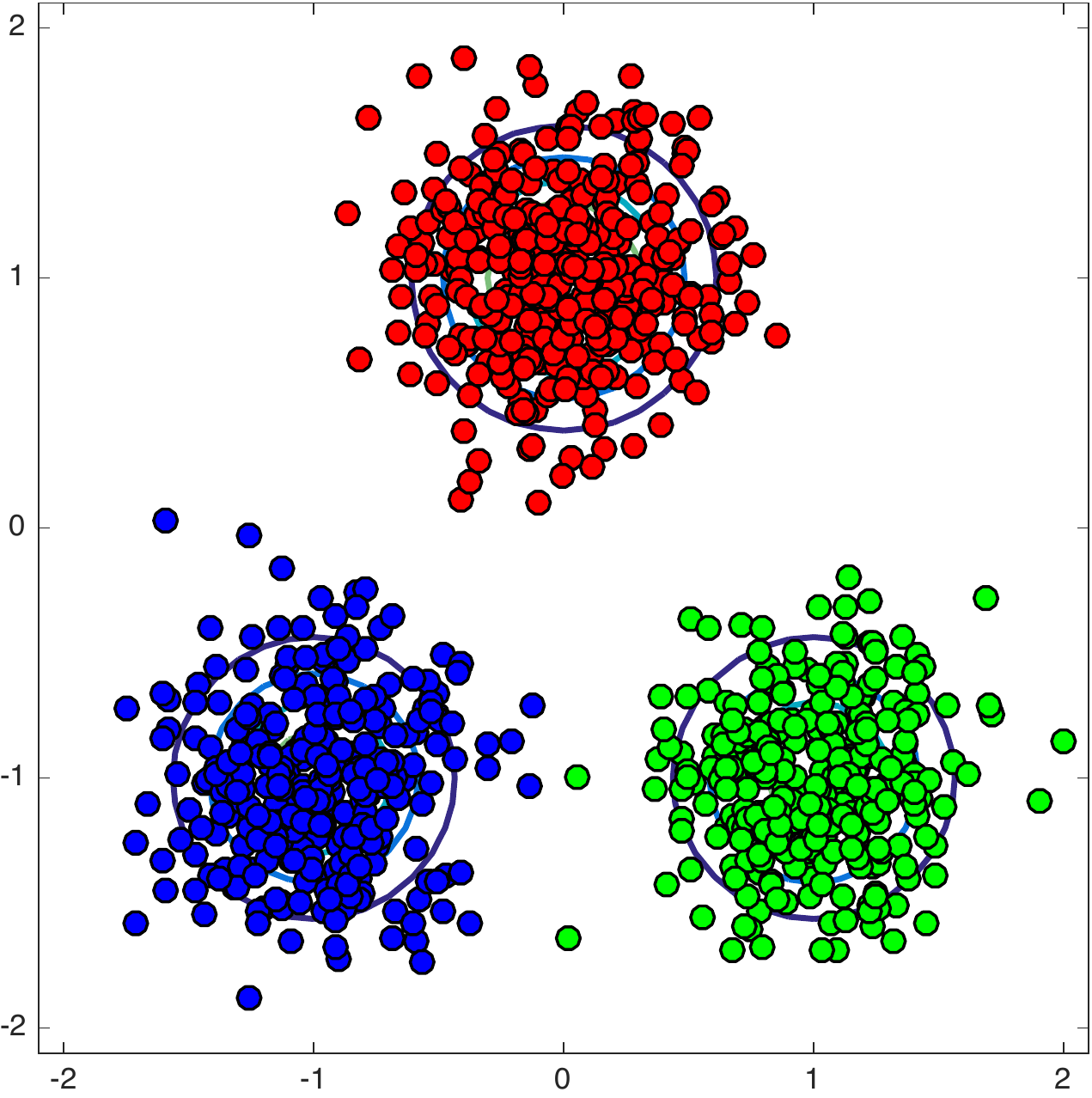}}
   \subfigure[Iteration=1]{\includegraphics[width=0.245\textwidth,clip]{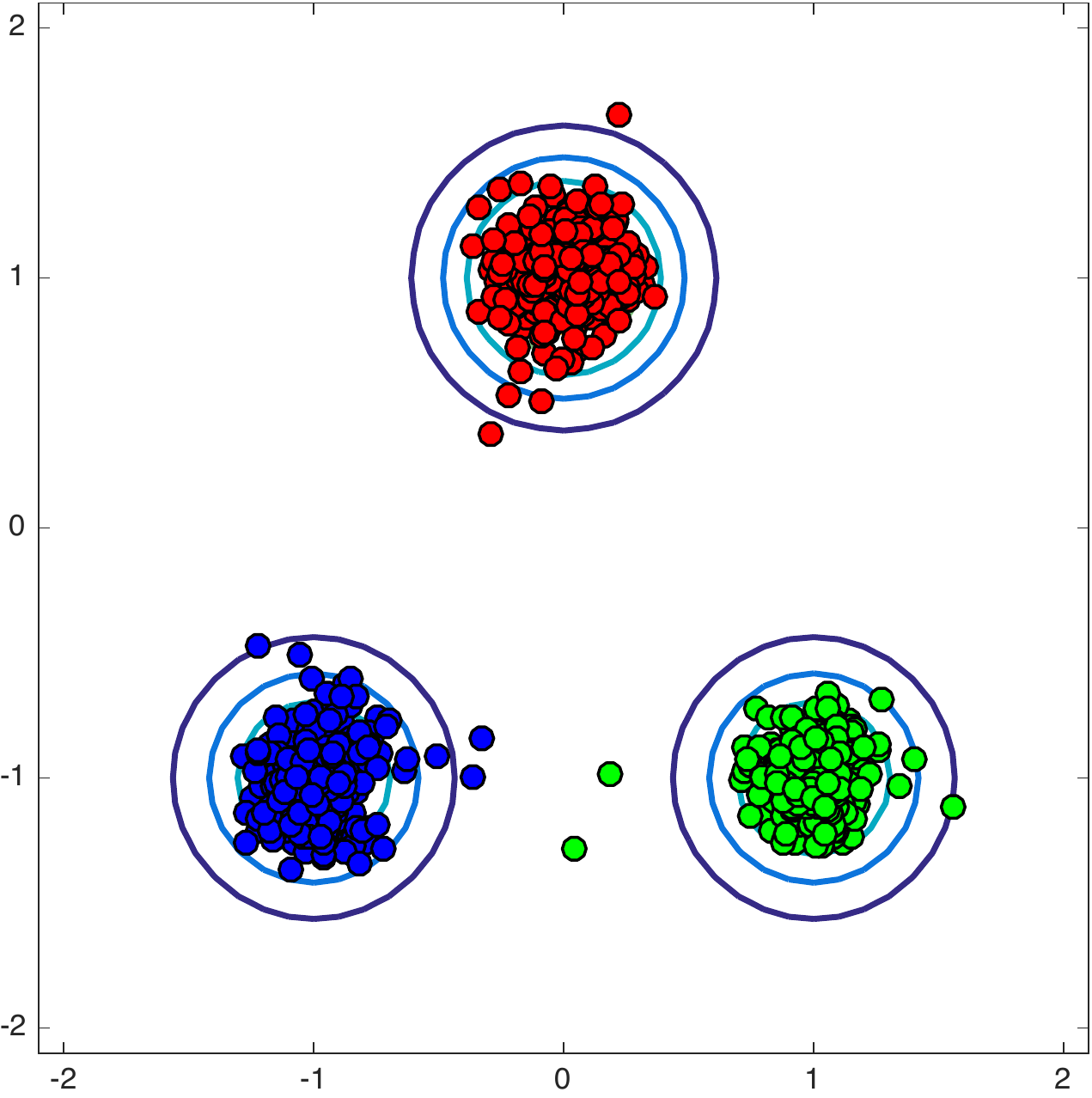}}
   \subfigure[Iteration=3]{\includegraphics[width=0.245\textwidth,clip]{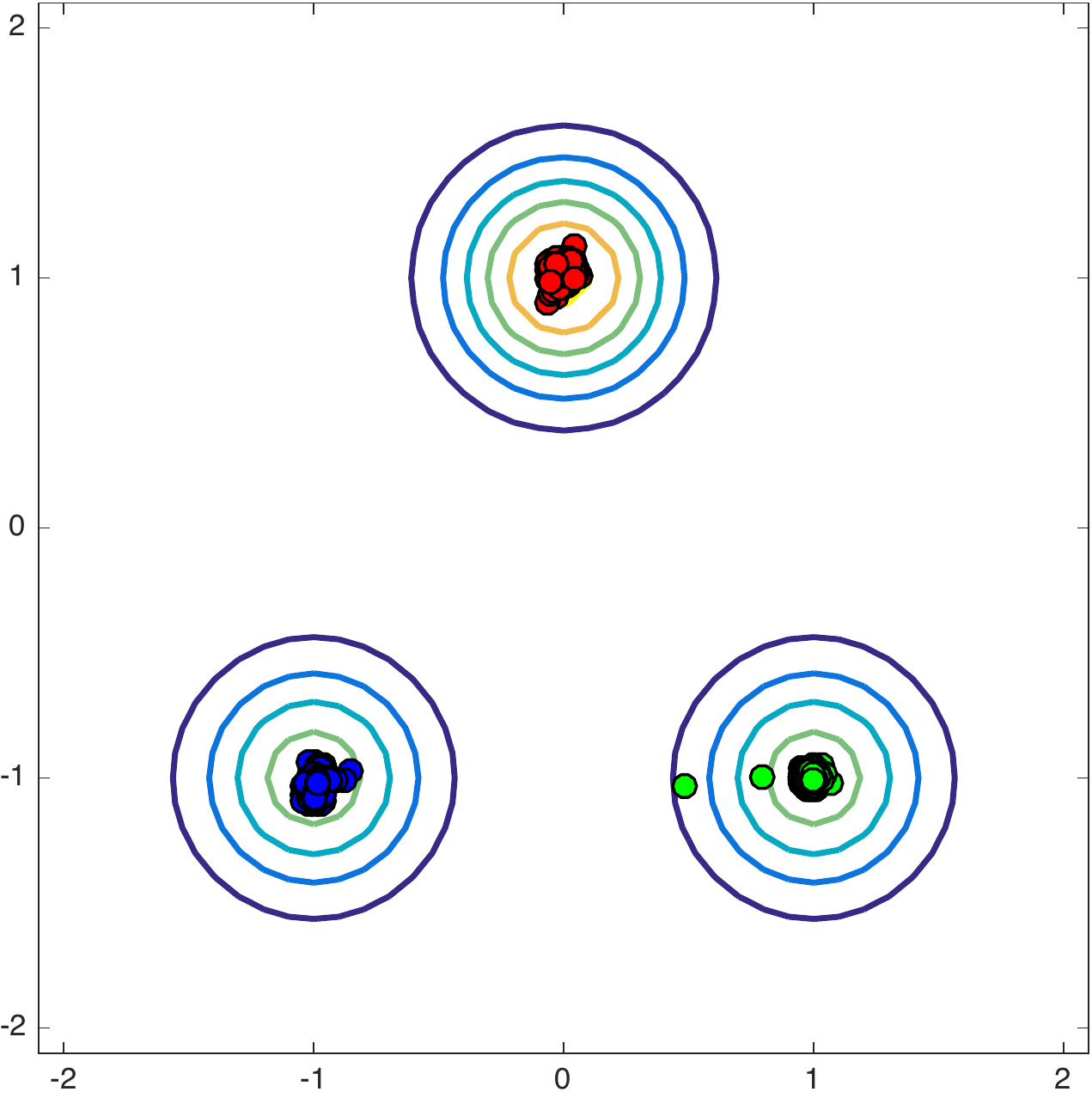}}
   \subfigure[Iteration=10]{\includegraphics[width=0.245\textwidth,clip]{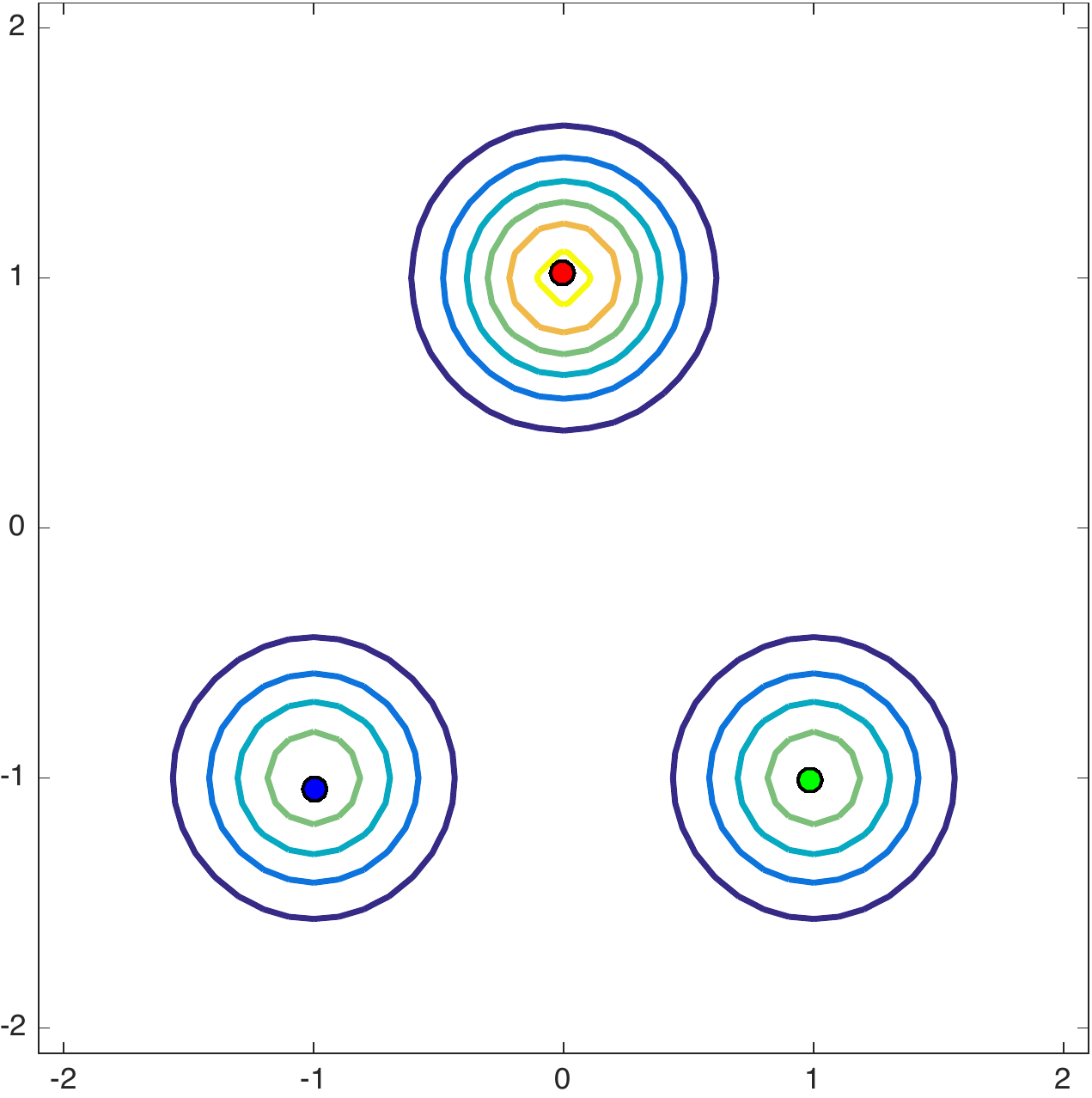}}
   \caption{\label{fig:MS} Illustration of a mode-seeking process. The
   contour plot indicates the probability density function that
   generates the data samples.}
   \vspace{1mm}
   \subfigure{\includegraphics[width=0.245\textwidth,clip]{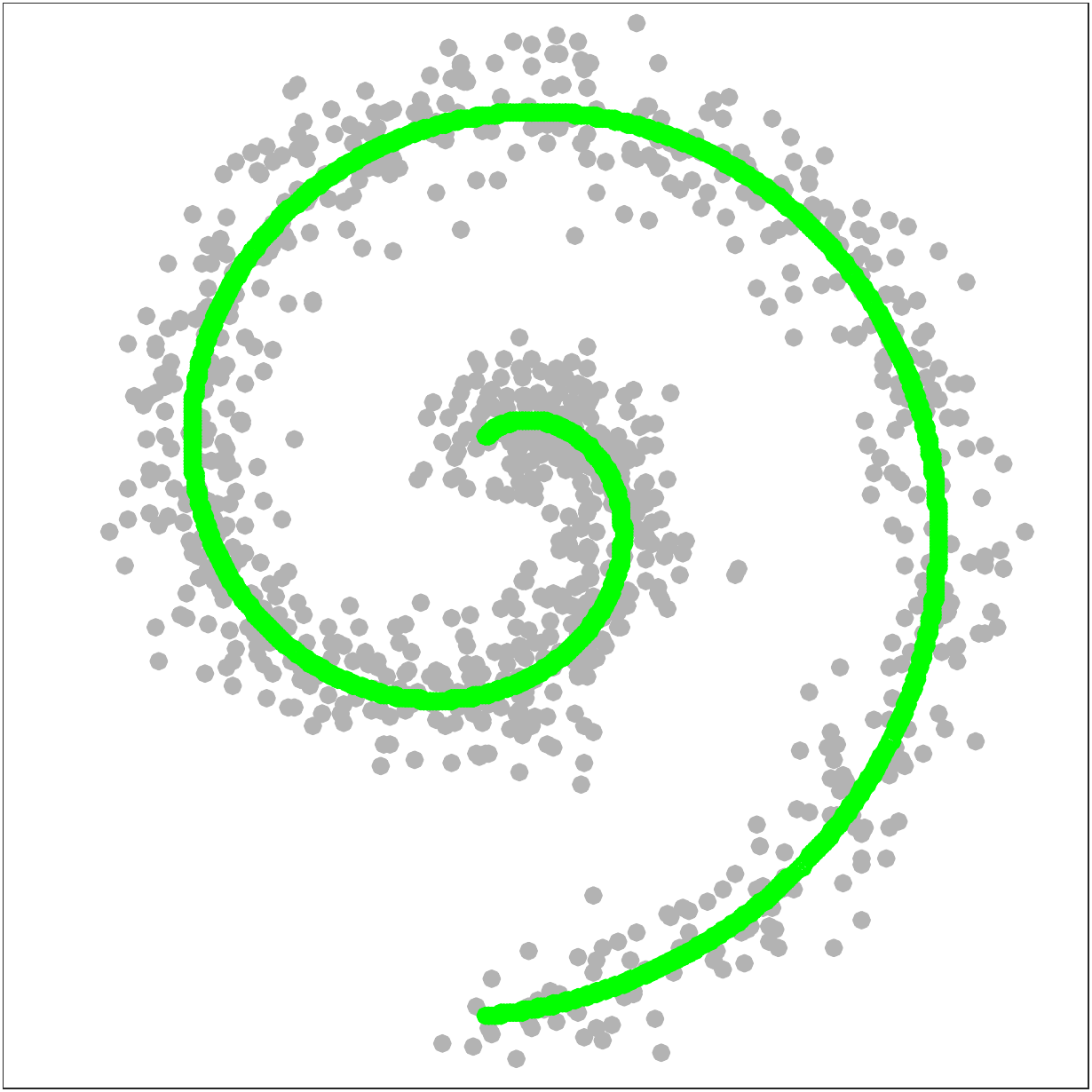}}
   \subfigure{\includegraphics[width=0.245\textwidth,clip]{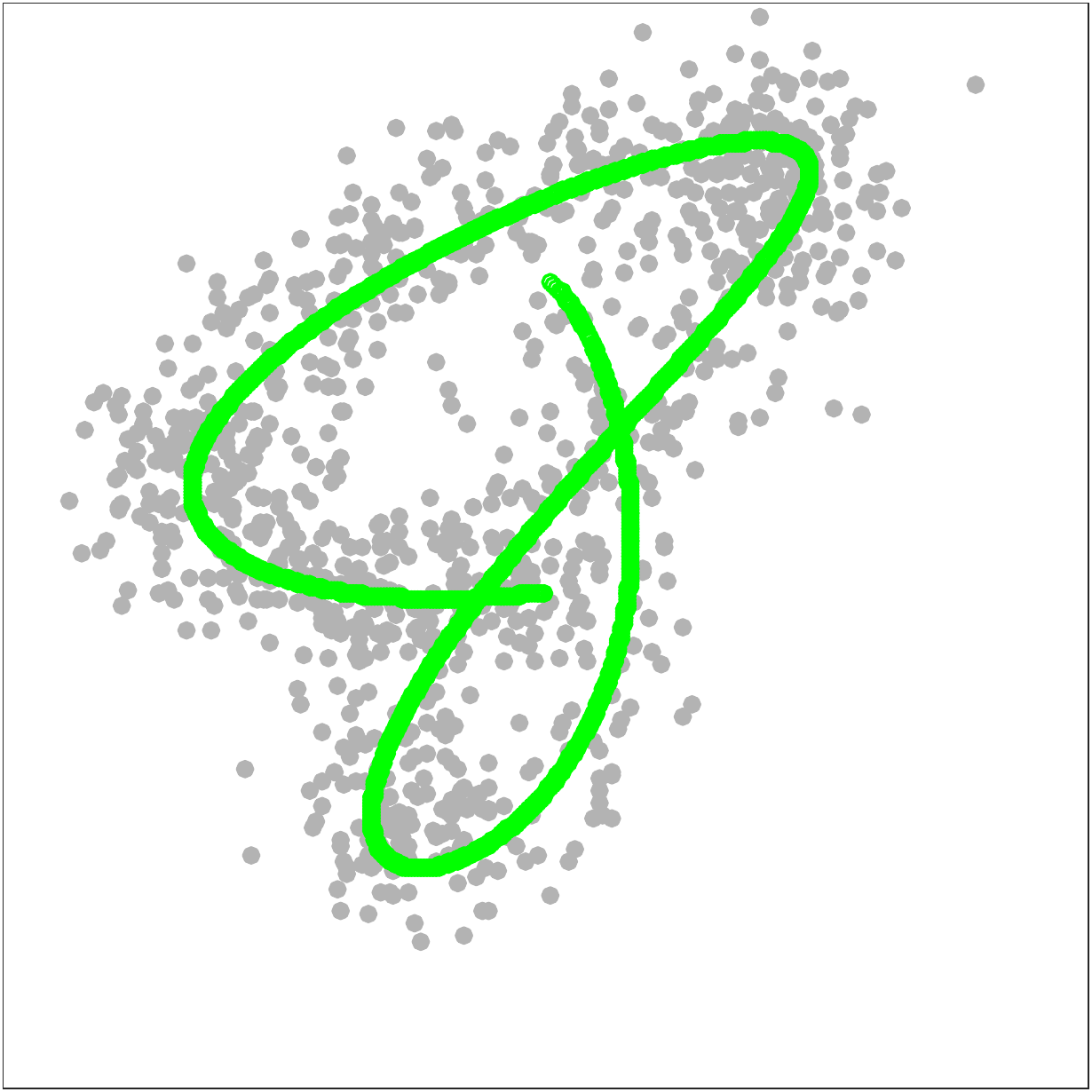}}
   \subfigure{\includegraphics[width=0.245\textwidth,clip]{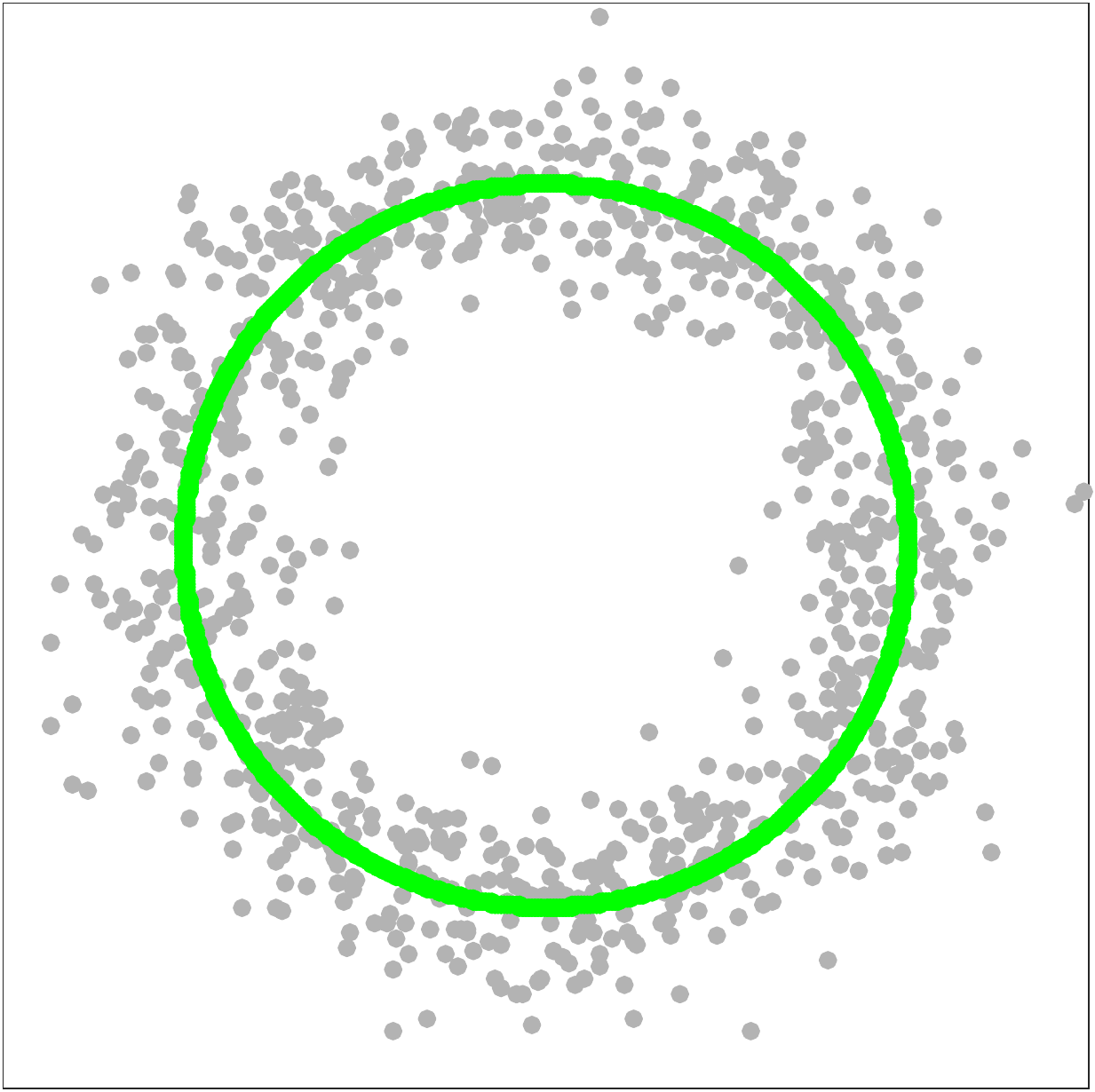}}
   \subfigure{\includegraphics[width=0.245\textwidth,clip]{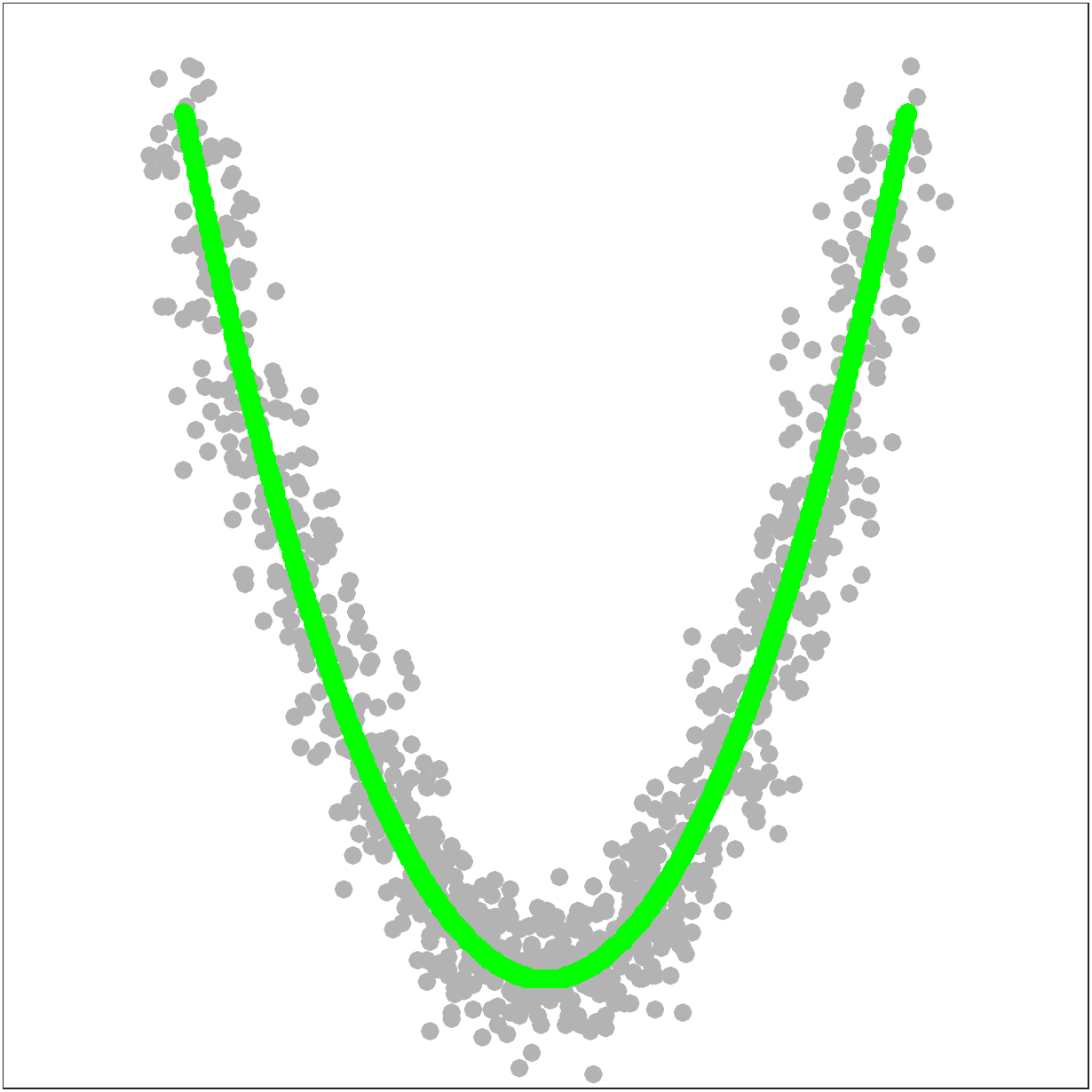}}
   \vspace{-2mm}
   \caption{\label{fig:Ridge} Examples of the density ridges hidden in
   data. Gray dot points and green curves indicate data samples and
   density ridges, respectively.}
  \end{center}
 \end{figure}

 A \emph{ridge} of the probability density function generalizes the
 notion of the mode. The density ridge is a lower-dimensional hidden
 structure of the data (Fig.\ref{fig:Ridge}), and the zero-dimensional
 ridge can be interpreted as the
 mode~\citep{genovese2014nonparametric}. Application of density ridge
 estimation can be found in a variety of fields such as filamentary
 structure estimation in cosmology~\citep{chen2016cosmic}, extraction of
 curvilinear structures (e.g., blood vessels in the eyes) in medical
 imaging~\citep{you2011principal}, and shape analysis in computer
 vision~\citep{su2013detection} (See~\citet{pulkkinen2015ridge} for more
 applications). Density ridge estimation is closely related to manifold
 estimation. When data is assumed to be generated on a lower-dimensional
 manifold with additive Gaussian noise, density ridge estimation offers
 a way to circumvent the difficulty of manifold estimation:
 \citet{genovese2014nonparametric} theoretically proved that the density
 ridges capture the essential properties of such manifolds and
 estimating the density ridge is substantially easier than estimating
 the manifold. A practical algorithm called \emph{subspace constrained
 mean shift} (SCMS) was proposed by~\citet{ozertem2011locally}. SCMS is
 an extension to MS, but a projected gradient ascent method is performed
 to find density ridges instead of the gradient ascent method in MS; the
 gradient vector of the estimated density is projected to the subspace
 which is orthogonal to the ridge. Such a subspace can be obtained by
 applying principal component analysis to an estimate of the Hessian
 matrix of the log-density, which is composed of the ratios of the
 first- and second-order density derivatives to the density. Along the
 projected gradient vector, SCMS updates data points toward the ridge of
 the estimated density until convergence.

 For MS, the technical challenge is accurate estimation of the
 derivatives of the probability density function. To derive practical
 methods, MS takes a two-step approach, firstly estimating the
 probability density function and then computing its
 derivatives~\citep[Section~2]{comaniciu2002mean}.\footnote{As reviewed
 in Section~\ref{ssec:meanshift}, practical methods themselves do not
 perform initial density estimation.} However, this approach can be
 unreliable because a good density estimator does not necessarily imply
 a good density derivative estimator in many practical situations. For
 example, small random fluctuations in a density estimate can create
 fake modes and may produce large errors in density-derivative
 estimation, even if the density estimate is fairly good in terms of
 density estimation~\citep[Fig.1]{genovese2016non}. Therefore, testing
 methods have been proposed to investigate whether the estimated modes
 are real modes from the underlying data density or fake modes due to
 the random
 fluctuations~\citep{godtliebsen2002significance,duong2008feature,genovese2016non}.
 For SCMS, it is even more challenging to estimate the ratios of density
 derivatives to the density, but SCMS also naively estimates the ratios
 by adding one more step to the two-step approach in MS: the computed
 density derivatives are divided by the estimated density. However, such
 a division could strongly magnify estimation error.
 
 To cope with these problems, we propose a novel estimator of the ratios
 of density derivatives to the density. In stark contrast with the
 approaches in MS and SCMS, the key idea is to \emph{directly} estimate
 the ratios without going through density estimation. Moreover, we
 theoretically analyze the proposed estimator and establish a
 convergence rate. The direct approach has been adopted and proved to be
 useful both empirically and theoretically when estimating the ratio of
 two probability density
 functions~\citep{sugiyama2008direct,nguyen2008estimating,kanamori2009least,kanamori2012statistical,sugiyama2012density-ratio,pmlr-v54-kpotufe17a}.
 Here, we follow the direct approach in the context of a different
 problem and derive an estimator in a substantially different
 way. Previously, a direct estimator has been proposed for the
 log-density derivatives~\citep{beran1976adaptive,cox1985penalty}, which
 are the ratios of first-order density derivatives to the density. On
 the other hand, the proposed estimator in this paper approximates the
 ratio of the derivatives of any order to the density, and thus
 generalizes the previous estimator.
 
 The proposed estimator is first applied to mode-seeking clustering. We
 derive an update rule for mode-seeking based on a fixed-point
 algorithm, while inheriting the advantage of MS: the proposed
 clustering method also does not require the number of clusters to be
 specified in advance. This is advantageous because clustering is an
 unsupervised learning problem and tuning the number of clusters is not
 straightforward in general. Next, based on the mode-seeking clustering,
 we propose a novel method for density ridge estimation. For both
 methods, we prove the consistency of the mode and ridge estimators, and
 establish the convergence rates. Finally, we experimentally demonstrate
 that our proposed methods outperform MS and SCMS, particularly for
 high(er)-dimensional data.
 
 This paper is organized as follows: In Section~\ref{sec:LSDDR}, we
 propose a novel estimator for the ratio of the derivatives of any order
 to the density, and establish a non-parametric convergence rate.  The
 proposed estimator is applied to develop novel methods for mode-seeking
 clustering and density ridge estimation in
 Sections~\ref{modeclustering} and~\ref{sec:ridge} respectively, and
 both methods are theoretically analyzed. Section~\ref{sec:illust}
 experimentally investigates the performance of the proposed methods for
 mode-seeking clustering and density ridge
 estimation. Section~\ref{sec:conclusion} concludes this
 paper. Preliminary results of this paper were presented at
 ECML/PKDD~2014~\citep{SasakiHS14clustering} and
 AISTATS~2017~\citep{SasakiEstimating2017}. However, in addition to
 combining the results in those conference papers, we have added new
 theoretical analysis of the proposed estimator, mode-seeking clustering
 and density ridge estimation methods. From a theoretical stand point,
 we further improved upon the methods appeared in the conference papers,
 and performed more experiments in this paper.
 \section{Direct Estimation of Density-Derivative-Ratios}
 \label{sec:LSDDR}
 This section proposes a novel estimator of the ratios of density
 derivatives to the density and performs theoretical analysis.
  \subsection{Problem Formulation}
  Suppose that $n$ i.i.d.~samples, which were drawn from a probability
  distribution on $\R{D}$ with density $p(\bm{x})$, are available:
  \begin{align*}
   \mathcal{D}:=\{\bm{x}_i
   =(x_i^{(1)},x_i^{(2)},\dots,x_i^{(D)})^{\top}\}_{i=1}^n
   \overset{\text{i.i.d.}}{\sim} p(\bm{x}).
  \end{align*}
  Here, our goal is to estimate the ratio of the $|\bm{j}|$-th order
  partial derivative of $p(\bm{x})$ to $p(\bm{x})$ from
  $\mathcal{D}=\{\bm{x}_i\}_{i=1}^n$,
  \begin{align}
   \frac{\partial_{\bm{j}} p(\bm{x})}{p(\bm{x})},
   \label{defi:der}
  \end{align}
  where $\partial_{\bm{j}}=\frac{\partial^{|\bm{j}|}}{\partial^{j_1}
  x^{(1)}\partial^{j_2}x^{(2)}\dots\partial^{j_D} x^{(D)}}$,
  $\bm{j}=(j_1,j_2,\dots,j_D)^{\top}$ and $|\bm{j}|=j_1+j_2+\dots+j_D$
  for non-negative integers $j_i=0,1,\dots,|\bm{j}|$. For instance, when
  $|\bm{j}|=1$ (or $|\bm{j}|=2$), $\partial_{\bm{j}}
  p(\bm{x})/p(\bm{x})$ is a single element of $\nabla
  p(\bm{x})/p(\bm{x})$ (or of $\nabla\nabla p(\bm{x})/p(\bm{x})$).
  \subsection{Least-Squares Density-Derivative-Ratios}
  \label{ssec:LSDDR}
  Our main idea is to directly fit a model $r_{\bm{j}}(\bm{x})$ to
  $\partial_{\bm{j}} p(\bm{x})/p(\bm{x})$ under the squared-loss:
  \begin{align}
   J_{\bm{j}}(r_{\bm{j}})&:=\int \left\{r_{\bm{j}}(\bm{x})
   -\frac{\partial_{\bm{j}} p(\bm{x})}{p(\bm{x})}\right\}^2
   p(\bm{x})\mathrm{d}\bm{x}\nonumber\\
   &=\int\left\{r_{\bm{j}}(\bm{x})\right\}^2 p(\bm{x})\mathrm{d}\bm{x}
   -2\int r_{\bm{j}}(\bm{x})\partial_{\bm{j}} p(\bm{x})\mathrm{d}\bm{x}
   +\int \left\{\frac{\partial_{\bm{j}}
   p(\bm{x})}{p(\bm{x})}\right\}^2 p(\bm{x})\mathrm{d}\bm{x}.
   \label{J}
  \end{align}
  The first term on the right-hand side of~\eqref{J} can be naively
  estimated from samples and the third term is ignorable, but it seems
  challenging to estimate the second term because it includes the
  derivative of the unknown density. However, as
  in~\citet{sasaki2015density}, repeatedly applying \emph{integration by
  parts} allows us to transform the second term as
  \begin{align}
   \int r_{\bm{j}}(\bm{x})
   \left\{\partial_{\bm{j}}p(\bm{x})\right\}\mathrm{d}\bm{x}
   &=(-1)^{|\bm{j}|} \int \left\{\partial_{\bm{j}}
   r_{\bm{j}}(\bm{x})\right\} p(\bm{x})\mathrm{d}\bm{x},
   \label{secondterm}
  \end{align}
  where we assumed that as $|x^{(j)}|\rightarrow\infty$ for all $j$, the
  product of $\partial_{\bm{j}_1}r_{\bm{j}}(\bm{x})$ and
  $\partial_{\bm{j}_2}p(\bm{x})$ approaches zero for any pairs of
  $\bm{j}_1$ and $\bm{j}_2$ satisfying
  $|\bm{j}_1|+|\bm{j}_2|=|\bm{j}|-1$ for
  $|\bm{j}_1|,|\bm{j}_2|=0,1,\dots,|\bm{j}|-1$. As a result, the
  right-hand side of~\eqref{secondterm} can be easily estimated from
  samples. Then, an empirical version of~\eqref{J} is given by
  \begin{align}
   \widehat{J}_{\bm{j}}(r_{\bm{j}}):=\frac{1}{n}\sum_{i=1}^n
   \left\{r_{\bm{j}}(\bm{x}_i)^2 -2(-1)^{|\bm{j}|} \partial_{\bm{j}}
   r_{\bm{j}}(\bm{x}_i)\right\}+\text{const.} \label{Jsample}
  \end{align}
  After adding the regularizer $R(r_{\bm{j}})$, the estimator is defined
  as the minimizer of
  \begin{align}
   \widehat{r}_{\bm{j}}:=\argmin_{r_{\bm{j}}} \left[
   \widehat{J}_{\bm{j}}(r_{\bm{j}})+\lambda_{\bm{j}}R(r_{\bm{j}})\right],
   \label{eqn:LSDDR}
  \end{align}
  where $\lambda_{\bm{j}}$ is the regularization parameter.
  
  We call this method the \emph{least-squares density-derivative ratios}
  (LSDDR). Note that when $|\bm{j}|=1$, $J_{\bm{j}}$ is called the
  Fisher divergence and has been used for parameter estimation of
  unnnormalized statistical models~\citep{hyvarinen2005estimation},
  density estimation with the computationally intractable partition
  function~\citep{JMLR:v18:16-011}, and direct estimation of log-density
  derivatives~\citep{beran1976adaptive,cox1985penalty,SasakiHS14clustering}.
  Therefore, LSDDR can be regarded as a generalization of such methods
  to higher-order derivatives.
  \subsection{Theoretical Analysis of LSDDR}
  Next, we theoretically analyze LSDDR.
   \subsubsection{Preliminaries and Notations}
   For a $D$-dimensional vector $\bm{x}\in\R{D}$, the norm is defined by
   $\|\bm{x}\|:=\sqrt{\sum_{j=1}^D(x^{(j)})^2}$. For a domain
   $\X(\subseteq\R{D})$, $C(\X)$ denotes the space of all continuous
  functions on $\X$. Furthermore, we define the \emph{$L^p$} space of
  functions $f$ on $\X$: For $1\leq p\leq\infty$,
  $L^p(\X):=\{f~:~\|f\|_{p}<\infty\}$ where $\|\cdot\|_p$ is the
  \emph{$L^p$} norm defined by
  $\|f\|_p:=\left(\int|f(\bm{x})|^p\intd\bm{x}\right)^{\frac{1}{p}}$
  with the Lebesgue measure for $1\leq p<\infty$ and
  $\|f\|_{\infty}:=\esssup_{\bm{x}\in\X}|f(\bm{x})|$.  For $f\in
  L^1(\R{D})$, the \emph{Fourier transform} is defined as
   \begin{align*}
    f^{\wedge}(\bm{\omega}):=\frac{1}{(2\pi)^{D/2}}
    \int f(\bm{x})e^{-i\bm{\omega}^{\top}\bm{x}}\intd\bm{x}, 
   \end{align*}
   where $i$ denotes the imaginary unit.
   
   Let $\rkhs$ be a reproducing kernel Hilbert space (RKHS) over
   $\mathcal{X}$ uniquely associated with the reproducing kernel
   $k:\mathcal{X}\times\mathcal{X}\rightarrow\R{}$. The norm and inner
   product on $\rkhs$ are denoted by $\|\cdot\|_{\rkhs}$ and
   $\<\cdot,\cdot\>_{\rkhs}$, respectively. $k$ is a real-valued,
   symmetric and positive definite function and has the reproducing
   property: For all $\bm{x}\in\X$ and $f\in\rkhs$,
   $\<f,k(\cdot,\bm{x})\>_{\rkhs}=f(\bm{x})$. An example of reproducing
   kernels is the \emph{Gaussian kernel},
   $k(\bm{x},\bm{y})=\exp\left(-\frac{\|\bm{x}-\bm{y}\|^2}{2\sigma^2}\right)$
   where $\sigma>0$ is the width parameter. Another example is the
   \emph{Mat{\'e}rn kernel},
   $k(\bm{x},\bm{y})=\psi(\bm{x}-\bm{y})=\frac{2^{1-s}}{\Gamma(s)}
   \|\bm{x}-\bm{y}\|^{s-D/2}\mathfrak{K}_{D/2-s}(\|\bm{x}-\bm{y}\|)$,
   whose corresponding RKHS $\rkhs$ coincides with the Sobolev space
   $H^{s}_2$ with the smoothness parameter
   $s>D/2$~\citep[Chapter~10]{wendland2004scattered}:
   \begin{align*}
    \rkhs=H^{s}_2:=\left\{f\in L^2(\R{D})\cap C(\R{D}) :
   \int(1+\|\bm{\omega}\|^2)^{s}|f^{\wedge}(\bm{\omega})|^2
   \intd\bm{\omega}<\infty\right\}.
   \end{align*}
   $\Gamma(\cdot)$ denotes the Gamma function, and
   $\mathfrak{K}_{v}(\cdot)$ is the modified Bessel function of the
   second kind of order $v$.
   \subsubsection{The Convergence Rate of LSDDR}
   Here, we derive a rate of convergence for LSDDR under the RKHS
   norm. To this end, we assume that the true density-derivative-ratio
   is contained in $\rkhs$:
   \begin{align*}
    r^*_{\bm{j}}(\bm{x}):= \frac{\partial_{\bm{j}}
    p(\bm{x})}{p(\bm{x})}\in\mathcal{H}.
   \end{align*}
   Furthermore, we restrict the search space of $r_{\bm{j}}$ to $\rkhs$
   and express LSDDR with $R(r_{\bm{j}})=\|r_{\bm{j}}\|_{\rkhs}^2$ as
   \begin{align}
    \widehat{r}_{\bm{j}}=\argmin_{r_{\bm{j}}\in\rkhs} \left[
    \widehat{J}_{\bm{j}}(r_{\bm{j}})+\lambda_{\bm{j}}\|r_{\bm{j}}\|_{\rkhs}^2\right].
    \label{eqn:LSDDRrkhs}
   \end{align}
   
   To establish a convergence rate under the RKHS norm, we make the
   following assumptions as in~\citet{sriperumbudur2013density}:
   \begin{itemize}
    \item[{\bf(A)}] $\mathcal{X}$ is compact.
		 
    \item[{\bf(B)}] $k$ is $2|\bm{j}|$ continuously differentiable.
		 
    \item[{\bf(C)}] The following equation holds:
		 \begin{align*}
		  \int_{\X}
		  k(\cdot,\bm{x})\partial_{\bm{j}}p(\bm{x})\intd\bm{x}
		  =(-1)^{|\bm{j}|}\int_{\X}
		  \partial_{\bm{j}}k(\cdot,\bm{x}) p(\bm{x})\intd\bm{x}.
		 \end{align*}
		 
    \item[{\bf(D)}] For all $\bm{j}$, there exists $\epsilon\geq 1$
		 subject to
		 \begin{align*}
		  \left(\int_{\X} \|k(\cdot,\bm{x})\|_{\rkhs}^{2\epsilon} 
		  p(\bm{x})\intd\bm{x}\right)^{\frac{1}{2\epsilon}}<\infty
		  \quad\text{and}\quad\left(\int_{\X}
		  \|\partial_{\bm{j}}k(\cdot,\bm{x})\|_{\rkhs}^{\epsilon} 
		  p(\bm{x})\intd\bm{x}\right)^{\frac{1}{\epsilon}}<\infty.
		 \end{align*}
   \end{itemize}
   Assumption~{\bf (A)} makes $\rkhs$
   separable~\citep[Lemma~4.33]{steinwart2008support} and the
   separability of $\rkhs$ is required to apply Proposition~A.2
   in~\citet{sriperumbudur2013density}. Assumption~{\bf (B)} ensures
   that arbitrary functions in $\rkhs$ are $2|\bm{j}|$ continuously
   differentiable~\citep[Corollary~4.36]{steinwart2008support}. Assumption~{\bf
   (C)} holds under mild assumptions of $k$ and $p$ as
   in~\eqref{secondterm}. From Assumption~{\bf (D)},
   $J_{\bm{j}}(r_{\bm{j}})<\infty$ when $\epsilon=1$. Then, the
   following theorem establishes the convergence rate under the RKHS
   norm:
   \begin{theorem}
    Let
    \begin{align*}
     C:=\int_{\X} k(\cdot,\bm{x})\otimes k(\cdot,\bm{x}) p(\bm{x})\intd\bm{x},
    \end{align*}
    where $\otimes$ denotes the tensor product, be an operator on
    $\rkhs$.
    If there exists $\gamma>0$ such that $r^*_{\bm{j}}$ is in the range
    of $C^{\gamma}$ (i.e., $r^*_{\bm{j}}\in\mathcal{R}(C^{\gamma})$),
    then
    \begin{align*}
     \|\widehat{r}_{\bm{j}}-r^*_{\bm{j}}\|_{\rkhs}
     =O_{\mathrm{P}}\left(n^{-\min\left\{\frac{1}{4},
     \frac{\gamma}{2(\gamma+1)}\right\}}\right),
    \end{align*}
    with $\epsilon=2$ and $\lambda_{\bm{j}}=
    O\left(n^{-\max\left\{\frac{1}{4},\frac{1}{2(\gamma+1)}\right\}}\right)$
    as $n\rightarrow\infty$.  \label{theo:Hconvrate}
   \end{theorem}
   The proof is given in Appendix~\ref{app:Hconv}.  We followed the
   proof techniques in~\citet{sriperumbudur2013density}, but adopted
   them to a different problem: \citet{sriperumbudur2013density}
   proposed and analyzed a non-parametric estimator for log-densities
   with the intractable partition functions based on the Fisher
   divergence, which is a special case of $J_{\bm{j}}$ at
   $|\bm{j}|=1$. The range space assumption
   $r^*_{\bm{j}}\in\mathcal{R}(C^{\gamma})$ is closely related to the
   smoothness of
   $r^*_{\bm{j}}$~\citep[Section~4.2]{sriperumbudur2013density}: Larger
   $\gamma$ implies that $r^*_{\bm{j}}$ is smoother. As seen in
   Sections~\ref{ssec:ModeConv} and~\ref{ssec:ridgeconv},
   Theorem~\ref{theo:Hconvrate} is particularly useful in the analysis
   of our mode-seeking clustering and density ridge estimation methods.
   
   \begin{remark*}
    By following~\citet[Section~4.2]{JMLR:v18:16-011},
    Theorem~\ref{theo:Hconvrate} has some connection to the minimax
    theory~\citep{tsybakov2008introduction} under Sobolev spaces where
    for any $\alpha>s\geq0$, the minimax rate is given by
    \begin{align*}
     \inf_{\widehat{r}_{\bm{j},n}}\sup_{r_{\bm{j}}^*\in H^{\alpha}_2}
     \|\widehat{r}_{\bm{j},n}-r_{\bm{j}}^*\|_{H^{s}_2} \asymp
     n^{-\frac{\alpha-s}{2(\alpha-s)+D}}.
    \end{align*}
    $\inf$ is taken over possible estimators $\widehat{r}_{\bm{j},n}$,
    and $a_n\asymp b_n$ means that $a_n/b_n$ has lower- and upper-bounds
    away from zero and infinity, respectively. To establish a connection
    to Sobolev spaces, suppose that the Mat{\'e}rn kernel is employed
    whose corresponding RKHS is a Sobolev space $\rkhs=H^s_2$ with the
    smoothness parameter $s>D/2$. As proved in
    Appendix~\ref{app:minimax}, when the true density belongs to
    $L^1(\R{D})$ (i.e., $p\in L^1(\R{D})$),
    $r^*_{\bm{j}}\in\mathcal{R}(C^{\gamma})$ for $\gamma\geq 1$ implies
    that $r^*_{\bm{j}}\in H^{\frac{3D}{2}-\frac{1}{2}+\epsilon}_2$ for
    arbitrarily small $\epsilon>0$. Then, the convergence rate
    $n^{-\frac{1}{4}}$ is minimax optimal under
    $\rkhs=H^{D-\frac{1}{2}+\epsilon}_2$.
    Furthermore, this result implies that the dimension effect is veiled
    through the relative smoothness between two Sobolev spaces
    ($H^{\frac{3D}{2}-\frac{1}{2}+\epsilon}_2$ and
    $H^{D-\frac{1}{2}+\epsilon}_2$), and therefore the rate in
    Theorem~\ref{theo:Hconvrate} is independent of data dimension
    $D$. Details are provided in Appendix~\ref{app:minimax}.
   \end{remark*}
  \subsection{Practical Implementation of LSDDR}
  \label{ssec:PracLSDDR}
  Here, we describe practical implementation of LSDDR.
  \begin{itemize}
   \item \emph{A practical version of LSDDR}: The representer
	 theorem~\citep[Theorem~2]{zhou2008derivative} states that the
	 estimator $\widehat{r}_{\bm{j}}$ should take the following
	 form:
	 \begin{align}
	  \widehat{r}_{\bm{j}}(\bm{x})=\sum_{i=1}^n\alpha^{(i)}_{\bm{j}}
	  k(\bm{x},\bm{x}_i) +\beta^{(i)}_{\bm{j}}\partial_{\bm{j}}'
	  k(\bm{x},\bm{x}')\Bigr|_{\bm{x}'=\bm{x}_i} =\sum_{i=1}^{2n}
	  \theta^{(i)}_{\bm{j}}\psi_{\bm{j}}^{(i)}(\bm{x})
	  =\bm{\theta}_{\bm{j}}^{\top}\bm{\psi}_{\bm{j}}(\bm{x}),
	  \label{eqn:g_model}
	 \end{align}
	 where $\partial_{\bm{j}}'$ denotes the partial derivative with 
	 respect to $\bm{x}'$, 
	 \begin{align*}
	  \theta^{(i)}_{\bm{j}}:=\left\{
	  \begin{array}{ll}
	   \alpha^{(i)}_{\bm{j}} & i=1,\dots,n,\\
	   \beta^{(i-n)}_{\bm{j}} & i=n+1,\dots,2n,
	  \end{array}\right. \qquad 
	  \psi_{\bm{j}}^{(i)}(\bm{x}):=\left\{
	  \begin{array}{ll}
	   k(\bm{x},\bm{x}_i) & i=1,\dots,n,\\
	   \partial_{\bm{j}}'k(\bm{x},\bm{x}')\Bigr|_{\bm{x}'=\bm{x}_{i-n}}
	    & i=n+1,\dots,2n.
	  \end{array}\right.
	 \end{align*}
	 
	 To estimate $\bm{\theta}_{\bm{j}}$, we
	 substitute~\eqref{eqn:g_model} into $\widehat{J}_{\bm{j}}$
	 in~\eqref{Jsample}. Then, when
	 $R(r_{\bm{j}})=\bm{\theta}_{\bm{j}}^\top \bm{\theta}_{\bm{j}}$,
	 the optimal solution of $\bm{\theta}_{\bm{j}}$ can be computed
	 analytically as
	 \begin{align*}
	  \widehat{\bm{\theta}}_{\bm{j}}
	  &:=\argmin_{\bm{\theta}_{\bm{j}}}
	  \left[\bm{\theta}_{\bm{j}}^{\top}
	  \widehat{\bm{G}}_{\bm{j}}\bm{\theta}_{\bm{j}}
	  -2(-1)^{|\bm{j}|}\bm{\theta}_{\bm{j}}^{\top}
	  \widehat{\bm{h}}_{\bm{j}}
	  +\lambda_{\bm{j}}\bm{\theta}_{\bm{j}}^\top
	  \bm{\theta}_{\bm{j}}
	  \right]
	  =(-1)^{|\bm{j}|}\left(\widehat{\bm{G}}_{\bm{j}}
	  +\lambda_{\bm{j}}\I_{2n}\right)^{-1}\widehat{\bm{h}}_{\bm{j}},
	 \end{align*}
	 where $\I_{2n}$ denotes the $2n$ by $2n$ identity matrix,
	 \begin{align*}
	  \widehat{\bm{G}}_{\bm{j}}:= \frac{1}{n}\sum_{i=1}^n
	  \bm{\psi}_{\bm{j}}(\bm{x}_i)
	  \bm{\psi}_{\bm{j}}(\bm{x}_i)^{\top}
	  \qquad\text{and}\qquad \widehat{\bm{h}}_{\bm{j}}
	  &:=\frac{1}{n}\sum_{i=1}^n
	  \partial_{\bm{j}}\bm{\psi}_{\bm{j}}(\bm{x}_i).
	 \end{align*}
	 Finally, a practical version of LSDDR is given by
	 \begin{align*}
	  \widehat{r}_{\bm{j}}(\bm{x}):=\widehat{\bm{\theta}}_{\bm{j}}^{\top}
	  \bm{\psi}_{\bm{j}}(\bm{x})=\sum_{i=1}^n
	  \widehat{\alpha}^{(i)}_{\bm{j}} k(\bm{x},\bm{x}_i)
	  +\widehat{\beta}^{(i)}_{\bm{j}}\partial_{\bm{j}}'
	  k(\bm{x},\bm{x}')\Bigr|_{\bm{x}'=\bm{x}_i}.
	 \end{align*}
	  
   \item \emph{Model selection by cross-validation}: Model selection is
	 a crucial problem in LSDDR.  As in standard model selection
	 methods for kernel density
	 estimation~\citep{bowman1984alternative,sheather2004density},
	 we take a least-squares approach based on~\eqref{J}, and
	 optimize the model parameters (parameters in $k(\cdot,\cdot)$
	 and the regularization parameter $\lambda_{\bm{j}}$) by
	 cross-validation as follows:
	 \begin{enumerate}
	  \item Divide the samples $\mathcal{D}=\{\bm{x}_i\}_{i=1}^{n}$
		into $T$ disjoint subsets $\{\mathcal{D}_t\}_{t=1}^{T}$.
		
	  \item Obtain the estimator
		$\widehat{r}^{(t)}_{\bm{j}}(\bm{x})$ from
		$\mathcal{D}\setminus\mathcal{D}_t$
		(i.e., $\mathcal{D}$ without $\mathcal{D}_t$), and then
		compute $\widehat{J}_{\bm{j}}$ from the hold-out samples
		as
		\begin{align*}
		 \text{CV}(t)&:=
		 \frac{1}{|\mathcal{D}_t|}\sum_{\bm{x}\in\mathcal{D}_t}
		 \left[\left\{\widehat{r}^{(t)}_{\bm{j}}(\bm{x})\right\}^2
		 -2(-1)^{|\bm{j}|} \partial_{\bm{j}}
		 \widehat{r}^{(t)}_{\bm{j}}(\bm{x})\right],
		\end{align*} 
		where $|\mathcal{D}_t|$ denotes the number of elements
		in $\mathcal{D}_t$.
		
	  \item Choose the model that minimizes
		$\frac{1}{T}\sum_{t=1}^T\text{CV}(t)$.
	 \end{enumerate}
  \end{itemize}
  \subsection{Notation}
  \label{ssec:notation1}
  In the rest of this paper, we consider LSDDR only for $|\bm{j}|=1$ and
  $|\bm{j}|=2$. Therefore, we use more specific notations as follows:
  \begin{itemize}
   \item (Sections~\ref{modeclustering} and~\ref{sec:ridge}) For
	 $|\bm{j}|=1$, a first order density-derivative-ratio
	 corresponds to a first order derivative of the log-density, and
	 we express the true derivative as
	 \begin{align*}
	  g_{j}(\bm{x}):=\frac{\partial_j p(\bm{x})}{p(\bm{x})}
	  =\partial_j\log p(\bm{x}),
	 \end{align*}
	 where $\partial_j:=\parder{x^{(j)}}$. Then, LSDDR to
	 $g_j(\bm{x})$ is denoted by 
	 \begin{align*}
	  \widehat{g}_{j}(\bm{x})
	  :=\sum_{i=1}^{2n}\widehat{\theta}_{j}^{(i)}\psi_j^{(i)}(\bm{x})
	  =\sum_{i=1}^n \widehat{\alpha}_{j}^{(i)}k(\bm{x},\bm{x}_i)
	  +\widehat{\beta}_{j}^{(i)}\partial_j' k(\bm{x},\bm{x}')\Bigr|_{\bm{x}'=\bm{x}_i},
	 \end{align*}
	 where $\partial_j'$ denotes the partial derivative with respect
	 to the $j$-th coordinate in $\bm{x}'$, and the subscript $j$ of
	 $\widehat{\theta}_{j}^{(i)}$ is simplified from $\bm{j}$
	 because only one element in $\bm{j}$ is one and the others are
	 zeros when $|\bm{j}|=1$.

   \item (Section~\ref{sec:ridge}) For $|\bm{j}|=2$, we express a true
	 second order density-derivative-ratio by
	 $[\bm{H}(\bm{x})]_{ij}:=\frac{\partial_i\partial_j
	 p(\bm{x})}{p(\bm{x})}$ where $[\bm{H}(\bm{x})]_{ij}$ denotes
	 the $(i,j)$-th element of the matrix $\bm{H}(\bm{x})$.  LSDDR
	 to $[\bm{H}(\bm{x})]_{ij}$ is denoted by
	 $[\widehat{\bm{H}}(\bm{x})]_{ij}$.
  \end{itemize}
 \section{Application to Mode-Seeking Clustering}
 \label{modeclustering}
 This section applies LSDDR to mode-seeking clustering. 
  \subsection{Problem Formulation for Clustering}
  Suppose that we are given a collection of data samples
  $\mathcal{D}=\left\{\bm{x}_i\right\}_{i=1}^n$.  The goal of clustering
  is to assign a cluster label $c_i\in\{1,\ldots,c\}$ to each data
  sample $\bm{x}_i$, where $c$ denotes the number of clusters, and is
  \emph{unknown}.
  \subsection{Brief Review of Mean Shift Clustering}
  \label{ssec:meanshift}
  \emph{Mean shift clustering}
  (MS)~\citep{fukunaga1975estimation,cheng1995mean,comaniciu2002mean} is
  a popular clustering method, and has been applied in a wide-range of
  fields such as image
  segmentation~\citep{comaniciu2002mean,tao2007color,wang2004image} and
  object tracking~\citep{collins2003mean,comaniciu2000real} (see a
  recent review article by~\citet{carreira2015review}). MS initially
  regards all data samples as candidates of cluster centers, and updates
  them toward the nearest modes of the estimated density by gradient
  ascent. Finally, the same cluster label is assigned to the data
  samples which converge to the same mode.  Unlike standard clustering
  methods such as \emph{k-means
  clustering}~\citep{BerkeleySymp:MacQueen:1967}, MS automatically
  determines the number of clusters according to the number of detected
  modes.
  
  To update data samples, the technical challenge is to accurately
  estimate the gradient of $p(\bm{x})$. MS takes a two-step approach:
  The first step performs kernel density estimation (KDE) as
  \begin{align*}
   \widehat{p}_{\KDE}(\bm{x})&:=\frac{1}{Z_{n,h}}\sum_{i=1}^n
   K_{\KDE}\left(\frac{\left\|\bm{x}-\bm{x}_i\right\|^2}{2h^2}\right),
  \end{align*}
  where $K_{\KDE}$ is a kernel function for KDE, $Z_{n,h}$ is the
  normalizing constant, and $h$ denotes the bandwidth parameter. Then,
  the second step computes the partial derivatives of
  $\widehat{p}_{\KDE}(\bm{x})$ as
  \begin{align*}
   \partial_{j} \widehat{p}_{\KDE}(\bm{x})
   &=\frac{1}{h^2Z_{n,h}}\sum_{i=1}^n (x^{(j)}_i-x^{(j)})
   G_{\KDE}\left(\frac{\left\|\bm{x}-\bm{x}_i\right\|^2}{2h^2}\right)\\
   &=\frac{1}{h^{2}Z_{n,h}}\left\{\sum_{i=1}^n
   G_{\KDE}\left(\frac{\left\|\bm{x}-\bm{x}_i\right\|^2}{2h^2}\right)
   \right\} \left\{\frac{\sum_{i=1}^n x_i^{(j)}
   G_{\KDE}\left(\frac{\left\|\bm{x}-\bm{x}_i\right\|^2}{2h^2}\right)}
   {\sum_{i=1}^n
   G_{\KDE}\left(\frac{\left\|\bm{x}-\bm{x}_i\right\|^2}{2h^2}\right)}
   -x^{(j)}\right\},
  \end{align*}
  where
  $G_{\KDE}(t)=-\frac{\mathrm{d}}{\mathrm{d}t}K_{\KDE}(t)$.
  
  By denoting the $\tau$-th update of a data sample by
  $\bm{z}^{\tau}_k=(z^{(\tau,1)}_k,z^{(\tau,2)}_k,\dots,z^{(\tau,D)}_k)^{\top}$
  where $\bm{z}^{0}_k=\bm{x}_k$, setting $\partial_{j}
  \widehat{p}_{\KDE}(\bm{x})=0$ yields the following fixed-point
  iteration formula:
  \begin{align}
   z_k^{(\tau+1,j)}=\frac{\sum_{i=1}^n x_i^{(j)}
   G_{\KDE}\left(\frac{\left\|\bm{z}_k^{\tau}-\bm{x}_i\right\|^2}{2h^2}\right)}
   {\sum_{i=1}^n
   G_{\KDE}\left(\frac{\left\|\bm{z}_k^{\tau}-\bm{x}_i\right\|^2}{2h^2}\right)}.
   \label{eqn:meanshift}
  \end{align}
  Simple calculation shows that~\eqref{eqn:meanshift} can be
  equivalently expressed as
  \begin{align}
   \bm{z}_k^{\tau+1}=\bm{z}_k^{\tau}
   +\frac{h^{2}Z_{n,h}}{\sum_{i=1}^nG_{\KDE}
   \left(\frac{\left\|\bm{z}_k^{\tau}-\bm{x}_i\right\|^2}{2h^2}\right)}
   \nabla\widehat{p}_{\KDE}(\bm{x})|_{\bm{x}=\bm{z}_k^{\tau}}
   =\bm{z}_k^{\tau}+\widehat{\bm{m}}_{\KDE}(\bm{z}_k^{\tau}),
  \label{ms-grad-ascent}
  \end{align}
  where $\nabla$ denotes the vector differential operator with respect
  to $\bm{x}$, and
  $\widehat{\bm{m}}_{\KDE}(\bm{z})=(\widehat{m}^{(1)}_{\KDE}(\bm{z}),\widehat{m}^{(2)}_{\KDE}(\bm{z}),\dots,\widehat{m}^{(D)}_{\KDE}(\bm{z}))^{\top}$
  is called the \emph{mean shift vector} and defined by
  \begin{align}
   \widehat{\bm{m}}_{\KDE}(\bm{z})=\frac{h^{2}Z_{n,h}}{\sum_{i=1}^nG_{\KDE}
   \left(\frac{\left\|\bm{z}-\bm{x}_i\right\|^2}{2h^2}\right)}
   \nabla\widehat{p}_{\KDE}(\bm{x})|_{\bm{x}=\bm{z}}.
   \label{mean-shift}
  \end{align}
  Eq.\eqref{ms-grad-ascent} indicates that MS performs gradient ascent.
  To speed up MS, acceleration strategies were also developed
  in~\citet{carreira2006acceleration}. 

  Properties of MS have been theoretically
  well-investigated~\citep{cheng1995mean,fashing2005mean,ghassabeh2013convergence,arias2016estimation}.
  For instance, a sequence $\{\bm{z}_k^{\tau}, \tau=0,1,2,\dots\}$
  generated by MS converges to a mode of $\widehat{p}_{\KDE}(\bm{x})$ as
  $\tau$ goes
  infinity~\citep{comaniciu2002mean,li2007note,ghassabeh2013convergence};
  \citet{carreira2007gaussian} showed that the algorithm of MS is
  equivalent to the EM algorithm~\citep{dempster1977maximum} when
  $K_{\KDE}(t)=\exp(-t)$; Furthermore, \citet{fashing2005mean} proved
  that MS performs a bound optimization. Although MS has good
  theoretical properties, the two-step approach in gradient estimation
  seems practically inappropriate because a good-density estimator does
  not necessarily mean a good-density gradient estimator. A more
  appropriate way would be to directly estimate the gradient. Following
  this idea, we apply LSDDR to mode-seeking clustering.
  \subsection{Least-Squares Log-Density Gradient Clustering}
  \label{sec:Clustering}
  Here, LSDDR is employed to develop a novel mode-seeking clustering
  method because LSDDR is an estimator of a single element in the
  log-density gradient when $|\bm{j}|=1$. The proposed clustering method
  is called the \emph{least-squares log-density gradient clustering}
  (LSLDGC).
   \subsubsection{Fixed-Point Iteration}
   \label{ssec:FixedPoint}
   First, when we estimate the $j$-th element in
   $\bm{g}(\bm{x})=\nabla\log p(\bm{x})$, the form of the kernel
   function is restricted as
   \begin{align*}
    k(\bm{x},\bm{x}_i)=\phi\left(\frac{\|\bm{x}-\bm{x}_i\|^2}
   {2\sigma_j^2}\right),
   \end{align*}
   where $\sigma_j$ denotes a bandwidth parameter, and $\phi$ is a
   non-negative, monotonically non-increasing, convex and differentiable
   function. For example, when $\phi(t)=\exp(-t)$, $k(\bm{x},\bm{x}_i)$
   is the Gaussian kernel. Under the restriction, LSDDR can be rewritten
   as
   \begin{align}
    \widehat{g}_{j}(\bm{x})=\sum_{i=1}^n\left[\widehat{\alpha}^{(i)}_{j}
    \phi\left(\frac{\|\bm{x}-\bm{x}_i\|^2} {2\sigma_j^2}\right)
    +\widetilde{\beta}^{(i)}_{j}\frac{x_i^{(j)}-x^{(j)}}{\sigma_j^2}
    \varphi\left(\frac{\|\bm{x}-\bm{x}_i\|^2}{2\sigma_j^2}\right)\right],
    \label{derGau}
   \end{align}
   where $\widetilde{\beta}^{(i)}_{j}=-\widehat{\beta}^{(i)}_{j}$ and
   $\varphi(t)=-\frac{\intd}{\intd t}\phi(t)$.
   
   For our mode-seeking clustering method, we derive a fixed-point
   iteration similarly to MS.  When
   $\sum_{i=1}^n\widetilde{\beta}^{(i)}_{j}
   \varphi\left(\frac{\|\bm{x}-\bm{x}_i\|^2}{2\sigma_j^2}\right)\neq 0$,
   \eqref{derGau} can be expanded as
   \begin{align*}
    \widehat{g}_{j}(\bm{x}) &= \sum_{i=1}^n
    \left[\widehat{\alpha}^{(i)}_{j}
    \phi\left(\frac{\|\bm{x}-\bm{x}_i\|^2} {2\sigma_j^2}\right)
    +\frac{\widetilde{\beta}^{(i)}_{j}x_i^{(j)}}{\sigma_j^2}
    \varphi\left(\frac{\|\bm{x}-\bm{x}_i\|^2}{2\sigma_j^2}\right)\right]
    -\frac{x^{(j)}}{\sigma_j^2} \sum_{i=1}^n\widetilde{\beta}^{(i)}_{j}
    \varphi\left(\frac{\|\bm{x}-\bm{x}_i\|^2}{\sigma_j^2}\right)
    \nonumber\\
    &=\frac{1}{\sigma_j^2}\sum_{i=1}^n\widetilde{\beta}^{(i)}_{j}
    \varphi\left(\frac{\|\bm{x}-\bm{x}_i\|^2}{2\sigma_j^2}\right)
    \left[\frac{\sum_{i=1}^n\left[\sigma_j^2\widehat{\alpha}^{(i)}_{j}
    \phi\left(\frac{\|\bm{x}-\bm{x}_i\|^2} {2\sigma_j^2}\right)
    +\widetilde{\beta}^{(i)}_{j}x_{i}^{(j)}
    \varphi\left(\frac{\|\bm{x}-\bm{x}_i\|^2}{2\sigma_j^2}\right)\right]}
    {\sum_{i=1}^n\widetilde{\beta}^{(i)}_{j}
    \varphi\left(\frac{\|\bm{x}-\bm{x}_i\|^2}{2\sigma_j^2}\right)}
    -x^{(j)}\right].
   \end{align*}
   As in MS, setting $\widehat{g}_j(\bm{x})=0$ yields the following
   update formula:
   \begin{align}
    z_k^{(\tau+1,j)}=\frac{\sum_{i=1}^n\left[\sigma_j^2\widehat{\alpha}^{(i)}_{j}
    \phi\left(\frac{\|\bm{z}_k^{\tau}-\bm{x}_i\|^2} {2\sigma_j^2}\right)
    +\widetilde{\beta}^{(i)}_{j}x_{i}^{(j)}
    \varphi\left(\frac{\|\bm{z}_k^{\tau}-\bm{x}_i\|^2}{2\sigma_j^2}\right)\right]}
    {\sum_{i=1}^n\widetilde{\beta}^{(i)}_{j}
    \varphi\left(\frac{\|\bm{z}_k^{\tau}-\bm{x}_i\|^2}{2\sigma_j^2}\right)},
    \label{updateelement}
   \end{align}
   where $\bm{z}_k^{\tau}$ denotes the $\tau$-th update of a data sample
   initialized by $\bm{x}_k$. Eq.\eqref{updateelement} can be also
   equivalently expressed as
   \begin{align}
    z^{(\tau+1,j)}=z^{(\tau,j)}+\frac{\sigma_j^2}
    {\sum_{i=1}^n\widetilde{\beta}^{(i)}_{j}
    \varphi\left(\frac{\|\bm{z}^{\tau}-\bm{x}_i\|^2}{2\sigma_j^2}\right)}
    \widehat{g}_j(\bm{z}^{\tau})=z^{(\tau,j)}
    +\widehat{m}^{(j)}(\bm{z}^{\tau}), \label{LS-grad}
   \end{align}
   where
   \begin{align}
    \widehat{m}^{(j)}(\bm{z}):=\frac{\sigma_j^2}
    {\sum_{i=1}^n\widetilde{\beta}^{(i)}_{j}
    \varphi\left(\frac{\|\bm{z}-\bm{x}_i\|^2}{2\sigma_j^2}\right)}
    \widehat{g}_j(\bm{z}).\label{LS-mean-shift}
   \end{align}
   When $\widehat{\alpha}^{(i)}_{j}=0$ and
   $\widetilde{\beta}^{(i)}_{j}=1/n$, \eqref{updateelement} is reduced
   to the MS update formula~\eqref{eqn:meanshift}. Thus, LSLDGC includes
   MS as a special case.
   
   The form of~\eqref{updateelement} motivates us to develop a
   coordinate-wise update rule. From $j=1$ to $j=D$, we iteratively
   update one coordinate at a time by simply
   modifying~\eqref{updateelement} as
   \begin{align}
    z_k^{(\tau+1,j)}=\frac{\sum_{i=1}^n\left[\sigma_j^2\widehat{\alpha}^{(i)}_{j}
    \phi\left(\frac{\|\tilde{\bm{z}}_k^{\tau}-\bm{x}_i\|^2} {2\sigma_j^2}\right)
    +\widetilde{\beta}^{(i)}_{j}x_{i}^{(j)}
    \varphi\left(\frac{\|\tilde{\bm{z}}_k^{\tau}-\bm{x}_i\|^2}{2\sigma_j^2}\right)
    \right]}
    {\sum_{i=1}^n\widetilde{\beta}^{(i)}_{j}
    \varphi\left(\frac{\|\tilde{\bm{z}}_k^{\tau}-\bm{x}_i\|^2}{2\sigma_j^2}\right)},
    \label{CWupdate}
   \end{align}
   where 
   \begin{align*}
    \tilde{\bm{z}}^{\tau}_k=(z_k^{(\tau+1,1)},\dots,z_k^{(\tau+1,j-1)},z_k^{(\tau,j)},z_k^{(\tau,j+1)},\dots,z_k^{(\tau,D)})^{\top}.
   \end{align*}
   Note that the $(j-1)$-th and $j$-th elements in
   $\tilde{\bm{z}}^{\tau}_k$ are different in terms of $\tau$.  As shown
   below, this coordinate-wise update rule has a nice theoretical
   property.
   \subsubsection{Sufficient Conditions for Monotonic Hill-Climbing}
   \label{ssec:LSLDGconv}
   LSLDGC updates data samples towards the modes like hill-climbing.
   Here, we show sufficient conditions for monotonic hill-climbing,
   i.e., LSLDGC makes data samples never climbing-down. The challenge in
   this analysis is that unlike MS, we cannot know the estimated
   density, and thus it is not straightforward to investigate this
   property for LSLDGC. To overcome this challenge, we employ {\it path
   integral}\footnote{Path integral is also called \emph{line
   integral}.}~\citep{strang1991calculus}: For the vector field
   $\g(\x)=\nabla\log{p}(\x)$ and a differentiable curve
   ${\bm\gamma}(t),\,t\in[0,s]$ connecting $\x$ and $\y$, i.e.,
   $\bm{\gamma}(0)=\bm{y},\,\bm{\gamma}(s)=\x$, the standard formula of
   path integral is given by
   \begin{align}
    D_{\g}[\x|\y]:= \int_{0}^s
    \<\g(\bm{\gamma}(t)),\,\dot{\bm\gamma}(t)\> \intd t
    =\log{p(\x)}-\log{p(\y)}, \label{pathintegral}
   \end{align}
   where $\dot{\bm\gamma}(t)=\frac{\intd}{\intd t}\bm\gamma(t)$ and
   $\<\cdot,\cdot\>$ denotes the inner product. The notable property of
   path integral is that the integral is independent of any choice of a
   path, and determined only by the two points, $\bm{y}$ and $\bm{x}$,
   as shown in the most right-hand side of~\eqref{pathintegral}. In this
   analysis, we use the following path along with one coordinate at a
   time repeatedly:
   \begin{align}
    \bm{y}=(y^{(1)},y^{(2)}&,y^{(3)},\dots,y^{(D)})
    \rightarrow(x^{(1)},y^{(2)},y^{(3)},\dots,y^{(D)})
    \rightarrow(x^{(1)},x^{(2)},y^{(3)},\dots,y^{(D)})
    \nonumber\\
    &\rightarrow(x^{(1)},x^{(2)},x^{(3)},\dots,y^{(D)})
    \rightarrow\dots
    \rightarrow(x^{(1)},x^{(2)},x^{(3)},\dots,x^{(D)})=\bm{x}.
    \label{path}
   \end{align}

   By substituting our gradient estimate $\widehat{\bm{g}}(\bm{x})$ into
   the middle part of~\eqref{pathintegral} under the
   path~\eqref{path},
   \begin{align}
    \widehat{D}_{\widehat{\bm{g}}}[\bm{x}|\bm{y}]
    &:=\int_{0}^s
    \<\widehat{\bm{g}}(\bm{\gamma}(t)),\,\dot{\bm{\gamma}}(t)\>\intd t
    =\sum_{j=1}^D \int_{y^{(j)}}^{x^{(j)}}
    \widehat{g}_{j}(x^{(1)},x^{(2)},\dots,z^{(j)},\dots,y^{(D)})\intd
    z^{(j)}.  \label{mypathintegral}
   \end{align}
   From~\eqref{pathintegral},
   $\widehat{D}_{\widehat{\bm{g}}}[\bm{x}|\bm{y}]$ can be regarded as an
   estimator of $\log{p(\x)}-\log{p(\y)}$ when we fix the curve that
   connects $\x$ and $\y$. Thus,
   $\widehat{D}_{\widehat{\bm{g}}}[\bm{z}_k^{\tau+1}|\bm{z}_k^{\tau}]\geq
   0$ for all $\tau$ implies that the data samples updated by LSLDGC
   never climb down. The following theorem provides some sufficient
   conditions:
   \begin{theorem}
    \label{theo:LSLDGCconv} Suppose that $\phi$ is a non-negative,
    monotonically non-increasing, convex and differentiable
    function. Then, if $\widehat{\alpha}_{j}^{(i)}=0$ and
    $\widetilde{\beta}_{j}^{(i)}\geq 0$, under the coordinate-wise
    update rule~\eqref{CWupdate} and path~\eqref{path},
    \begin{align*}
     \widehat{D}_{\widehat{\bm{g}}}[\bm{z}_k^{\tau+1}|\bm{z}_k^{\tau}]
     \geq 0.
    \end{align*}
   \end{theorem}
   The proof is deferred to Appendix~\ref{app:LSLDGCconv}.
   \begin{remark*}
    Theorem~\ref{theo:LSLDGCconv} shows sufficient conditions that
    LSLDGC with the coordinate-wise update rule~\eqref{CWupdate} makes
    data samples monotonically hill-climb towards the modes. However,
    without satisfying the conditions, we empirically observed that most
    of data samples monotonically converge to modes. Therefore, we
    conjecture that some milder conditions exist, and do not apply all
    sufficient conditions in practice. Practical implementation is
    described in Section~\ref{LSLDGCpractice}.
   \end{remark*}
   \begin{remark*}    
    For another update rule~\eqref{updateelement}, sufficient conditions
    for monotonic hill-climbing were not established as in
    Theorem~\ref{theo:LSLDGCconv}. However,
    Theorem~\ref{theo:nonintegrable-field} implies that accurate
    mode-seeking is possible for both update rules as long as
    $\widehat{D}_{\widehat{\bm{g}}}[\bm{z}_k^{\tau+1}|\bm{z}_k^{\tau}]$
    is kept non-negative for all $\tau$. Therefore, in practice,
    whenever
    $\widehat{D}_{\widehat{\bm{g}}}[\bm{z}_k^{\tau+1}|\bm{z}_k^{\tau}]$
    is negative, we perform standard gradient ascent. The details are
    given in Section~\ref{LSLDGCpractice}.
   \end{remark*}
   \begin{remark*}    
    Sufficient conditions for monotonic hill-climbing have been
    established in
    MS~\citep{comaniciu2002mean,li2007note,ghassabeh2013convergence}.
    The main difference is that we obtain the difference of two
    log-density estimates from a gradient estimate, while previous work
    directly begins with density estimation based on KDE. Thus, the
    proof is substantially different.
   \end{remark*}
   
   Theorem~\ref{theo:LSLDGCconv} holds under the path~\eqref{path}.
   However, the following theorem states that as $n$ increases,
   $\widehat{D}_{\widehat{\bm{g}}}[\bm{x}|\bm{y}]$ approaches
   $D_{\bm{g}}[\bm{x}|\bm{y}]$, which is independent of the choice of a
   path:
   \begin{theorem}
    \label{theo:nonintegrable-field} Suppose that both $\bm{g}$ and
    $\widehat{\bm{g}}$ are finite on the path~\eqref{path} and the
    assumptions in Theorem~\ref{theo:Hconvrate} hold. Then, for arbitrary
    $\bm{x}$ and $\bm{y}$,
    \begin{align*}
     \left|
     D_{\bm{g}}[\bm{x}|\bm{y}]-\widehat{D}_{\widehat{\bm{g}}}[\bm{x}|\bm{y}]
     \right|&\leq
     \|\bm{g}-\widehat{\bm{g}}\|_{\infty}\|\bm{x}-\bm{y}\|_{1}
     \leq O_{\mathrm{P}}\left(n^{-\min\left\{\frac{1}{4},
     \frac{\gamma}{2(\gamma+1)}\right\}}\right),     
    \end{align*}
    where $\|\cdot\|_1$ denotes the $\ell_1$ norm.
   \end{theorem}
   The proof is given in Appendix~\ref{app:proofnonintegrable-field}.
   \begin{remark*}
    Theorem~\ref{theo:nonintegrable-field} shows
    \begin{align}
     \label{eqn:one-step-diff}
     |D_{\g}[\z^{\tau}_k|\z^{\tau+1}_k]
     -\widehat{D}_{\widehat{\g}}[\z^{\tau}_k|\z^{\tau+1}_k]|
     &\leq \|\g-\widehat{\g}\|_\infty 
     \|\z^{\tau}_k-\z^{\tau+1}_k\|_1.
    \end{align}
    From~\eqref{eqn:one-step-diff}, the non-negativity of
    $\widehat{D}_{\widehat{\g}}[\z^{\tau}_k|\z^{\tau+1}_k]$ implies that
    $D_{{\g}}[\z^{\tau}_k|\z^{\tau+1}_k]$ is also non-negative when $n$
    is sufficiently large. Thus, Theorem~\ref{theo:nonintegrable-field}
    ensures that accurate mode-seeking is possible by both update
    rules~\eqref{updateelement} and~\eqref{CWupdate}.
   \end{remark*}
   \subsubsection{The Convergence Rate to the True Mode Set}
   \label{ssec:ModeConv}
   First, we define the set of the true mode points as
   \begin{align}
    \mathcal{M}:=\left\{\bm{\mu}~:~\bm{g}(\bm{\mu})=\bm{0},
    \nabla\bm{g}(\bm{\mu})\prec\bm{O}\right\},
    \label{definition-of-mode-set}
   \end{align}
   where $\nabla\bm{g}(\bm{\mu})$ is the Hessian matrix of the
   log-density at a mode point $\bm{\mu}$, and
   $\nabla\bm{g}(\bm{\mu})\prec\bm{O}$ means that
   $\nabla\bm{g}(\bm{\mu})$ is (strictly) negative definite. The set of
   the estimated mode points is also denoted by $\widehat{\mathcal{M}}$.
   Our goal is to establish the convergence rate between $\mathcal{M}$
   and $\widehat{\mathcal{M}}$ under the Hausdorff distance:
   \begin{align}
    \Haus(\mathcal{A},\mathcal{B})
   :=\max\left(\sup_{\bm{x}\in\mathcal{A}}\inf_{\bm{y}\in\mathcal{B}}\|\bm{x}-\bm{y}\|,
   \sup_{\bm{y}\in\mathcal{B}}\inf_{\bm{x}\in\mathcal{A}}\|\bm{x}-\bm{y}\|\right),
   \label{Hausdist}
   \end{align}
   where $\mathcal{A}$ and $\mathcal{B}$ denote two sets.
   
   The following theorem establishes the convergence rate of
   $\Haus(\widehat{\mathcal{M}},\mathcal{M})$.
   \begin{theorem}
    \label{theo:Modeconv} Suppose that the assumptions in
    Theorem~\ref{theo:Hconvrate} hold. Further assume that
    each mode point $\bm{\mu}\in\mathcal{M}$ is approximated by a unique
    estimated mode point
    $\widehat{\bm{\mu}}\in\widehat{\mathcal{M}}$. Then, with high
    probability,
    \begin{align}
     \Haus(\widehat{\mathcal{M}},\mathcal{M})=
     O_{\mathrm{P}}\left(n^{-\min\left\{\frac{1}{4},
     \frac{\gamma}{2(\gamma+1)}\right\}}\right). \label{Modeconv}
    \end{align}
   \end{theorem}
   The proof can be seen in Appendix~\ref{app:proofModeconv}.
   \begin{remark*}
    \citet[Theorem~1]{chen2016comprehensive} established the following
    convergence rate based on KDE: With the asymptotically optimal
    bandwidth $h=O\left(n^{-\frac{1}{D+6}}\right)$,
    \begin{align}
     \text{Haus}(\widehat{\mathcal{M}}_{\text{KDE}},\mathcal{M})
     =O_{\mathrm{P}}\left(n^{-\frac{2}{D+6}}\right), \label{KDEModeconv}
    \end{align}
    where $\widehat{\mathcal{M}}_{\text{KDE}}$ denotes the set of mode
    points based on KDE. Eq.\eqref{KDEModeconv} shows that the
    convergence rate of
    $\text{Haus}(\widehat{\mathcal{M}}_{\text{KDE}},\mathcal{M})$
    depends on data dimension $D$, although direct comparison to our
    result is not straightforward due to the different assumptions in
    both analyses.
   \end{remark*}
   \subsection{Practical Implementation of LSLDGC}
   \label{LSLDGCpractice}
   Here, we describe details of practical implementation of LSLDGC.
   \begin{itemize}
    \item \emph{Sufficient conditions in Theorem~\ref{theo:LSLDGCconv}}:
	  The conditions, $\widehat{\alpha}_{j}^{(i)}=0$ and
	  $\widetilde{\beta}_{j}^{(i)}\left(=-\widehat{\beta}_{j}^{(i)}\right)\geq
	  0$ , ensure that
	  $\widehat{D}_{\widehat{\bm{g}}}[\bm{z}^{\tau+1}_k|\bm{z}^{\tau}_k]\geq0$.
	  Here, we set $\widehat{\alpha}_{j}^{(i)}=0$ for all $i$ and
	  $j$, and the coordinate-wise update rule~\eqref{CWupdate} is
	  simplified as
	  \begin{align*}
	   z_k^{(\tau+1,j)}=\frac{\sum_{i=1}^n
	   \widetilde{\beta}^{(i)}_{j}x_{i}^{(j)}
	   \varphi\left(\frac{\|\tilde{\bm{z}}_k^{\tau}-\bm{x}_i\|^2}{2\sigma_j^2}
	   \right)}
	   {\sum_{i=1}^n\widetilde{\beta}^{(i)}_{j}
	   \varphi\left(\frac{\|\tilde{\bm{z}}_k^{\tau}-\bm{x}_i\|^2}{2\sigma_j^2}
	   \right)}.
	  \end{align*}
	  The same simplification is applied to the update
	  rule~\eqref{updateelement} as well. This significantly reduces
	  the computational costs in LSDDR because $\alpha_j^{(i)}$ do
	  not need to be estimated.  On the other hand, to satisfy
	  $\widetilde{\beta}_{j}^{(i)}\geq 0$, we have to solve a
	  constrained optimization problem, which tends to be
	  time-consuming. Therefore, the unconstrained optimization
	  problem is solved as in Section~\ref{LSLDGCpractice}, but as a
	  remedy we perform gradient ascent whenever
	  $\widehat{D}_{\widehat{\bm{g}}}[\bm{z}^{\tau+1}_k|\bm{z}^{\tau}_k]<0$.
	  Details of the gradient ascent are given below.
	  
    \item \emph{Stability in the mode-seeking process}: The derivation
	  of~\eqref{updateelement} indicates that the mode-seeking
	  (hill-climbing) process in LSLDGC can be unstable when
	  $f_j(\bm{z}_k^{\tau}):=\sum_{i=1}^n\widetilde{\beta}_{j}^{(i)}
	  \varphi\left(\frac{\|\bm{z}_k^{\tau}-\bm{x}_i\|}{2\sigma_j^2}\right)$
	  is close to zero. To cope with this problem, we simply perform
	  gradient ascent when $f_j(\bm{z}_k^{\tau})$ is close to zero.
	  
    \item \emph{Gradient ascent}: Whenever
	  $\widehat{D}_{\widehat{\bm{g}}}[\bm{z}^{\tau+1}_k|\bm{z}^{\tau}_k]<0$
	  or $\exists j$, $f_j(\bm{z}_k^{\tau})\approx 0$, we perform
	  the following gradient ascent:
	  \begin{align}
	   \bm{z}_k^{\tau+1}
	   =\bm{z}_k^{\tau}+\eta\widehat{\bm{g}}(\bm{z}_k^{\tau}),
	   \label{eqn:grad-ascent}
	  \end{align}
	  where the step size parameter $\eta$ is selected so that
	  $\widehat{D}_{\widehat{\bm{g}}}[\bm{z}_k^{\tau}+\eta\widehat{\bm{g}}(\bm{z}_k^{\tau})|\bm{z}^{\tau}_k]$
	  is maximized.

    \item \emph{Choice of the kernel function}: Throughout the paper, we
	  use the Gaussian kernel:
	  \begin{align*}
	   k(\bm{x},\bm{x}_i)
	   =\phi\left(\frac{\|\bm{x}-\bm{x}_i\|^2} {2\sigma_j^2}\right)
	   =\exp\left(-\frac{\|\bm{x}-\bm{x}_i\|^2}{2\sigma_j^2}\right).
	  \end{align*}
	  The Gaussian kernel satisfies the conditions of $\phi$ in
	  Theorem~\ref{theo:LSLDGCconv}, and is a universal kernel
	  associated with which RKHS covers a wide range of
	  functions~\citep{micchelli2006universal}.	  

    \item \emph{Decreasing the computation costs}: After the
	  simplification above, LSDDR requires to compute the inverse of
	  a $2n$ by $2n$ matrix, which is computationally costly to
	  large $n$. To decrease the computation costs, we reduce the
	  number of center points as
	  $\phi\left(\frac{\|\bm{x}-\bm{c}_i\|^2}{2\sigma_j^2}\right)$
	  and
	  $\varphi\left(\frac{\|\bm{x}-\bm{c}_i\|^2}{2\sigma_j^2}\right)$
	  where $\{\bm{c}_i\}_{i=1}^b$ is a randomly chosen subset of
	  $\{\bm{x}_i\}_{i=1}^n$. As a result, the coefficients can be
	  represented as
	  $\widetilde{\bm{\beta}}_{j}=(\widetilde{\beta}_{j}^{(1)},\widetilde{\beta}_{j}^{(2)},\dots,\widetilde{\beta}_{j}^{(b)})^{\top}$.
	  Appendix~\ref{app:kernelcenter} shows that this significantly
	  decreases the computation cost without scarifying clustering
	  performance.  In this paper, we fix the number of centers at
	  $b=\min(n,100)$ as long as we do not specify it.
   \end{itemize}
   The mode-seeking algorithm in LSLDGC is summarized in
   Figs.\ref{alg:LSLDGC} and~\ref{alg:mode-seeking}.\footnote{A MATLAB
   package of LSLDGC is available at
   \url{https://sites.google.com/site/hworksites/home/software/lsldg}.}
   \begin{figure}[p]
   \centering \noindent\fbox{
    \begin{minipage}{\dimexpr\linewidth-2\fboxsep-2\fboxrule\relax}
     \begin{algorithmic}
      \State \textbf{Input:} $\{\vector{x}_i\}_{i=1}^n$. \\
      \State{$\left\{\{\widetilde{\bm{\beta}}_{j}\}_{j=1}^D, \{\bm{c}_i\}_{i=1}^b\right\}
      \leftarrow\text{LSDDR1}(\{\vector{x}_i\}_{i=1}^n)$};\\
      \For{$k=1$ to $n$}
      \State{$\tau\leftarrow 0$};
      \State{$\vector{z}^{\tau}_k\leftarrow \vector{x}_k$};
      \Repeat 
      \State{$\{\bm{z}_k^{\tau+1},\{f_j\}_{j=1}^D\}
      \leftarrow\text{ModeSeeking}
      (\{\widetilde{\bm{\beta}}_{j}\}_{j=1}^D,\{\bm{c}_i\}_{i=1}^b,
      \bm{z}_k^{\tau})$}
      \State{$\widehat{D}\leftarrow     
      \widehat{D}_{\widehat{\bm{g}}}[\bm{z}_k^{\tau+1}|\bm{z}_k^{\tau}]$};\\
      \If{$\widehat{D}<0$ or~$\exists j,~|f_j|\approx 0$}
      \State{$\bm{z}_k^{\tau+1} \leftarrow
      \bm{z}_k^{\tau}+\eta\widehat{\bm{g}}(\bm{z}_k^{\tau})$};
      \State{$\widehat{D}\leftarrow
      \widehat{D}_{\widehat{\bm{g}}}[\bm{z}_k^{\tau+1}|\bm{z}_k^{\tau}]$};
     \EndIf
     \State{$\tau\leftarrow\tau+1$};
      \Until
     \State{$\vector{z}_k\leftarrow\vector{z}^{\tau}_k$};
     \EndFor\\
     \State \textbf{Outputs:} $\{\vector{z}_i\}_{i=1}^n$.
     \end{algorithmic}
    \end{minipage}}
    \caption{\label{alg:LSLDGC} The mode-seeking algorithm in LSLDGC.
    $\text{LSDDR1}(\{\vector{x}_i\}_{i=1}^n)$ denotes the LSDDR
    estimator for the first-order density-derivative-ratios from data
    samples $\{\vector{x}_i\}_{i=1}^n$, and
    $\{\widetilde{\bm{\beta}}_{j}\}_{j=1}^D$ and $\{\bm{c}_i\}_{i=1}^b$
    are the coefficients and (sub-sampled) centers, respectively.
    $\text{ModeSeeking}
    (\{\widetilde{\bm{\beta}}_{j}\}_{j=1}^D,\{\bm{c}_i\}_{i=1}^b,
    \bm{z}_k^{\tau})$ is a single step mode-seeking process whose
    details are given in Fig.\ref{alg:mode-seeking}. The update of
    $\bm{z}_k^{\tau}$ terminates when either $\widehat{D}$ or
    $\|\bm{z}_k^{\tau+1}-\bm{z}_k^{\tau}\|$ is less than a small
    positive constant.}
    \centering \noindent\fbox{
    \begin{minipage}{0.465\textwidth}
     \begin{algorithmic}
      \State \textbf{Input:}
      $\{\widetilde{\bm{\beta}}_{j}\}_{j=1}^D,\{\bm{c}_i\}_{i=1}^b,
      \bm{z}_k^{\tau}$\\
      \State{$\bm{z}_k^{\tau+1}\leftarrow\bm{z}_k^{\tau}
      +\widehat{\bm{m}}(\bm{z}_k^{\tau})$};\\
      \State{$\left\{f_j\right\}_{j=1}^D\leftarrow
      \left\{\widetilde{\bm{\beta}}_j^{\top}\bm{\varphi}_j(\bm{z}_k^{\tau})\right\}_{j=1}^D$};\\
      \State \textbf{Outputs:} $\vector{z}_k^{\tau+1},
      \left\{f_j\right\}_{j=1}^D$.
     \end{algorithmic}
    \end{minipage}
    } 
    \centering \noindent\fbox{
     \begin{minipage}{0.465\textwidth}
     \begin{algorithmic}
      \State \textbf{Input:}
      $\{\widetilde{\bm{\beta}}_{j}\}_{j=1}^D,\{\vector{c}_i\}_{i=1}^b,      
      \bm{z}_k^{\tau}$
      \State{$\tilde{\bm{z}}\leftarrow\bm{z}_k^{\tau}$};
      \For{$j\in\{1,\dots,D\}$} 
      \State{$\tilde{z}^{(j)}\leftarrow z_k^{(\tau,j)}
      +m^{(j)}(\tilde{\bm{z}})$};
      \State{$f_j\leftarrow\widetilde{\bm{\beta}}_j^{\top}\bm{\varphi}_j(\tilde{\bm{z}})$};
      \EndFor
      \State{$\bm{z}_k^{\tau+1}\leftarrow\tilde{\bm{z}}$};
      \State \textbf{Outputs:} $\vector{z}_k^{\tau+1},
      \left\{f_j\right\}_{j=1}^D$.
     \end{algorithmic}
     \end{minipage}
    } \caption{\label{alg:mode-seeking} Two mode-seeking algorithms in
    LSLDGC. The left figure uses the update rule~\eqref{updateelement},
    while the right one is based on the coordinate-wise update
    rule~\eqref{CWupdate}.
    $\bm{\varphi}_j(\bm{z})=(\varphi_j^{(1)}(\bm{z}),\varphi_j^{(2)}(\bm{z}),
    \dots,\varphi_j^{(b)}(\bm{z}))^{\top}$ where
    $\varphi_j^{(i)}(\bm{z})=\varphi\left(\frac{\|\bm{z}-\bm{c}_i\|^2}{2\sigma_j^2}\right)$.}
     \end{figure}
 \section{Application to Density Ridge Estimation}
 \label{sec:ridge}
 This section applies LSDDR to density ridge estimation and develops a
 novel method.
  \subsection{Problem Formulation for Density Ridge Estimation}
  \label{ssec:problem-ridge}
  For a positive integer $d$ such that $d<D$, the goal is to estimate
  from a collection of data samples $\mathcal{D}=\{\bm{x}_i\}_{i=1}^n$
  the $d$-dimensional \emph{density ridge}, which is defined as a
  collection of points satisfying
  \begin{align}
   \mathcal{R}&:=\{\bm{x}\in\R{D}~|~\|\V(\bm{x})\V(\bm{x})^{\top}
   \bm{g}(\bm{x})\|=0, \eta_{d+1}(\bm{x})<0\},\label{ridge}
  \end{align}
  where $\bm{g}(\bm{x})=\nabla\log p(\bm{x})$,
  $\V(\bm{x})=(\bm{v}_{d+1},\dots,\bm{v}_{D})$, and $\bm{v}_{i}$ is the
  eigenvector associated with the eigenvalue $\eta_i(\bm{x})$ of the
  Hessian matrix of the logarithm of the probability density function,
  $\nabla\nabla \log p(\bm{x})$. We assume that the eigenvalues are
  sorted in descending order such that
  $\eta_1(\bm{x})\geq\eta_2(\bm{x})\geq\dots\geq\eta_D(\bm{x})$.

  Here, we defined the density ridge in terms of the logarithm of the
  probability density function because our practical algorithm is
  proposed based on the logarithm. While the density ridge has been
  previously defined without the
  logarithm~\citep{eberly1996ridges,ozertem2011locally,genovese2014nonparametric,chen2015asymptotic},
  both definitions offer the same density ridge.
  \subsection{Brief Review of Subspace Constrained Mean Shift}
  \label{ssec:SCMS}
  A practical algorithm for density ridge estimation called
  \emph{subspace constrained mean shift} (SCMS) was proposed
  by~\citet{ozertem2011locally}. SCMS extends MS: SCMS performs
  projected gradient ascent on the subspace orthogonal to the density
  ridge, while MS updates data points by gradient ascent. SCMS obtains
  such a subspace as the span of the eigenvectors of the negative
  Hessian matrix of the log-density, which is called the \emph{inverse
  local-covariance matrix}~\citep{ozertem2011locally}:
  \begin{align}
   \Sig^{-1}(\bm{x})&:= -\nabla\nabla \log p(\bm{x})
   =-\frac{\nabla\nabla p(\bm{x})}{p(\bm{x})} +\frac{\nabla p(\bm{x})\nabla
   p(\bm{x})^\top}{p(\bm{x})^2} 
   =-\bm{H}(\bm{x})+\bm{g}(\bm{x})\bm{g}(\bm{x})^{\top}.\label{local}
  \end{align}
  An advantage of employing the log-density is discussed in the context
  of manifold estimation in~\citet{genovese2014nonparametric}: Theorem~7
  in~\citet{genovese2014nonparametric} states that when $D$-dimensional
  data is assumed to be generated on a $d$-dimensional manifold with
  $D$-dimensional Gaussian noise, the density ridge is close to the
  lower-dimensional manifold in the sense of the Hausdorff distance, and
  thus can be a surrogate for the manifold. This surrogate property
  holds in an $O(1)$ neighborhood of the manifold for the log-density,
  while the theorem holds in an $O(\sigma_n)$ neighborhood of the
  manifold for the (non-log) density, where $\sigma_n$ is the standard
  deviation of the Gaussian noise.  Furthermore, when $p(\bm{x})$ is
  Gaussian, \eqref{local} reduces to the inverse of the covariance
  matrix. This allows us to intuitively understand that SCMS finds the
  subspace by PCA to the non-stationary covariance matrix at a location
  $\bm{x}$ around the ridge.

  In practice, SCMS substitutes $\widehat{p}_{\KDE}(\bm{x})$
  into~\eqref{local}:
  \begin{align*}
   \widehat{\Sig}^{-1}_{\KDE}(\bm{x})&:= 
   -\frac{\nabla\nabla \widehat{p}_{\KDE}(\bm{x})}{\widehat{p}_{\KDE}(\bm{x})} 
   +\frac{\nabla \widehat{p}_{\KDE}(\bm{x})\nabla\widehat{p}_{\KDE}(\bm{x})^\top}
   {\widehat{p}_{\KDE}(\bm{x})^2}.
  \end{align*}
  Then, SCMS obtains the orthogonal projector to the subspace as
  $\widehat{\bm{L}}_{\KDE}(\bm{x})
  =\widehat{\V}_{\KDE}(\bm{x})\widehat{\V}_{\KDE}(\bm{x})^{\top}$, where
  $\widehat{\V}_{\KDE}(\bm{x})\in\R{D\times(D-d)}$ consists of the $D-d$
  eigenvectors associated with the $D-d$ largest eigenvalues of
  $\widehat{\Sig}^{-1}_{\KDE}(\bm{x})$. Then, the update rule of SCMS is
  given by
  \begin{align}
   \bm{z}^{\tau+1}=\bm{z}^{\tau}
   +\widehat{\bm{L}}_{\KDE}(\bm{z}^{\tau})
   \widehat{\bm{m}}_{\KDE}(\bm{z}^{\tau}),
   \label{projeted-gradient-SCMS}
  \end{align}
  where $\bm{z}^{\tau}$ denotes the $\tau$-th update of an arbitrarily
  initialized point and $\widehat{\bm{m}}_{\KDE}(\bm{x})$ is the mean
  shift vector defined
  in~\eqref{mean-shift}. Eq.\eqref{projeted-gradient-SCMS} is repeatedly
  applied until convergence. The monotonic hill-climbing property for
  SCMS is proved in~\citet{ghassabeh2013some}.

  One of the key challenges in SCMS is to accurately estimate
  $\Sig^{-1}(\bm{x})$ in~\eqref{local}. SCMS takes a three-step
  approach, i.e., estimate $p(\bm{x})$ by KDE, compute its derivatives,
  and plug them into $\Sig^{-1}(\bm{x})$. However, this approach can
  perform poorly because of the same reason as MS, i.e., a good density
  estimator does not necessarily mean a good density derivative
  estimator. In addition, division by the estimated density could
  further magnify the estimation error for density derivatives. To cope
  with this problem, we employ LSDDR for direct estimation of
  density-derivative-ratios in $\Sig^{-1}(\bm{x})$ without going through
  density estimation and division, and propose a novel method for
  density ridge estimation.
  \subsection{Least-Squares Density Ridge Finder}
  \label{ssec:LSDRF}
  Based on LSDDR, we develop a novel density ridge finder called the
  \emph{least-squares density ridge finder} (LSDRF), which extends
  LSLDGC for density ridge estimation.
  \subsubsection{Algorithm of LSDRF}
   \begin{figure}[!t]
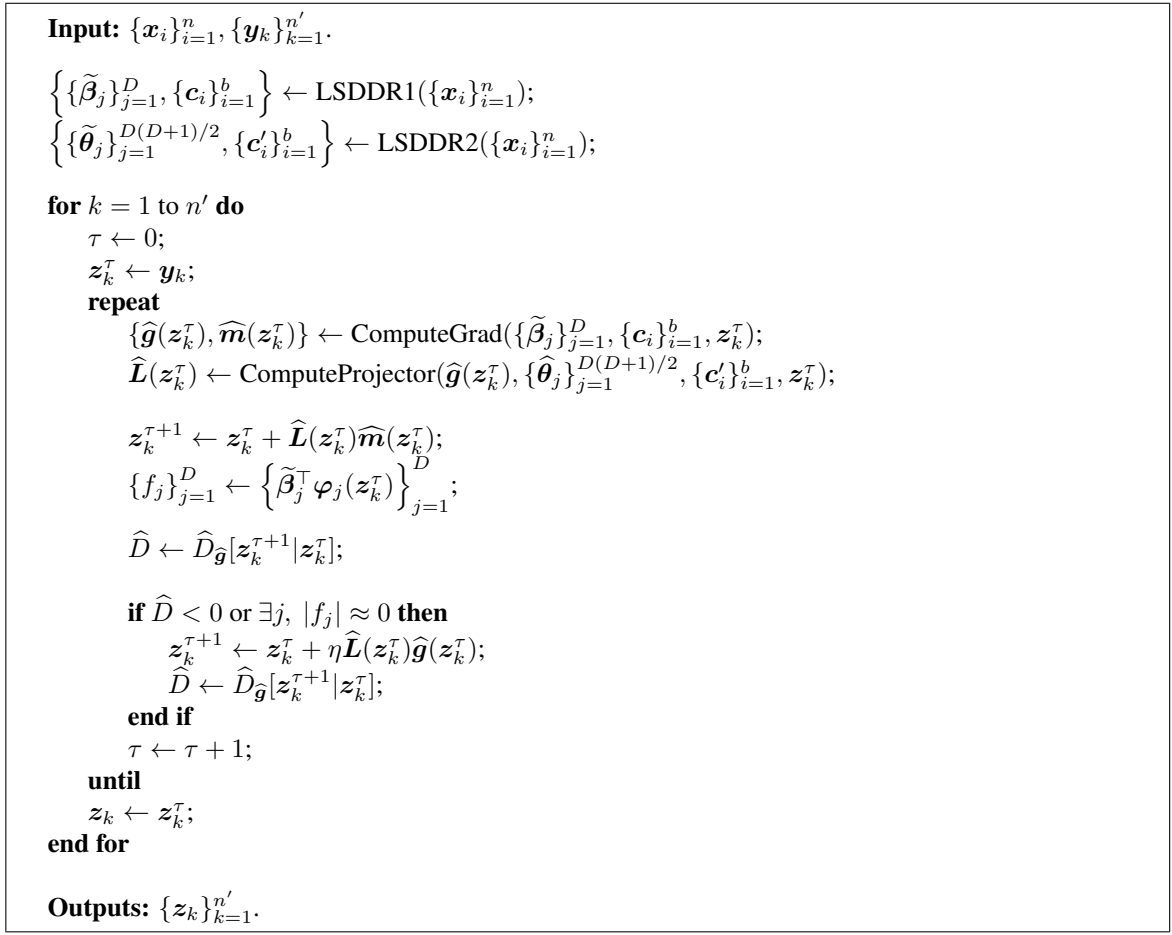

    \centering \noindent\fbox{
    \begin{minipage}{\dimexpr\linewidth-2\fboxsep-2\fboxrule\relax}
     \begin{algorithmic}
      \State \textbf{Input:} $\{\vector{x}_i\}_{i=1}^n,
      \{\vector{y}_k\}_{k=1}^{n'}$. \\
      \State{$\left\{\{\widetilde{\bm{\beta}}_{j}\}_{j=1}^D,
      \{\bm{c}_i\}_{i=1}^b\right\}
      \leftarrow\text{LSDDR1}(\{\vector{x}_i\}_{i=1}^n)$};
      \State{$\left\{\{\widetilde{\bm{\theta}}_{j}\}_{j=1}^{D(D+1)/2},
      \{\bm{c}_i'\}_{i=1}^b\right\}
      \leftarrow\text{LSDDR2}(\{\vector{x}_i\}_{i=1}^n)$};\\
      \For{$k=1$ to $n'$}
      \State{$\tau\leftarrow 0$};
      \State{$\vector{z}^{\tau}_k\leftarrow \vector{y}_k$};
      \Repeat
      \State{$\left\{\widehat{\bm{g}}(\bm{z}^{\tau}_k),\widehat{\bm{m}}(\bm{z}^{\tau}_k)\right\}
      \leftarrow\text{ComputeGrad}
      (\{\widetilde{\bm{\beta}}_{j}\}_{j=1}^D,\{\vector{c}_i\}_{i=1}^b,
      \bm{z}_k^{\tau})$};
      \State{$\widehat{\bm{L}}(\bm{z}^{\tau}_k)\leftarrow\text{ComputeProjector}
      (\widehat{\bm{g}}(\bm{z}^{\tau}_k),
      \{\widehat{\bm{\theta}}_{j}\}_{j=1}^{D(D+1)/2},\{\vector{c}_i'\}_{i=1}^b,
      \bm{z}_k^{\tau})$};\\
      \State{$\bm{z}_k^{\tau+1} \leftarrow \bm{z}^{\tau}_k
      +\widehat{\bm{L}}(\bm{z}_k^{\tau})\widehat{\bm{m}}(\bm{z}_k^{\tau})$;}
      \State{$\left\{f_j\right\}_{j=1}^D\leftarrow
      \left\{\widetilde{\bm{\beta}}_{j}^{\top}\bm{\varphi}_j(\bm{z}_k^{\tau})\right\}_{j=1}^D$;}\\
      \State{$\widehat{D}\leftarrow
      \widehat{D}_{\widehat{\bm{g}}}[\bm{z}_k^{\tau+1}|\bm{z}_k^{\tau}]$;}\\
      \If{$\widehat{D}<0$ or~$\exists j,~|f_j|\approx 0$}
      \State{$\bm{z}_k^{\tau+1} \leftarrow \bm{z}^{\tau}_k
      +\eta\widehat{\bm{L}}(\bm{z}_k^{\tau})\widehat{\bm{g}}(\bm{z}_k^{\tau})$};
      \State{$\widehat{D}\leftarrow
      \widehat{D}_{\widehat{\bm{g}}}[\bm{z}_k^{\tau+1}|\bm{z}_k^{\tau}]$};
      \EndIf
      \State{$\tau\leftarrow\tau+1$};
      \Until
      \State{$\vector{z}_k\leftarrow\vector{z}^{\tau}_k$};
      \EndFor\\
      \State \textbf{Outputs:} $\{\vector{z}_k\}_{k=1}^{n'}$.
     \end{algorithmic}
    \end{minipage}}
    \caption{\label{alg:LSDRF} The algorithm of LSDRF.
    $\text{LSDDR2}(\{\vector{x}_i\}_{i=1}^n)$ denotes the LSDDR
    estimator for the second-order density-derivative-ratios, and
    $\{\widehat{\bm{\theta}}_{j}\}_{j=1}^{D(D+1)/2}$ are the
    corresponding coefficient vectors. $\{\vector{y}_k\}_{k=1}^{n'}$ are
    initial points to approximate the density ridge.
    $\text{ComputeGrad}(\{\widetilde{\bm{\beta}}_{j}\}_{j=1}^D,\{\vector{c}_i\}_{i=1}^b,
    \bm{z}_k^{\tau})$ computes the estimated log-density gradient
    $\widehat{\bm{g}}(\bm{z}_k^{\tau})$ and
    $\widehat{\bm{m}}(\bm{z}_k^{\tau})$ in~\eqref{LS-mean-shift}, while
    $\text{ComputeProjector}(\widehat{\bm{g}}(\bm{z}^{\tau}_k),
    \{\widehat{\bm{\theta}}_{j}\}_{j=1}^{D(D+1)/2},\{\vector{c}_i'\}_{i=1}^b,
    \bm{z}_k^{\tau})$ computes the subspace projector
    $\widehat{\bm{L}}(\bm{z}^{\tau}_k)$.  The update of
    $\bm{z}_k^{\tau}$ terminates when either $\widehat{D}$ or
    $\|\bm{z}_k^{\tau+1}-\bm{z}_k^{\tau}\|$ is less than a small
    positive constant.  The other notations follow Figs.\ref{alg:LSLDGC}
    and~\ref{alg:mode-seeking}.}
   \end{figure}
  
  The algorithm of LSDRF essentially follows the same line as SCMS,
  which performs projected gradient ascent. By employing LSDDR, we
  obtain an estimate of $\Sig^{-1}(\bm{x})$ as
  \begin{align}
   \widehat{\bm{\Sig}}^{-1}(\bm{x}):=-\widehat{\bm{H}}(\bm{x})
  +\widehat{\bm{g}}(\bm{x})\widehat{\bm{g}}^{\top}(\bm{x}),
  \end{align}
  where we recall that $\widehat{g}_j(\bm{x})$ and
  $[\widehat{\bm{H}}(\bm{x})]_{ij}$ are LSDDR to $\partial_j
  p(\bm{x})/p(\bm{x})$ and $\partial_i \partial_j p(\bm{x})/p(\bm{x})$,
  respectively. Then, we obtain the orthogonal projector to the subspace
  as $\widehat{\bm{L}}(\bm{x})
  =\widehat{\V}(\bm{x})\widehat{\V}^{\top}(\bm{x})$ where
  $\widehat{\V}(\bm{x})$ consists of the $D-d$ eigenvectors associated
  with the $D-d$ largest eigenvalues of
  $\widehat{\bm{\Sig}}^{-1}(\bm{x})$. By replacing
  $\widehat{\bm{L}}_{\KDE}(\bm{x})$ and
  $\widehat{\bm{m}}_{\KDE}(\bm{x})$ in~\eqref{projeted-gradient-SCMS}
  with $\widehat{\bm{L}}(\bm{x})$ and $\widehat{\bm{m}}(\bm{x})$
  respectively, the following update rule for LSDRF is obtained by
  \begin{align}
   \bm{z}^{\tau+1}=\bm{z}^{\tau}
   +\widehat{\bm{L}}(\bm{z}^{\tau})\widehat{\bm{m}}(\bm{z}^{\tau}),
   \label{eqn:LSDRFupdate}
  \end{align}
  where $\widehat{\bm{m}}(\bm{x})
  =(\widehat{m}^{(1)}(\bm{x}),\widehat{m}^{(2)}(\bm{x}),\dots,\widehat{m}^{(D)}(\bm{x}))$
  is used in LSLDGC for mode-seeking whose definition is given
  in~\eqref{LS-mean-shift}.

  The implementation techniques of LSLDGC in
  Section~\ref{LSLDGCpractice} are inherited, but LSDRF performs
  projected gradient ascent instead of the gradient ascent: Whenever
  $\widehat{D}_{\widehat{\bm{g}}}[\bm{z}^{\tau+1}|\bm{z}^{\tau}]<0$ or
  $\exists j, f_j(\bm{z}^{(\tau)})\approx 0$, we perform the projected
  gradient ascent as
  \begin{align}
   \bm{z}^{\tau+1}=\bm{z}^{\tau}
   +\eta\widehat{\bm{L}}(\bm{z}^{\tau})\widehat{\bm{g}}(\bm{z}^{\tau}).
   \label{eqn:proj-grad-ascent}
  \end{align}
  The step size parameter $\eta$ is selected so that
  $\widehat{D}_{\widehat{\bm{g}}}[\bm{z}^{\tau}
  +\eta\widehat{\bm{L}}(\bm{z}^{\tau})\widehat{\bm{g}}(\bm{z}^{\tau})
  |\bm{z}^{\tau}]$ is maximized. The algorithm of LSDRF is summarized in
  Fig.\ref{alg:LSDRF}.\footnote{A MATLAB package of LSDRF is available
  at
  \url{https://sites.google.com/site/hworksites/home/software/lsdrf}.}
  The algorithm is essentially the same as LSLDGC based on the update
  rule~\eqref{updateelement} (Figs.~\ref{alg:LSLDGC}
  and~\ref{alg:mode-seeking}), where we only replace~\eqref{LS-grad}
  and~\eqref{eqn:grad-ascent} in LSLDGC with~\eqref{eqn:LSDRFupdate}
  and~\eqref{eqn:proj-grad-ascent} in LSDRF, respectively. Unlike
  clustering, for density ridge estimation, the starting points
  $\bm{z}^{\tau=0}$ are arbitrary, but in this paper, we set them at
  data samples $\bm{x}_i$ because data samples are fairly good starting
  points.
  \subsubsection{The Convergence Rate to the True Ridge}
  \label{ssec:ridgeconv}
  Here, we establish the convergence rate to understand how the
  estimated ridge approaches to the true ridge as $n$ increases. Based
  on LSDDR, the estimated ridge is defined as
  \begin{align*}
   \widehat{\mathcal{R}}
   &:=\{\bm{x}\in\R{D}~|~\|\widehat{\V}(\bm{x})\widehat{\V}(\bm{x})^{\top}
   \widehat{\bm{g}}(\bm{x})\|=0, \widehat{\eta}_{d+1}(\bm{x})<0\},
  \end{align*}
  where $\widehat{\eta}_{i}(\bm{x})$ denotes the $i$-th largest
  eigenvalue of $-\widehat{\Sig}^{-1}(\bm{x})$.

  In our analysis, we make the following assumptions:
  \begin{enumerate}[(A0)]   
   \item[(A0)] Kernel boundedness: $k(\bm{x},\bm{x}^{\prime})$ and
	       $\partial_j\partial_j^{\prime} k(\bm{x},\bm{x}^{\prime})$
	       for all $j$ are uniformly bounded, where
	       $\partial_j^{\prime}$ denotes the partial derivative with
	       respect to the $j$-th coordinate in $\bm{x}^{\prime}$.
	      
   \item[(A1)] Differentiability and boundedness: Let
	       $B_D(\bm{x},\delta)$ be the $D$-dimensional ball of
	       radius $\delta>0$ centered at $\bm{x}$ and let
	       $\mathcal{R}\oplus\delta:=
	       \cup_{\bm{x}\in\mathcal{R}}B_D(\bm{x},\delta)$. For all
	       $\bm{x}\in \mathcal{R}\oplus\delta$, the $|\bm{j}|$-th
	       order derivatives of $\log p(\bm{x})$ for
	       $|\bm{j}|=0,1,2,3$ exist and are bounded.
	       
   \item[(A2)] Eigengap: Assume that there exists $\kappa>0$ and
	       $\delta$ such that for all
	       $\bm{x}\in\mathcal{R}\oplus\delta$,
	       $\eta_{d+1}(\bm{x})<-\kappa$ and
	       $\eta_{d}(\bm{x})-\eta_{d+1}(\bm{x})>\kappa$, where
	       $\eta_{i}(\bm{x})$ denotes the $i$-th eigenvalue of
	       $\nabla\nabla \log p(\bm{x})$.

   \item[(A3)] Path smoothness: For each $\bm{x}\in\mathcal{R}\oplus\delta$,
	       \begin{align*}
		\|\bm{L}^{\perp}(\bm{x})\bm{g}(\bm{x})\|
		\cdot\|\bm{\Sig}^{-1\prime}(\bm{x})\|_{\max}
		<\frac{\kappa^2}{2D^{3/2}},
	       \end{align*}
	       where
	       $\bm{L}^{\perp}(\bm{x}):=\I_D-\bm{V}(\bm{x})\bm{V}(\bm{x})^{\top}$,
	       $\bm{\Sig}^{-1\prime}(\bm{x})
	       :=\nabla\vecope\left(\bm{\Sig}^{-1}(\bm{x})\right)$,
	       $\vecope(\cdot)$ denotes vectorization of matrices by
	       concatenating the columns, and
	       $\|\bm{A}\|_{\max}:=\max_{i,j}|[\bm{A}]_{ij}|$.  The
	       $(i,j)$-th element in
	       $\nabla\vecope\left(\bm{\Sig}^{-1}(\bm{x})\right)(\in\R{D^2\times
	       D})$ is given by
	       $\partial_j[\vecope\left(\bm{\Sig}^{-1}(\bm{x})\right)]_i$.
  \end{enumerate}
  Assumptions~(A2) and~(A3) are a straightforward modification of the
  assumptions in~\citet{genovese2014nonparametric} from the (non-log)
  density to the log-density. Assumption~(A2) indicates that the density
  ridge has a sharp and curvilinear shape in the subspace orthogonal to
  the ridge. Assumption (A3) indicates that
  $\|\bm{L}^{\perp}(\bm{x})\bm{g}(\bm{x})\|$ and
  $\|\bm{\Sig}^{-1\prime}(\bm{x})\|_{\max}$ are both bounded. Since
  $\bm{L}^{\perp}(\bm{x})$ is orthogonal to
  $\bm{V}(\bm{x})\bm{V}(\bm{x})^{\top}$ for all $\bm{x}$, the
  boundedness of $\|\bm{L}^{\perp}(\bm{x})\bm{g}(\bm{x})\|$ implies that
  the gradient $\bm{g}(\bm{x})$ is not too steep in the orthogonal
  subspace. The boundedness of $\|\bm{\Sig}^{-1\prime}(\bm{x})\|_{\max}$
  means that the third-order derivative is bounded and thus the subspace
  direction does not abruptly change, which implies that the (projected)
  gradient ascent path cannot be too
  wiggly~\citep[Section~2.2]{genovese2014nonparametric}. Note that
  Assumptions (A1)-(A3) are only valid in the neighborhood around the
  ridge.
  
  Let
  \begin{align*}
   \epsilon^{\prime}&:=\max_j\|g_j(\bm{x})-\widehat{g}_j(\bm{x})\|_{\infty},
   \qquad \epsilon^{\prime\prime}:=\max_{ij}\|
   [\bm{\Sig}^{-1}(\bm{x})]_{ij}
   -[\widehat{\bm{\Sig}}^{-1}(\bm{x})]_{ij}\|_{\infty},\\
   \epsilon^{\prime\prime\prime}&:=\max_{ij}\|
   [\bm{\Sig}^{-1\prime}(\bm{x})]_{ij}
   -[\widehat{\bm{\Sig}}^{-1\prime}(\bm{x})]_{ij}\|_{\infty}.
  \end{align*}
  To establish the convergence rate, we rely on two lemmas.  The first
  lemma is a simple modification of Theorem~4
  in~\citet{genovese2014nonparametric} to the log-density from the
  (non-log) density, and we use it without proof. The lemma states that
  if $\epsilon^{\prime}$, $\epsilon^{\prime\prime}$ and
  $\epsilon^{\prime\prime\prime}$ are sufficiently small, then the true
  and estimated ridges are close to each other:
  \begin{lemma}
   \label{Genovese} Suppose that (A1)-(A3) hold. Let
   $\psi:=\max\{\epsilon^{\prime}, \epsilon^{\prime\prime}\}$ and
   $\Psi:=\max\{\epsilon^{\prime},
   \epsilon^{\prime\prime},\epsilon^{\prime\prime\prime}\}$. When $\Psi$
  is sufficiently small, the following statements hold:
   \begin{enumerate}[(i)]
    \item[(i)] Conditions (A2) and (A3) hold for $\widehat{\bm{g}}$,
	       $\widehat{\bm{\Sig}}^{-1}$ and
	       $\widehat{\bm{\Sig}}^{-1\prime}$.

    \item[(ii)] $\Haus(\mathcal{R},\widehat{\mathcal{R}})$ is bounded
		as
		\begin{align}
		 \Haus(\mathcal{R},\widehat{\mathcal{R}})
		 = O(\psi). \label{Haus-dis-ridge}
		\end{align}
   \end{enumerate}   
  \end{lemma}
  The next lemma characterizes the convergence rates of
  $\epsilon^{\prime}$, $\epsilon^{\prime\prime}$ and
  $\epsilon^{\prime\prime\prime}$ when we employ LSDDR:
  \begin{lemma}
   \label{Epsiconv} Suppose that the assumptions in
   Theorem~\ref{theo:Hconvrate} and (A0) hold. When LSDDR is applied for
   density-derivative-ratio estimation,
   \begin{align}
    \epsilon^{\prime}&=O_{\mathrm{P}}\left(n^{-\min\left\{\frac{1}{4},
    \frac{\gamma}{2(\gamma+1)}\right\}}\right),\label{epsi'}\\
    \epsilon^{\prime\prime}&=O_{\mathrm{P}}\left(n^{-\min\left\{\frac{1}{4},
    \frac{\gamma}{2(\gamma+1)}\right\}}\right),\label{epsi''}\\
    \epsilon^{\prime\prime\prime}&=O_{\mathrm{P}}\left(n^{-\min\left\{\frac{1}{4},
    \frac{\gamma}{2(\gamma+1)}\right\}}\right)\label{epsi'''}.
   \end{align}
  \end{lemma}
  The proof is given in Appendix~\ref{app:Epsiconv}.
  
  Combining Lemma~\ref{Genovese} with Lemma~\ref{Epsiconv} yields the
  following theorem:
  \begin{theorem}
   \label{theo:Rconvrate} Suppose that the assumptions in
   Theorem~\ref{theo:Hconvrate} and (A0)-(A3) hold. Then,
   \begin{align}
    \Haus(\mathcal{R},\widehat{\mathcal{R}})
    =O_{\mathrm{P}}\left(n^{-\min\left\{\frac{1}{4},
    \frac{\gamma}{2(\gamma+1)}\right\}}\right).
    \label{Rconvrate}
   \end{align}
  \end{theorem}
  \begin{proof}
   Lemma~\ref{Epsiconv} ensures that
    $\psi=O_{\mathrm{P}}\left(n^{-\min\left\{\frac{1}{4},
    \frac{\gamma}{2(\gamma+1)}\right\}}\right)$. This completes the
    proof from~\eqref{Haus-dis-ridge}.
  \end{proof}
  \begin{remark*}
   \citet[Eq.(1)]{genovese2014nonparametric} established the following
   convergence rate based on KDE:
   \begin{align}
    \Haus(\mathcal{R},\widehat{\mathcal{R}}_{\text{KDE}})
    =O_{\mathrm{P}}\left(\left(\frac{\log n}{n}\right)^{\frac{2}{D+8}}
    \right), \label{GenoveseRidge}
   \end{align}
   where $\widehat{\mathcal{R}}_{\text{KDE}}$ denotes the estimated
   ridge by KDE. Comparison to our result is difficult, but the main
   difference is that the rate in~\eqref{GenoveseRidge} explicitly
   depends on data dimension $D$.
  \end{remark*}
  \section{Numerical Illustration on Mode-Seeking Clustering and Density Ridge Estimation}
  \label{sec:illust}
  This section experimentally illustrates the performance of the
  proposed methods for mode-seeking clustering and density ridge
  estimation on a variety of datasets.
  \subsection{Illustration on Clustering}
  \label{ssec:Illust-Clust}
  First, we illustrate the performance of LSLDGC both on artificial and
  benchmark datasets.
  \subsubsection{Artificial Datasets: LSLDGC vs MS}
  Here, we compare the performance of LSLDGC to MS with two different
  bandwidth selection methods:
  \begin{itemize}
   \item {\bf LSLDGC}: LSLDGC based on the update
	 rule~\eqref{updateelement}.  The width parameter $\sigma_j$ in
	 the Gaussian kernel and regularization parameter $\lambda_j$
	 were selected by cross-validation as in
	 Section~\ref{ssec:PracLSDDR}. We selected ten candidates of
	 $\sigma_j$ and $\lambda_j$ from $c_{\sigma}\times
	 \sigma^{(j)}_{\text{med}}$ ($0.5\leq c_{\sigma} \leq 5$) and
	 $10^{m}$ ($-3\leq m \leq 0$), respectively where
	 $\sigma^{(j)}_{\text{med}}$ is the median value of
	 $|x_i^{(j)}-x_k^{(j)}|$ with respect to $i$ and $k$.

   \item {\bf LSLDGC$_{\text{CW}}$}: LSLDGC based on the coordinate-wise
	 update rule~\eqref{CWupdate}. The same cross-validation was
	 performed as above.
	 
   \item {\bf MS$_{\text{LS}}$}: The bandwidth parameter $h$ was
	 cross-validated based on the standard \emph{integrated squared
	 error}. We selected ten candidates of $h$ from $10^{l}\times
	 h_{\text{med}}$ ($-1.5\leq l \leq 0$) where $h_{\text{med}}$ is
	 the median value of $|x_i^{(j)}-x_k^{(j)}|$ with respect to
	 $i$, $j$ and $k$.
	 
   \item {\bf MS$_{\text{NR}}$}: The bandwidth parameter $h$ was
	 determined by
	 \begin{align*}
	  \bar{S}_n\left(\frac{4}{D+4}\right)^{\frac{1}{D+6}}
	  n^{-\frac{1}{d+6}},
	 \end{align*}
	 where $\bar{S}_n=\frac{1}{nD}\sum_{j=1}^D\sum_{i=1}^n
	 (x^{(j)}_i-\bar{x}^{(j)})^2$ and
	 $\bar{x}^{(j)}=\frac{1}{n}\sum_{i=1}^n x_i^{(j)}$. This
	 bandwidth parameter was used in~\citet{chen2016comprehensive}
	 and a slight modification of the normal reference
	 rule~\citep{silverman1986density}.
  \end{itemize}
  
  First, we generated three kinds of two-dimensional data as follows:
  \begin{enumerate}[(a)]    
   \item[(a)] {\bf Three Gaussian blobs} (Fig.\ref{fig:ClustArt}(a)):
	      Each data sample was drawn from a mixture of three
	      Gaussians with means $(0,1)^{\top}$, $(-1,-1)^{\top}$ and
	      $(1,-1)^{\top}$, and covariance matrices $0.1\I_{2}$. The
	      mixing coefficients were $0.4, 0.3, 0.3$, respectively.
         
   \item[(b)] {\bf Two curves} (Fig.\ref{fig:ClustArt}(d)): Two curves
	      are generated as $(x^{(1)}, x^{(2)})=(\cos(\pi
	      t^{(1)}),\sin(\pi t^{(1)}))^{\top}$ and $(x^{(1)},
	      x^{(2)}) =(-\cos(\pi t^{(2)})+1, -\sin(\pi
	      t^{(2)}))^{\top}$ where $t^{(1)}$ and $t^{(2)}$ are
	      independently drawn from the Gaussian density with mean
	      $0.5$ and standard deviation $0.15$. Then, Gaussian noise
	      with covariance matrix $0.1\I_{2}$ was added to these
	      curves. The numbers of data samples for both curves were
	      approximately same.
         
   \item[(c)] {\bf Two curves \& a Gaussian blob}
	      (Fig.\ref{fig:ClustArt}(g)): Data samples from the
	      Gaussian density with mean $0$ and standard deviation
	      $0.1$ were added to the two curves similarly generated as
	      in (b). The number of samples for the two curves was same,
	      and for the Gaussian blob, we set the number at $n/3$
	      approximately.
  \end{enumerate}
  When higher-dimensional data were generated, we simply appended
  Gaussian variables with mean $0$ and standard deviation $0.1$ to the
  two-dimensional data.  Clustering performance was measured by the
  adjusted Rand index (ARI)~\citep{JoC:Hubert+Arabie:1985}: ARI takes a
  value less than or equal to one, a larger value indicates a better
  clustering result, and when a clustering result is perfect, the ARI
  value equals to one.
  
  Fig.\ref{fig:ClustArt}(b,e,h) clearly indicates the advantage of our
  clustering methods over MS: Both LSLDGC and LSLDGC$_{\text{CW}}$
  significantly outperform MS$_{\text{LS}}$ and MS$_{\text{NR}}$
  particularly for higher-dimensional data. When the dimensionality of
  data is low, MS$_{\text{NR}}$ performs well to all kinds of
  datasets. However, the ARI values of both MS$_{\text{LS}}$ and
  MS$_{\text{NR}}$ quickly approach zero as the dimensionality of data
  increases. These unsatisfactory results seem to be due to the fact
  that the bandwidth selection in KDE is more difficult for
  high(er)-dimensional data. Thus, our direct approach would be more
  suitable particularly for high(er)-dimensional data.

  Both LSLDGC and LSLDGC$_{\text{CW}}$ keep the ARI values high on a
  wide range of sample sizes (Fig.\ref{fig:ClustArt}(c,f,i)). The
  performance of MS$_{\text{NR}}$ is improved as $n$ increases. However,
  MS$_{\text{LS}}$ performs rather worse for large(r) datasets. The
  least-squares cross-validation often suggests small bandwidth
  parameters for large(r) datasets, which make the estimated density
  unsmooth. Thus, the estimated density can include a lot of spurious
  modes with small peaks even if it was good in terms of density
  estimation. This also supports that our direct estimation is a more
  appropriate approach.
  \begin{figure}[!t]
   \begin{center}
    \subfigure[Gaussian
    blobs]{\includegraphics[width=0.27\textwidth,clip]{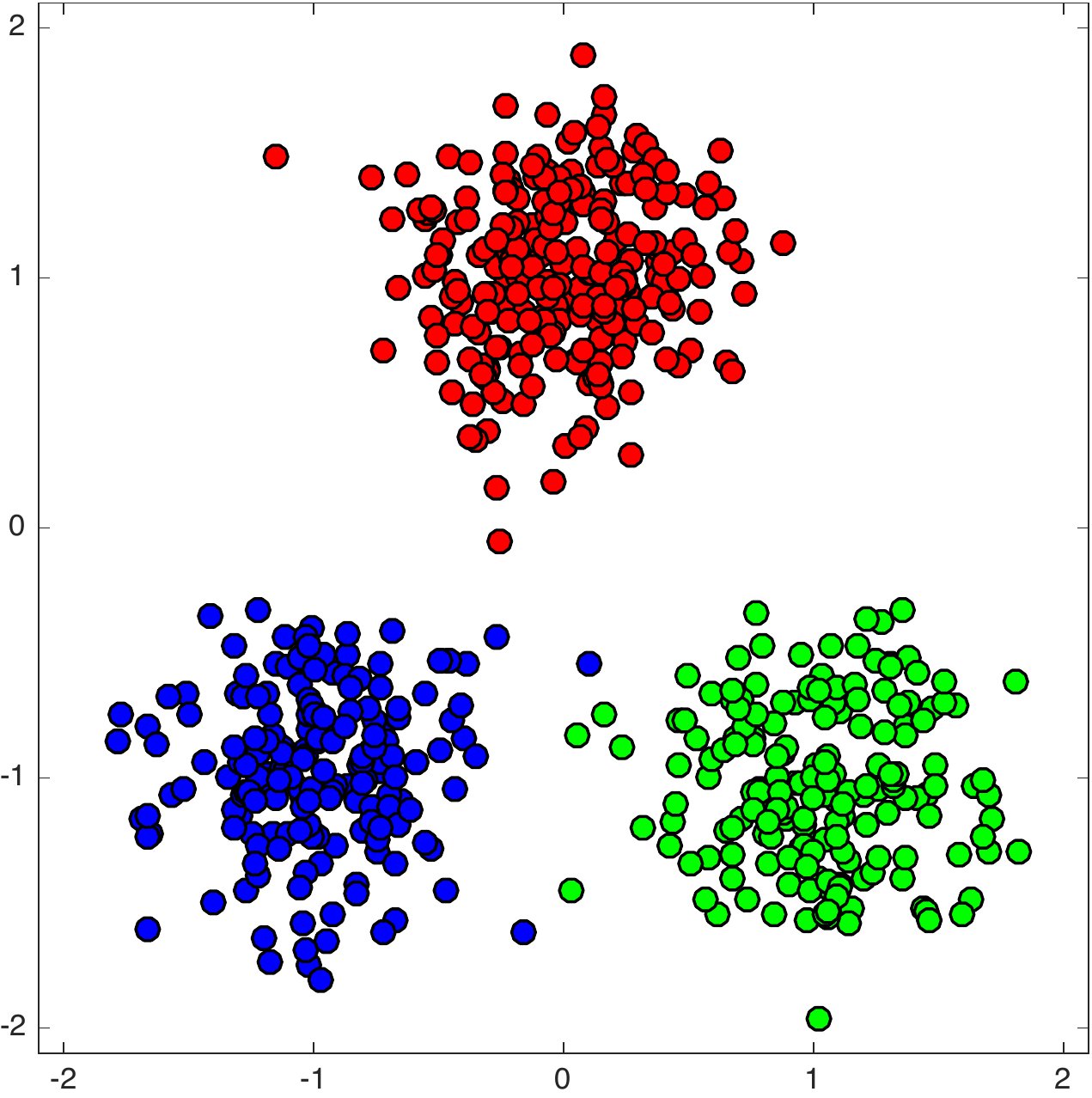}}
    \subfigure[$n=600$]{\includegraphics[width=0.34\textwidth,clip]{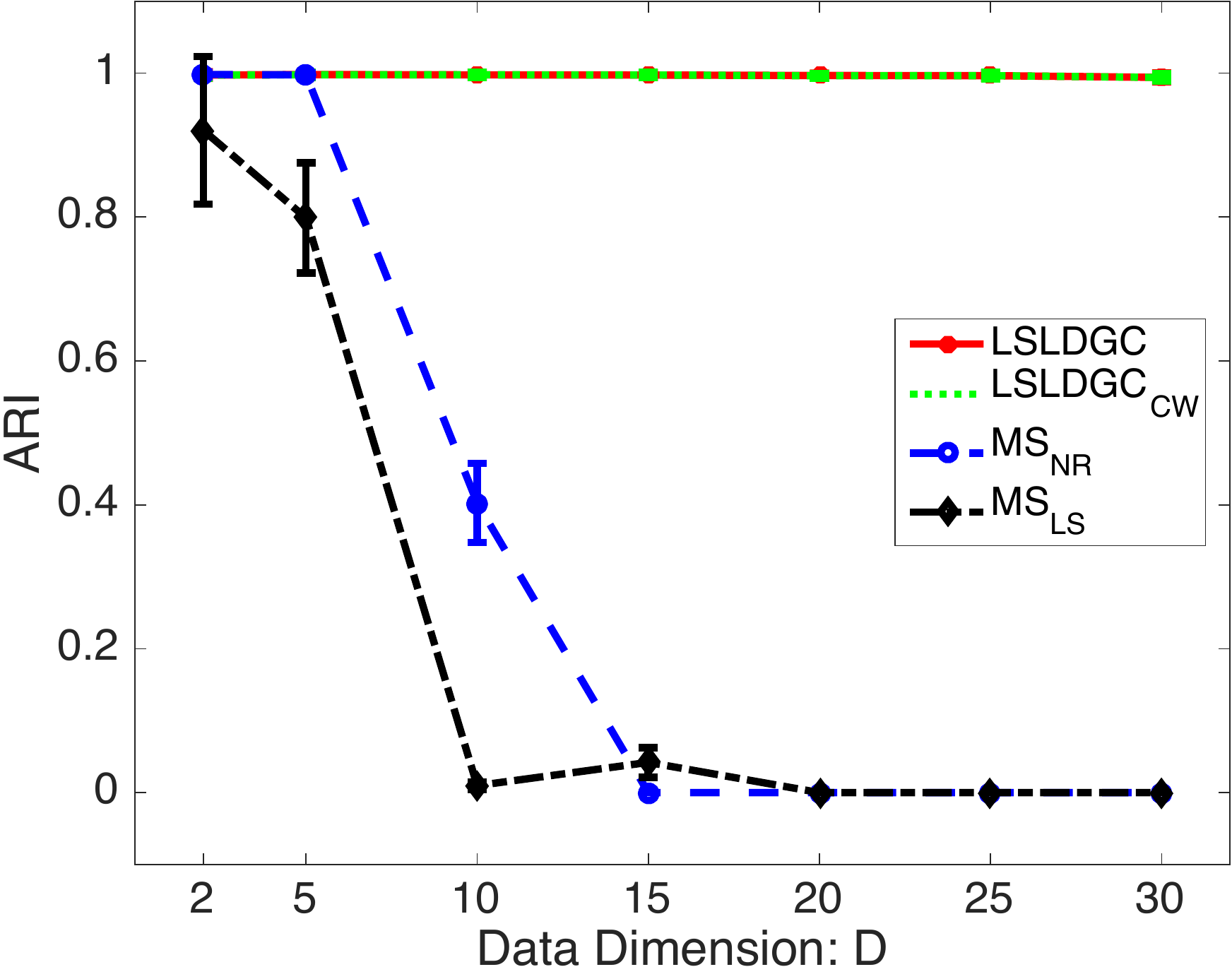}}
    \subfigure[$D=10$]{\includegraphics[width=0.34\textwidth,clip]{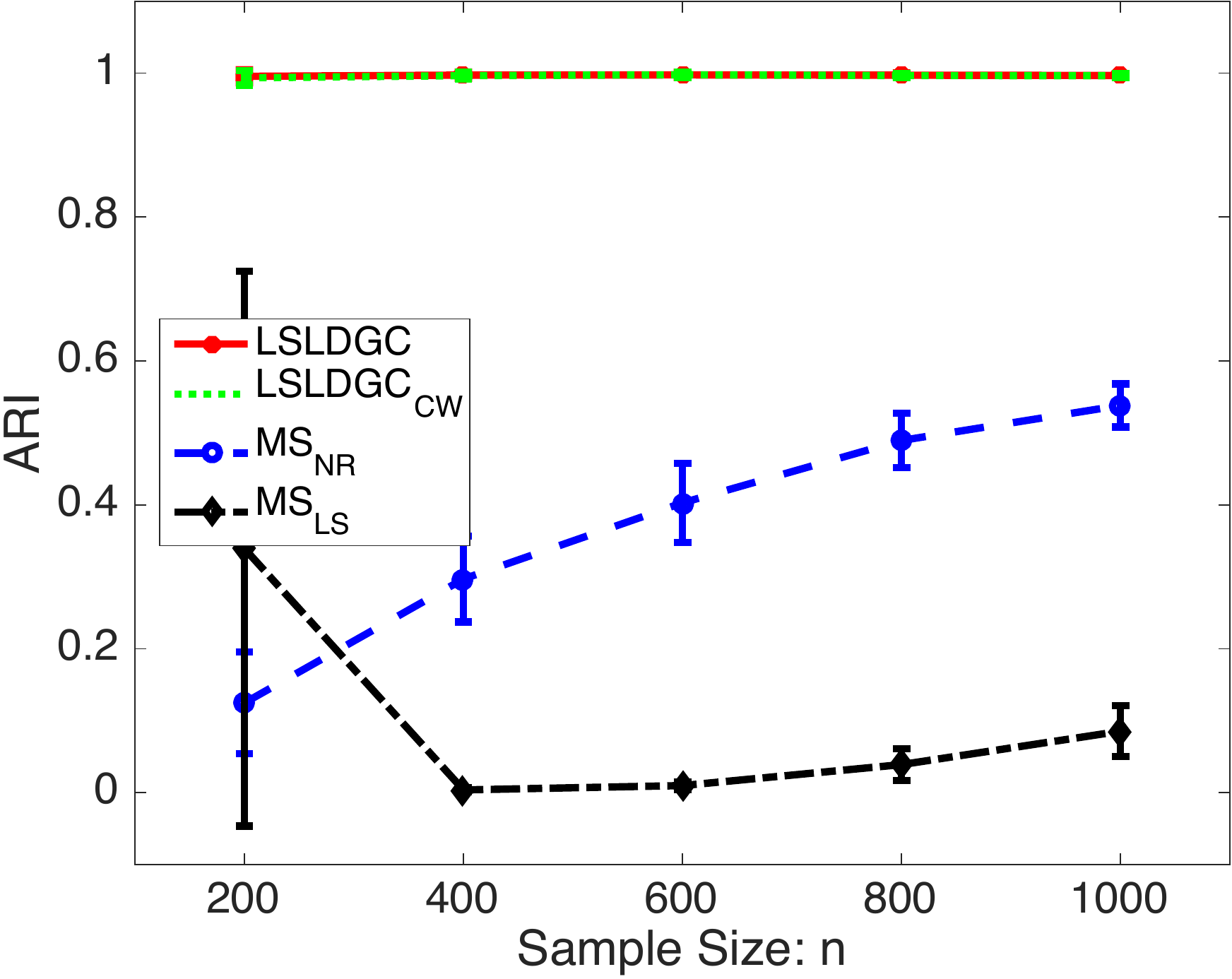}}\\
    \subfigure[Two
    curves]{\includegraphics[width=0.27\textwidth,clip]{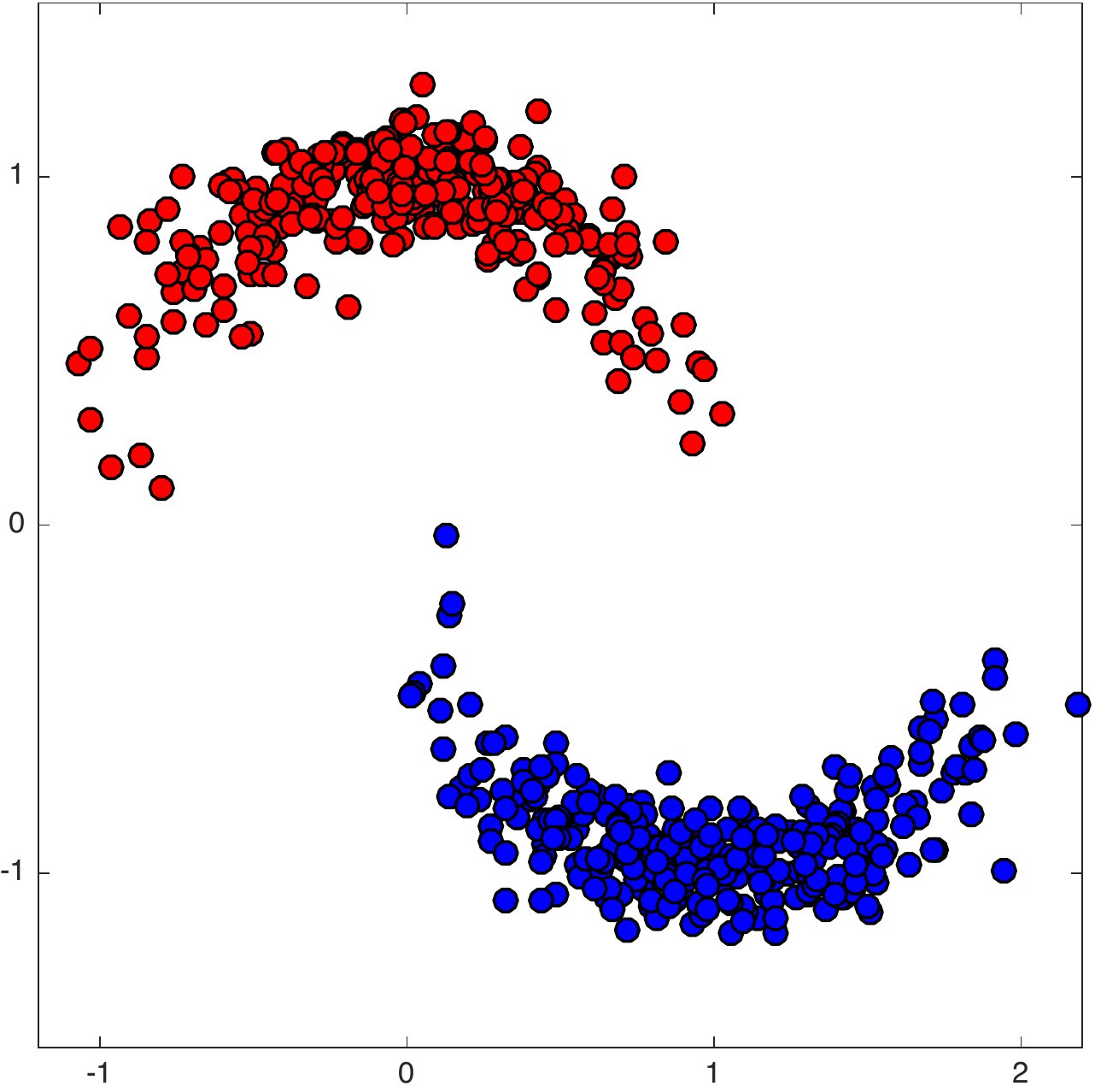}}
    \subfigure[$n=600$]{\includegraphics[width=0.34\textwidth,clip]{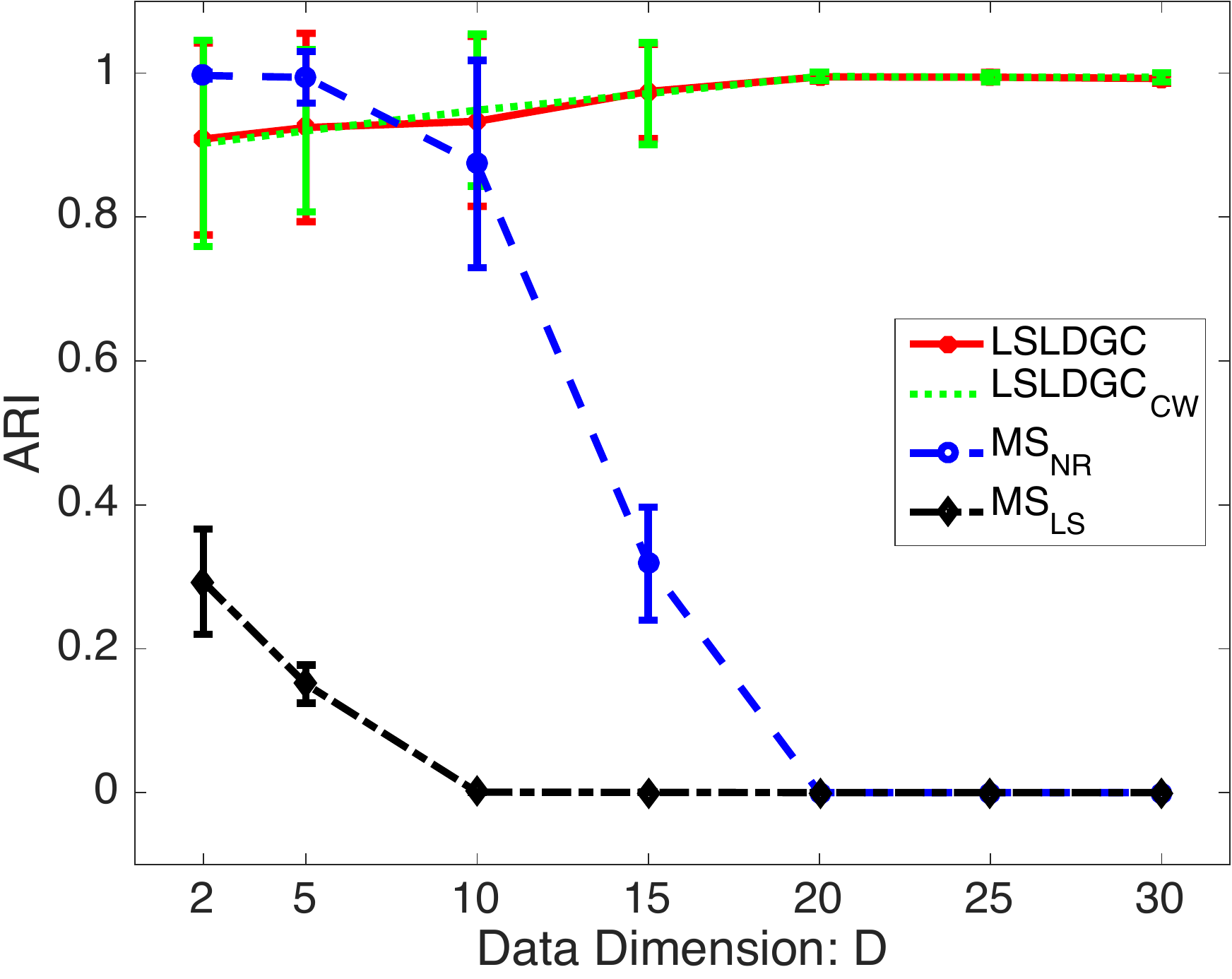}}
    \subfigure[$D=10$]{\includegraphics[width=0.34\textwidth,clip]{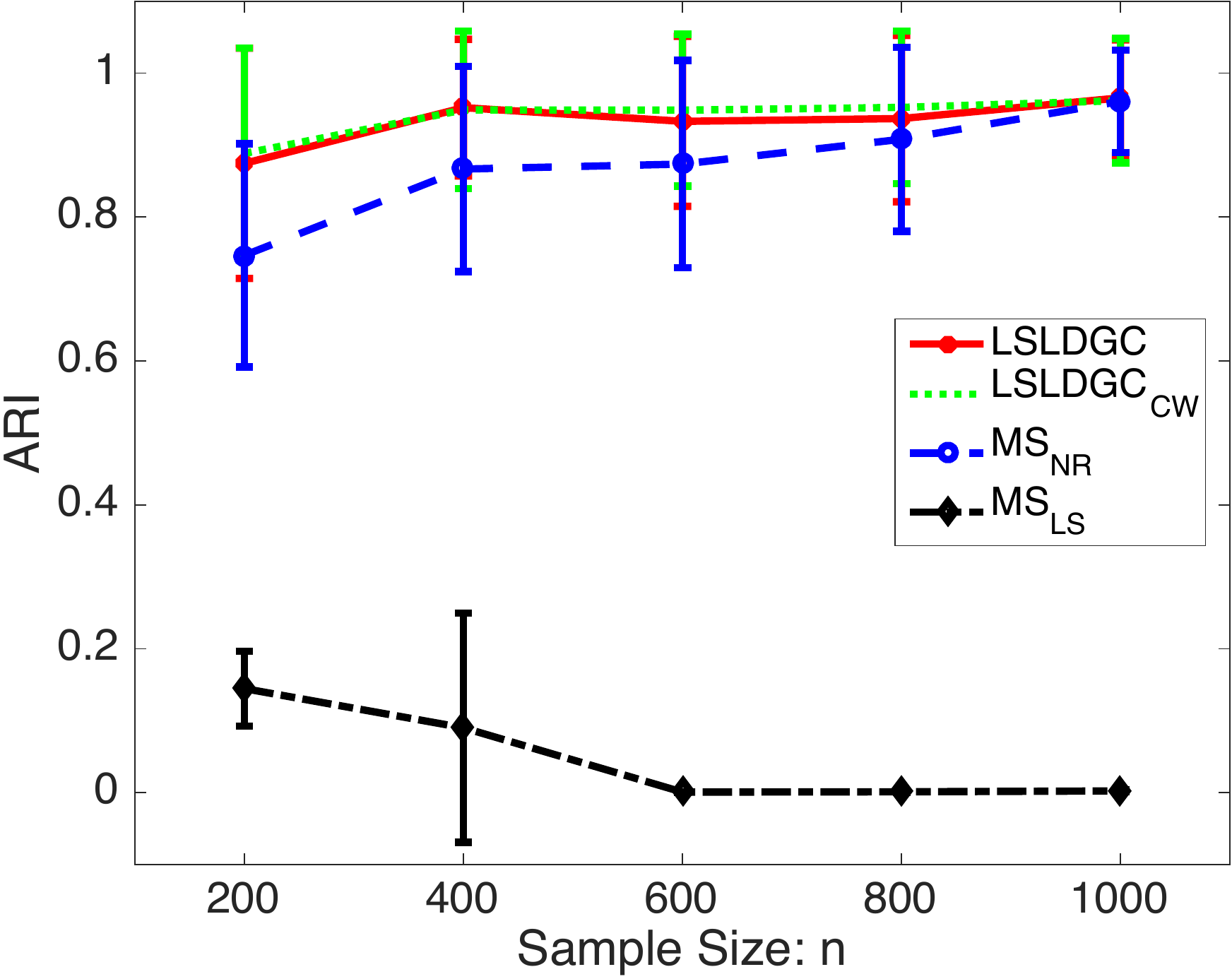}}\\
    \subfigure[Two curves \& a Gaussian
    blob]{\includegraphics[width=0.27\textwidth,clip]{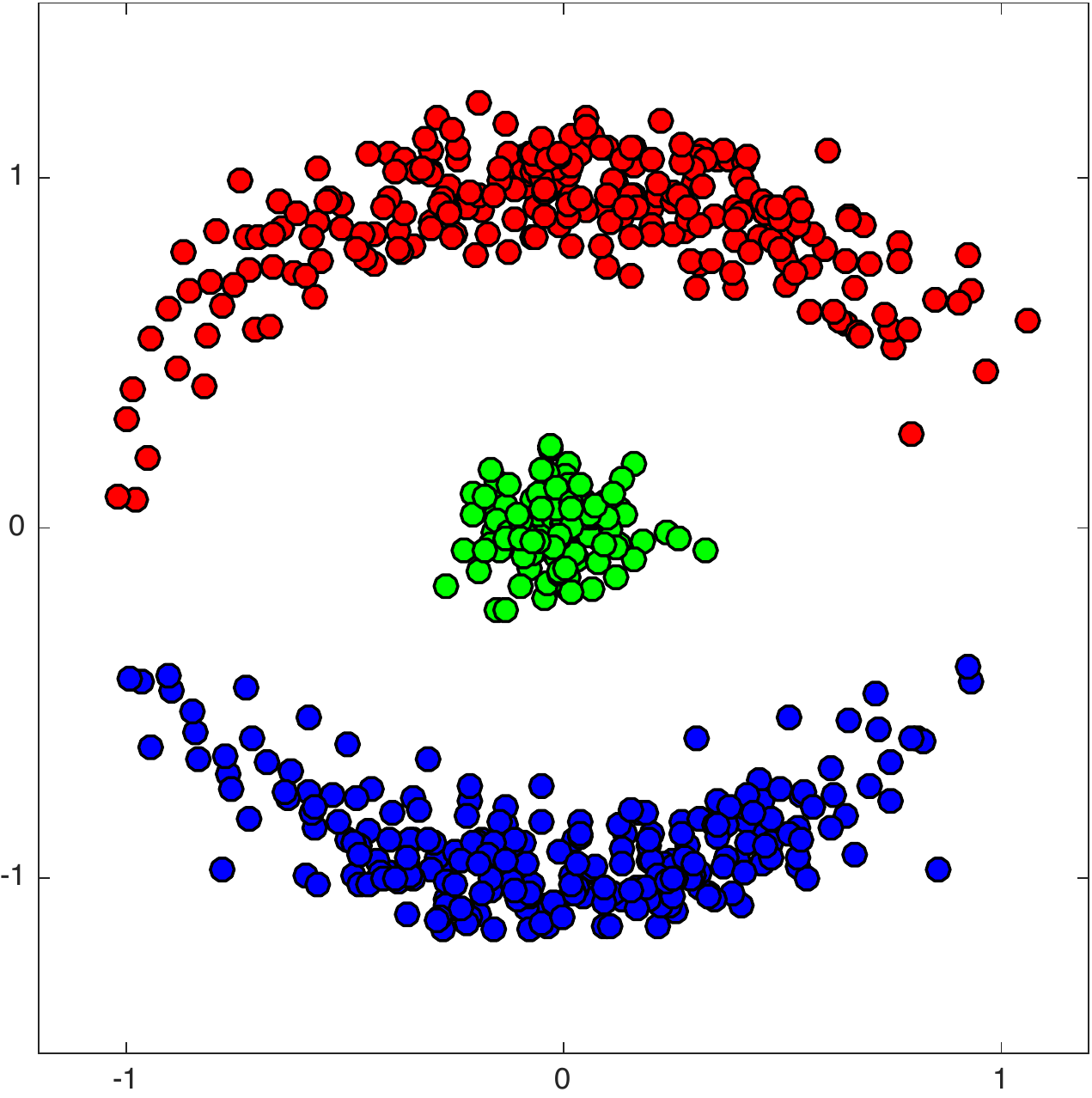}}
    \subfigure[$n=600$]{\includegraphics[width=0.34\textwidth,clip]{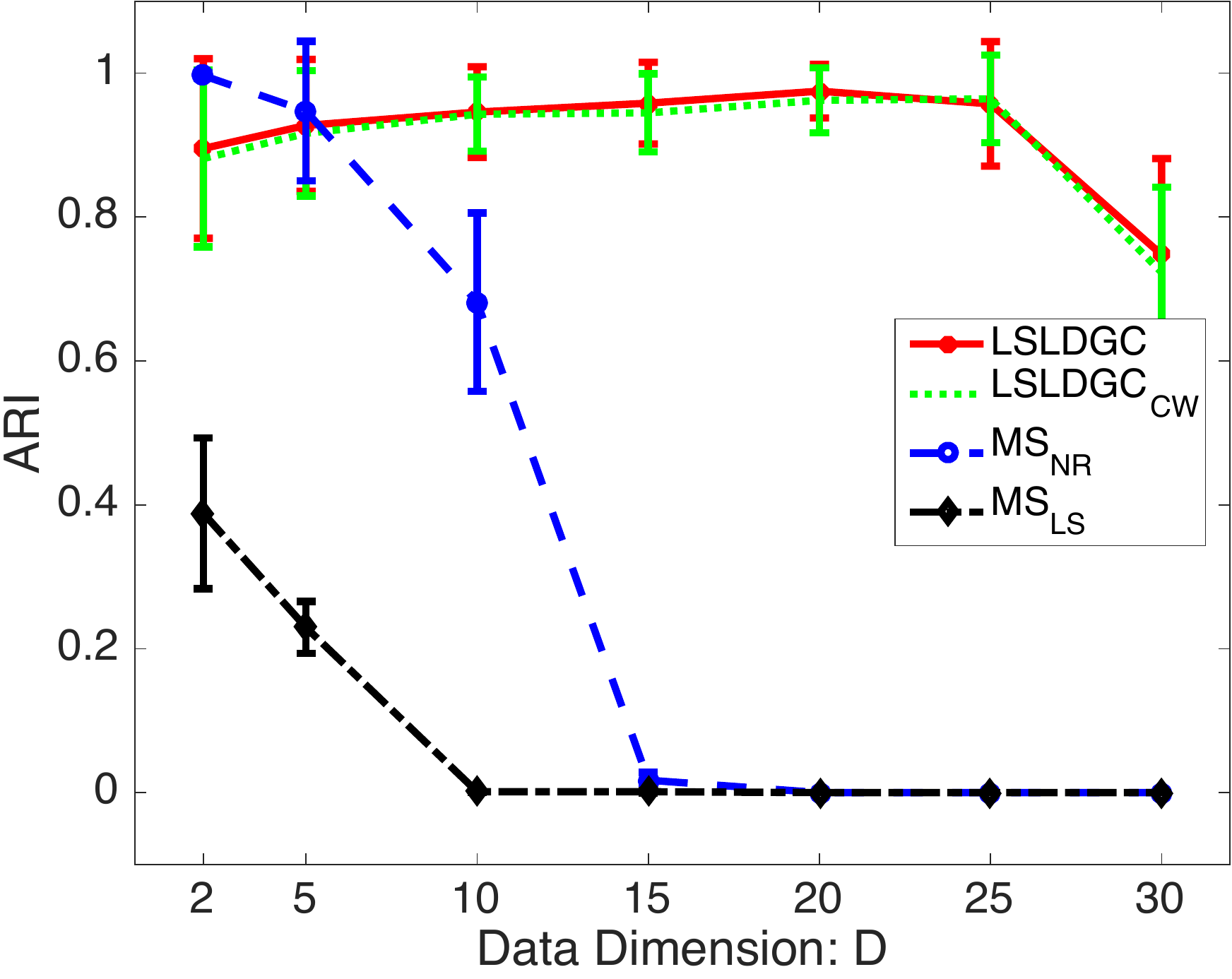}}
    \subfigure[$D=10$]{\includegraphics[width=0.34\textwidth,clip]{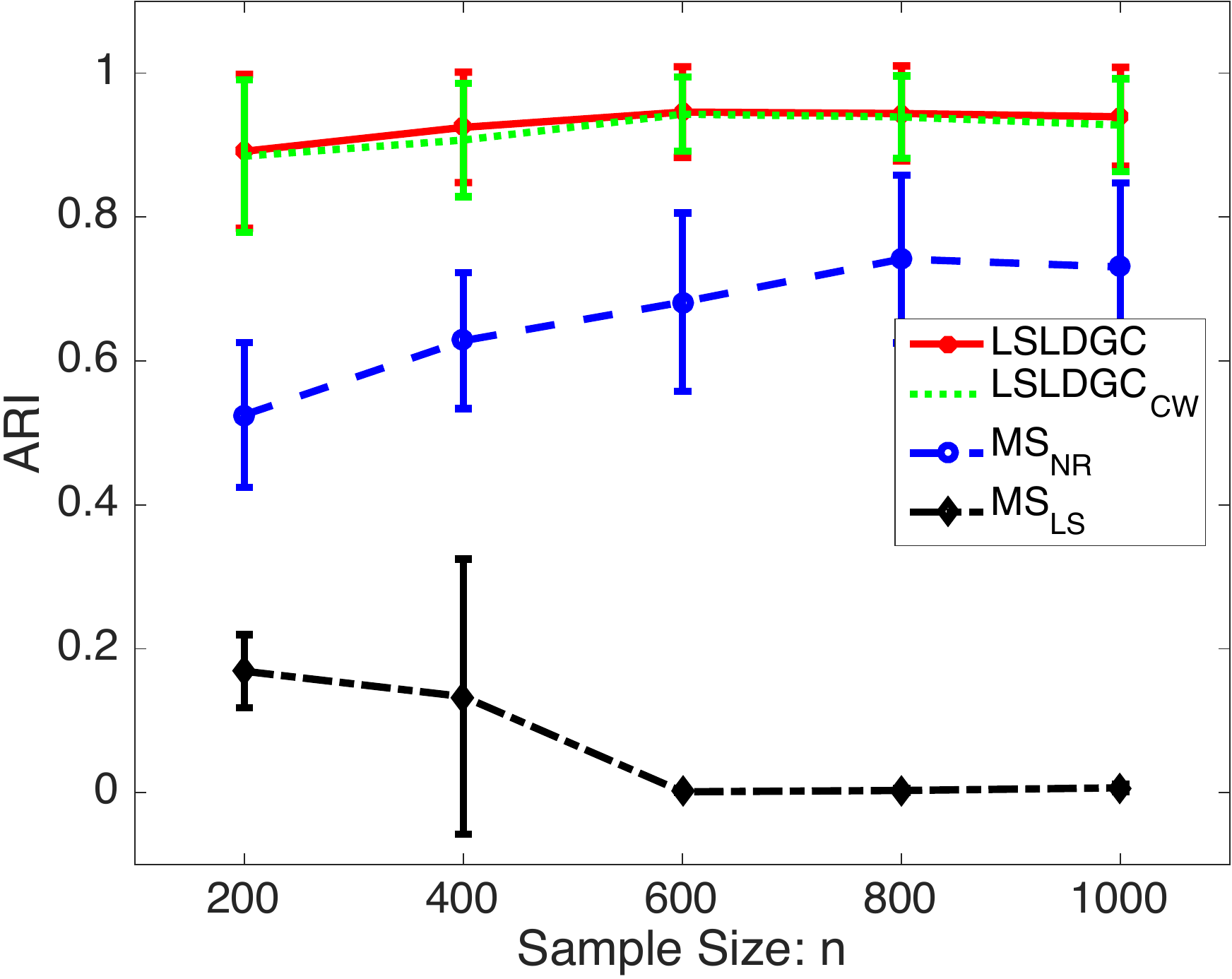}}
    \caption{\label{fig:ClustArt} Clustering performance on artificial
    data. Each point and error bar denote the average and standard
    deviation of ARI over $50$ runs, respectively.}
   \end{center}
  \end{figure}
  \subsubsection{Benchmark Datasets}
  \begin{table}[t]
   \caption{\label{tab:benchclust} The average and standard deviation of
   ARI values over $50$ runs. A larger value means a better
   result. Numbers in the parentheses are standard deviations. The best
   and comparable methods judged by the unpaired t-test at the
   significance level $1\%$ are described in boldface.}
   \begin{center}
    \begin{tabular}{@{\ }c@{\ }|@{\ }c@{\ }|@{\ }c@{\ }|@{\ }c@{\ }|@{\ }c@{\ }|@{\ }c@{\ }}
     \hline
     \multicolumn{6}{c}{Banknote $(D,n,c)=(4,100,2)$}\\
     \hline
     LSLDGC & LSLDGC$_{\text{CW}}$ & MS$_{\text{LS}}$ & MS$_{\text{NR}}$ & SC & KM \\
     \hline
     {\bf 0.165(0.059)} & {\bf 0.169(0.055)} 
	 & 0.036(0.014) & {\bf 0.167(0.147)} & 0.054(0.064) & 0.039(0.051)\\
     \hline
    \end{tabular}
   \end{center}
   \begin{center}
    \begin{tabular}{@{\ }c@{\ }|@{\ }c@{\ }|@{\ }c@{\ }|@{\ }c@{\ }|@{\ }c@{\ }|@{\ }c@{\ }}
     \hline
     \multicolumn{6}{c}{Accelerometry $(D,n,c)=(5,300,3)$}\\
     \hline
     LSLDGC & LSLDGC$_{\text{CW}}$ & MS$_{\text{LS}}$ & MS$_{\text{NR}}$ & SC & KM \\
     \hline
     {\bf 0.628(0.058)} & {\bf 0.624(0.065)} & 0.029(0.007) 
	     & 0.500(0.041) & 0.226(0.271) & 0.499(0.023) \\
     \hline     
    \end{tabular}
   \end{center}
   \begin{center}
    \begin{tabular}{@{\ }c@{\ }|@{\ }c@{\ }|@{\ }c@{\ }|@{\ }c@{\ }|@{\ }c@{\ }|@{\ }c@{\ }}
     \hline
     \multicolumn{6}{c}{Olive oil $(D,n,c)=(8,200,9)$}\\
     \hline
     LSLDGC & LSLDGC$_{\text{CW}}$ & MS$_{\text{LS}}$ & MS$_{\text{NR}}$ & SC & KM \\
     \hline
     {\bf 0.717(0.081)} & {\bf 0.728(0.062)} & 0.020(0.019) 
	     & {\bf 0.756(0.078)} & 0.552(0.060) & 0.618(0.063)\\
     \hline
    \end{tabular}
   \end{center}
   \begin{center}
    \begin{tabular}{@{\ }c@{\ }|@{\ }c@{\ }|@{\ }c@{\ }|@{\ }c@{\ }|@{\ }c@{\ }|@{\ }c@{\ }}
     \hline
     \multicolumn{6}{c}{Vowel $(D,n,c)=(10,110,11)$}\\
     \hline
     LSLDGC & LSLDGC$_{\text{CW}}$ & MS$_{\text{LS}}$ & MS$_{\text{NR}}$ & SC & KM \\
     \hline
     0.147(0.037) & 0.139(0.032) & 0.017(0.010) 
	     & 0.133(0.026) & 0.145(0.027) & {\bf 0.180(0.027)} \\
     \hline
    \end{tabular}
   \end{center}
   \begin{center}
    \begin{tabular}{@{\ }c@{\ }|@{\ }c@{\ }|@{\ }c@{\ }|@{\ }c@{\ }|@{\ }c@{\ }|@{\ }c@{\ }}
     \hline
     \multicolumn{6}{c}{Sat-image $(D,n,c)=(36,120,6)$}\\
     \hline
     LSLDGC & LSLDGC$_{\text{CW}}$ & MS$_{\text{LS}}$ & MS$_{\text{NR}}$ & SC & KM \\
     \hline
     {\bf 0.427(0.072)} & {\bf 0.422(0.073)} & 0.000(0.000) 
	     & 0.343(0.063) & {\bf 0.418(0.056)} & {\bf 0.434(0.052)}\\
     \hline
    \end{tabular}
   \end{center}
   \begin{center}
    \begin{tabular}{@{\ }c@{\ }|@{\ }c@{\ }|@{\ }c@{\ }|@{\ }c@{\ }|@{\ }c@{\ }|@{\ }c@{\ }}
     \hline
     \multicolumn{6}{c}{Speech $(D,n,c)=(50,400,2)$}\\
     \hline
     LSLDGC & LSLDGC$_{\text{CW}}$ & MS$_{\text{LS}}$ & MS$_{\text{NR}}$ & SC & KM \\
     \hline     
     {\bf 0.146(0.063)} & {\bf 0.147(0.054)} & 0.000(0.000) 
	     & 0.000(0.000) & 0.004(0.004) & 0.002(0.004) \\
     \hline
    \end{tabular}    
   \end{center}
  \end{table}
  Next, we investigate the performance of LSLDGC over the following
  benchmark datasets:
  \begin{itemize}
   \item {\it Banknote} $(D=4,
	 n=100,~\mathrm{and}~c=2)$~\citep{UCIBench}\footnote{\url{https://archive.ics.uci.edu/ml/datasets/banknote+authentication\#}}:
	 This dataset consists of four-dimensional features from $400$
	 by $400$ images for genuine and forged banknote-like
	 specimens. The features were extracted by wavelet
	 transformation. We randomly chose $50$ samples from each of the
	 two classes.
	 
   \item {\it Accelerometry} $(D=5,
	 n=300,~\mathrm{and}~c=3)$\footnote{\url{http://alkan.mns.kyutech.ac.jp/web/data.html}}:
	 The ALKAN dataset contains $3$-axis (i.e., x-, y-, and z-axes)
	 accelerometric data. During the data collection, subjects were
	 instructed to perform walking, running, and standing up. After
	 segmenting each data stream into windows, five
	 orientation-invariant-features were computed from each
	 window~\citep{sugiyama2014information}. We randomly chose $100$
	 samples from each of the three classes.

   \item {\it Olive oil} $(D=8,
	 n=200,~\mathrm{and}~c=9)$~\citep{forina1983classification}.
	 This dataset was obtained from the R
	 software.\footnote{\url{https://artax.karlin.mff.cuni.cz/r-help/library/pdfCluster/html/oliveoil.html}}
	 The dataset includes eight chemical measurements on different
	 specimen of olive oil produced in nine regions in Italy. We
	 randomly chose $200$ samples.
	 
   \item {\it Vowel} $(D=10,
	 n=110,~\mathrm{and}~c=11)$~\citep{turney1993robust,UCIBench}\footnote{\url{https://archive.ics.uci.edu/ml/datasets/Connectionist+Bench+(Vowel+Recognition+-+Deterding+Data)}}:
	 This consists utterance data for eleven vowels of British
	 English. Each utterance is expressed by a ten-dimensional
	 vector. We randomly chose $10$ samples from each of the eleven
	 classes.
	 
   \item {\it Sat-image} $(D=36,
	 n=120,~\mathrm{and}~c=6)$~\citep{UCIBench}\footnote{\url{https://archive.ics.uci.edu/ml/datasets/Statlog+(Landsat+Satellite)}}:
	 The dataset contains the multi-spectral values of pixels in
	 $3\times 3$ neighborhoods in a satellite image with six
	 classes. We randomly chose $20$ samples from each of the six
	 classes.

   \item {\it Speech} $(D=50, n=400,~\mathrm{and}~c=2)$. An in-house
	 speech dataset~\citep{sugiyama2014information}, which contains
	 short utterance samples recorded from $2$ male subjects
	 speaking in French with sampling rate $44.1$kHz.
	 50-dimensional line spectral frequencies
	 vectors~\citep{kain1998spectral} were computed from each
	 utterance sample. We randomly chose $200$ samples from each of
	 the two classes.
  \end{itemize}
  As preprocessing, each data sample was standardized by the sample mean
  and standard deviation in coordinate-wise manner.  For comparison, we
  applied \emph{k-means clustering}
  (KM)~\citep{BerkeleySymp:MacQueen:1967} and \emph{spectral clustering}
  (SC)~\citep{ng2001spectral,shi2000normalized} to the same
  datasets. Since KM and SC require to input the number of clusters, we
  set it at the correct number.
  
  As seen in the illustration on artificial data, when the
  dimensionality of data is low, the performance of LSLDGC,
  LSLDGC$_{\text{CW}}$ and MS$_{\text{NR}}$ is comparable, but LSLDGC
  and LSLDGC$_{\text{CW}}$ significantly work better than
  MS$_{\text{NR}}$ to higher-dimensional datasets (sat-image and speech
  datasets). KM and SC have prior information about the number of
  clusters. Nonetheless, the performance of LSLDGC and
  LSLDGC$_{\text{CW}}$ are often better than KM and SC.
  
  From the results of both the artificial and benchmark datasets, we
  conclude that LSLDGC and LSLDGC$_{\text{CW}}$ are advantageous to
  relatively high-dimensional data.
  \subsection{Illustration on Density Ridge Estimation}
  \label{ssec:illustridge}
  Next, we illustrate the performance of LSDRF, and compare LSDRF with
  SCMS both on artificial and standard benchmark datasets.
   \subsubsection{Artificial Data: LSDRF vs SCMS}
   \label{ridgeart}
   \begin{figure}[t]
    \begin{center}
     \subfigure{\includegraphics[width=0.32\textwidth,clip]
     {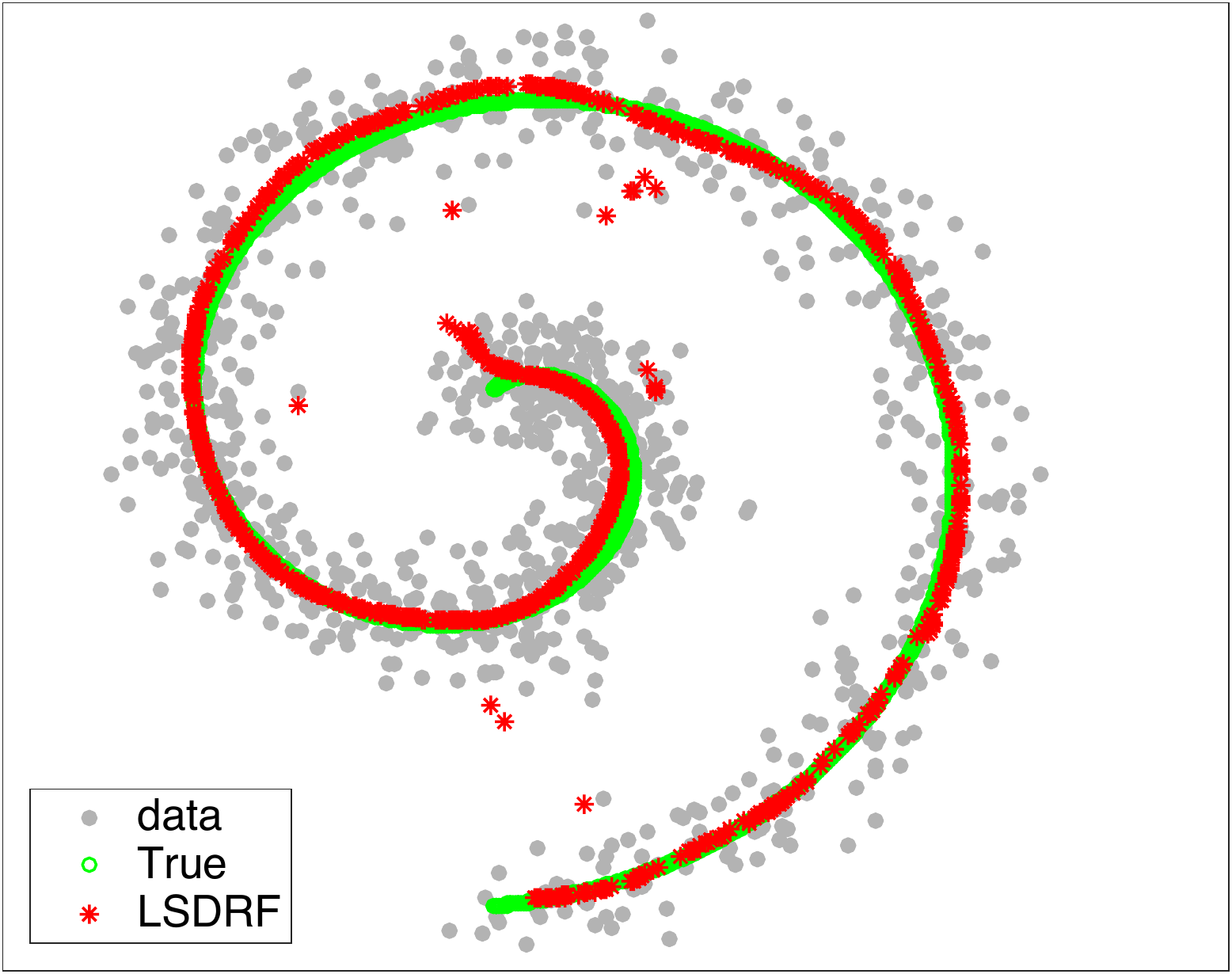}}
     \subfigure{\includegraphics[width=0.32\textwidth,clip]
     {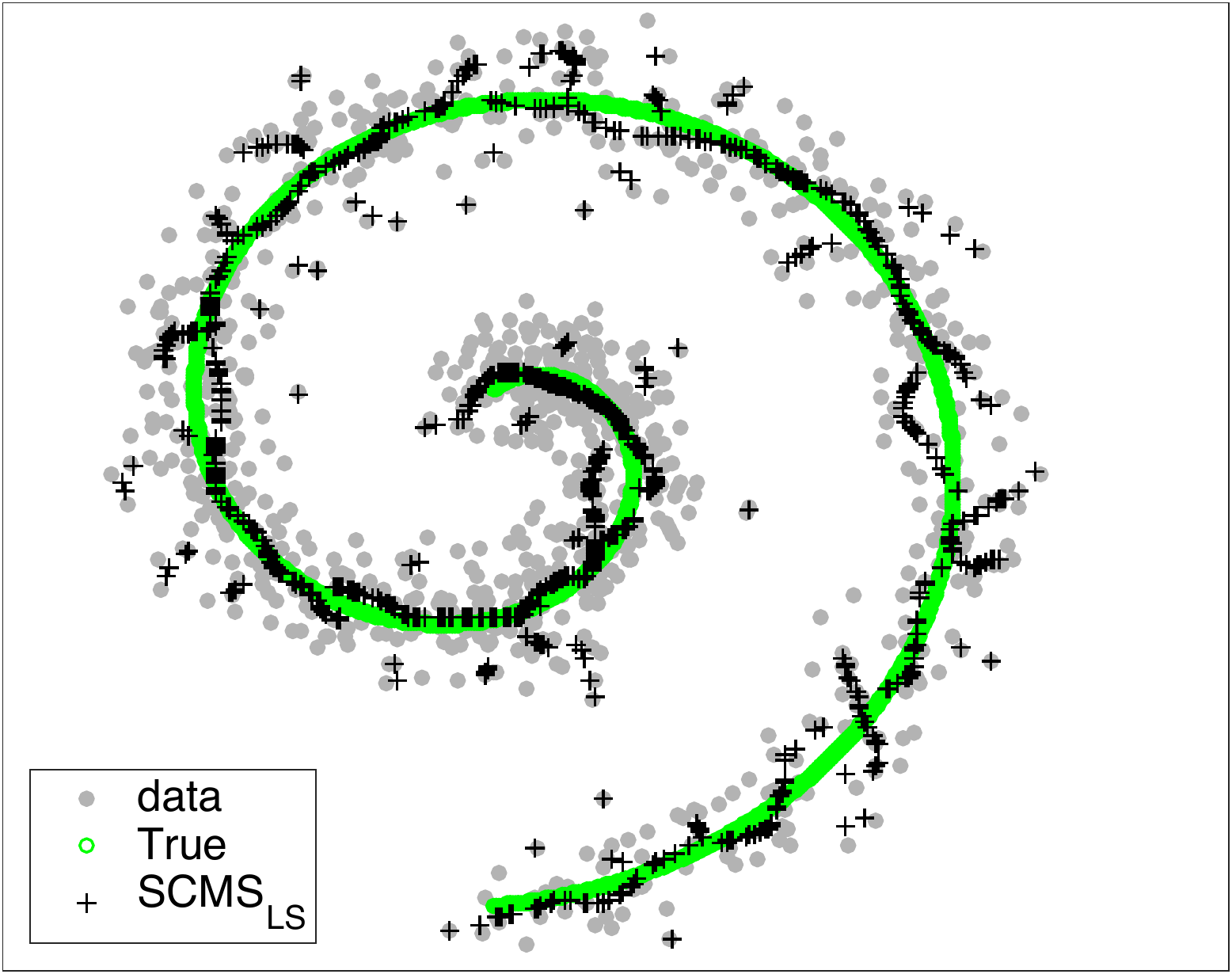}}
     \subfigure{\includegraphics[width=0.32\textwidth,clip]
     {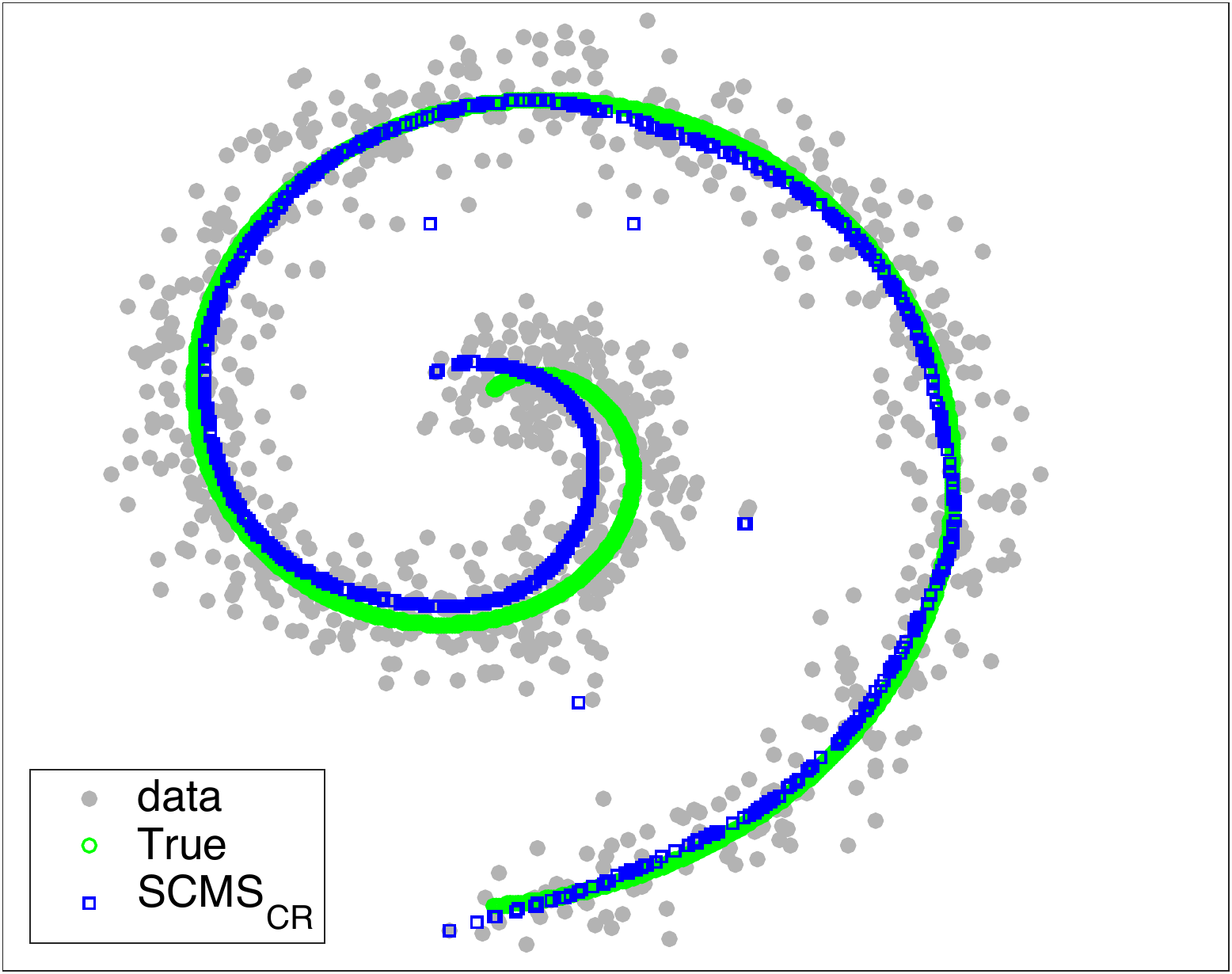}}
     \setcounter{subfigure}{0}
     \subfigure[LSDRF]{\includegraphics[width=0.32\textwidth,clip]
     {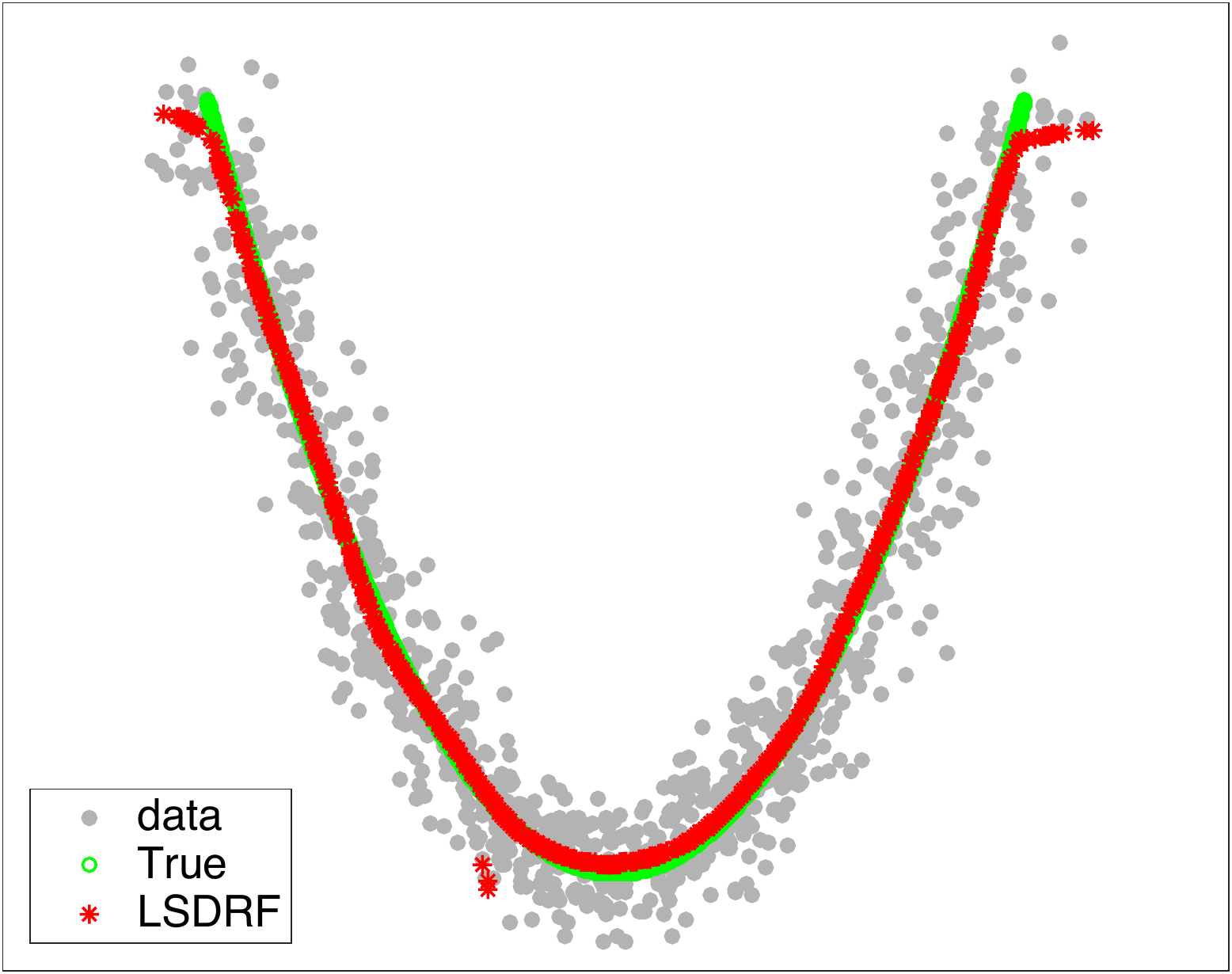}}
     \subfigure[SCMS$_{\text{LS}}$]{\includegraphics[width=0.32\textwidth,clip]
     {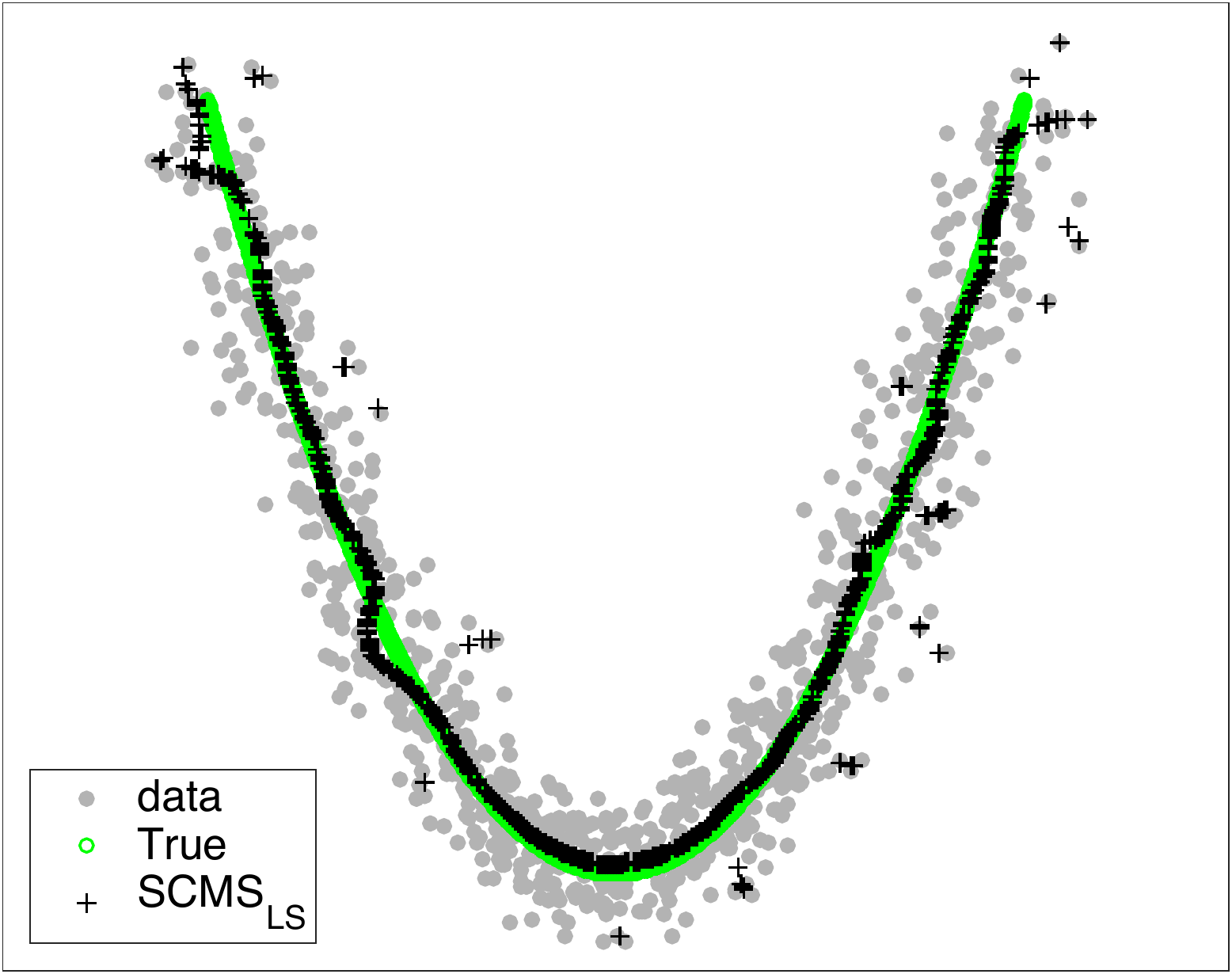}} 
     \subfigure[SCMS$_{\text{CR}}$]{\includegraphics[width=0.32\textwidth,clip]
     {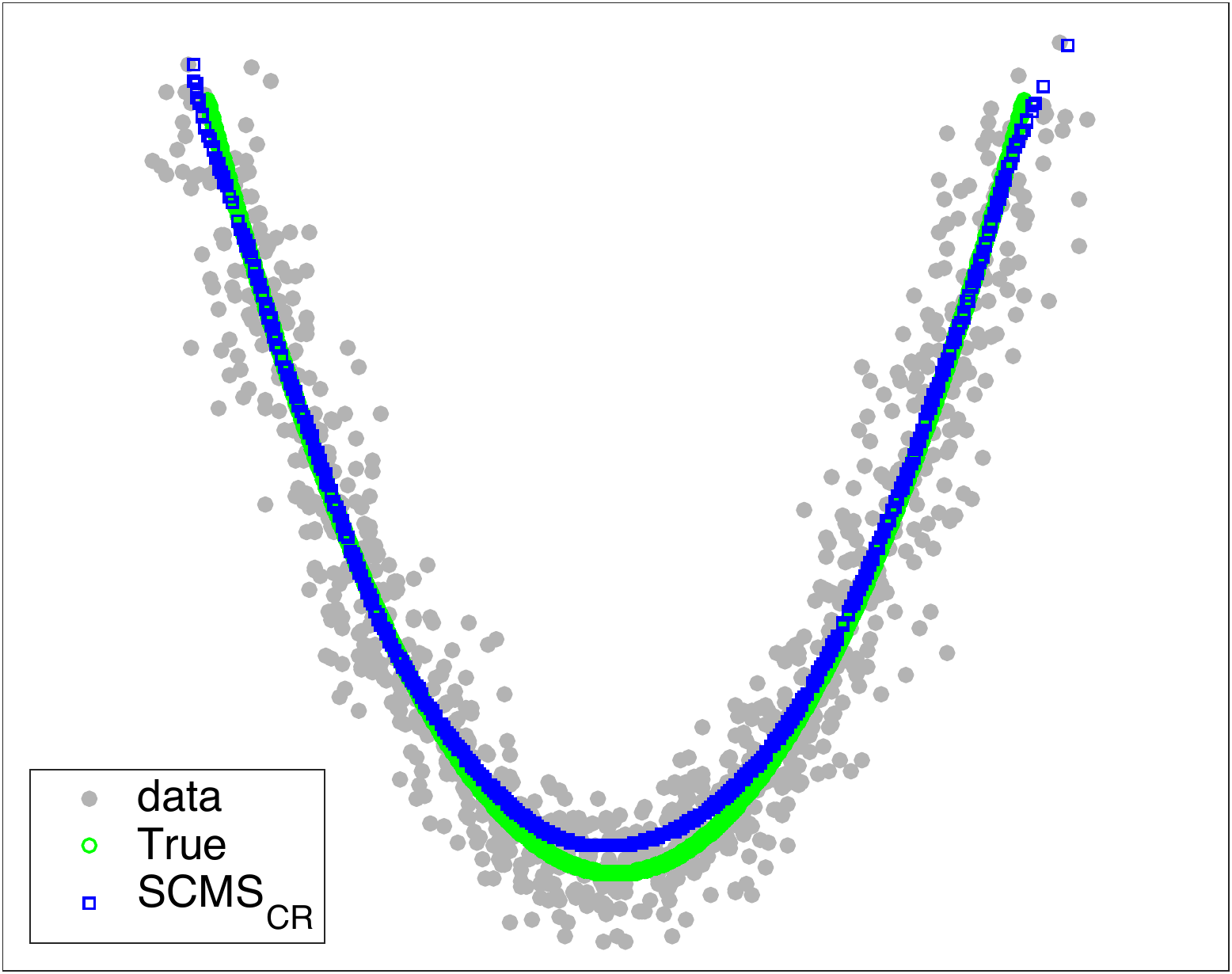}} 
     \caption{\label{fig:RidgeEach} Comparison of the two estimated
     ridges by LSDRF, SCMS$_{\text{LS}}$ and SCMS$_{\text{CR}}$.}
    \end{center}
   \end{figure}   
   \begin{figure}[p]
    \begin{center}
     \subfigure{\includegraphics[width=0.3\textwidth,clip]{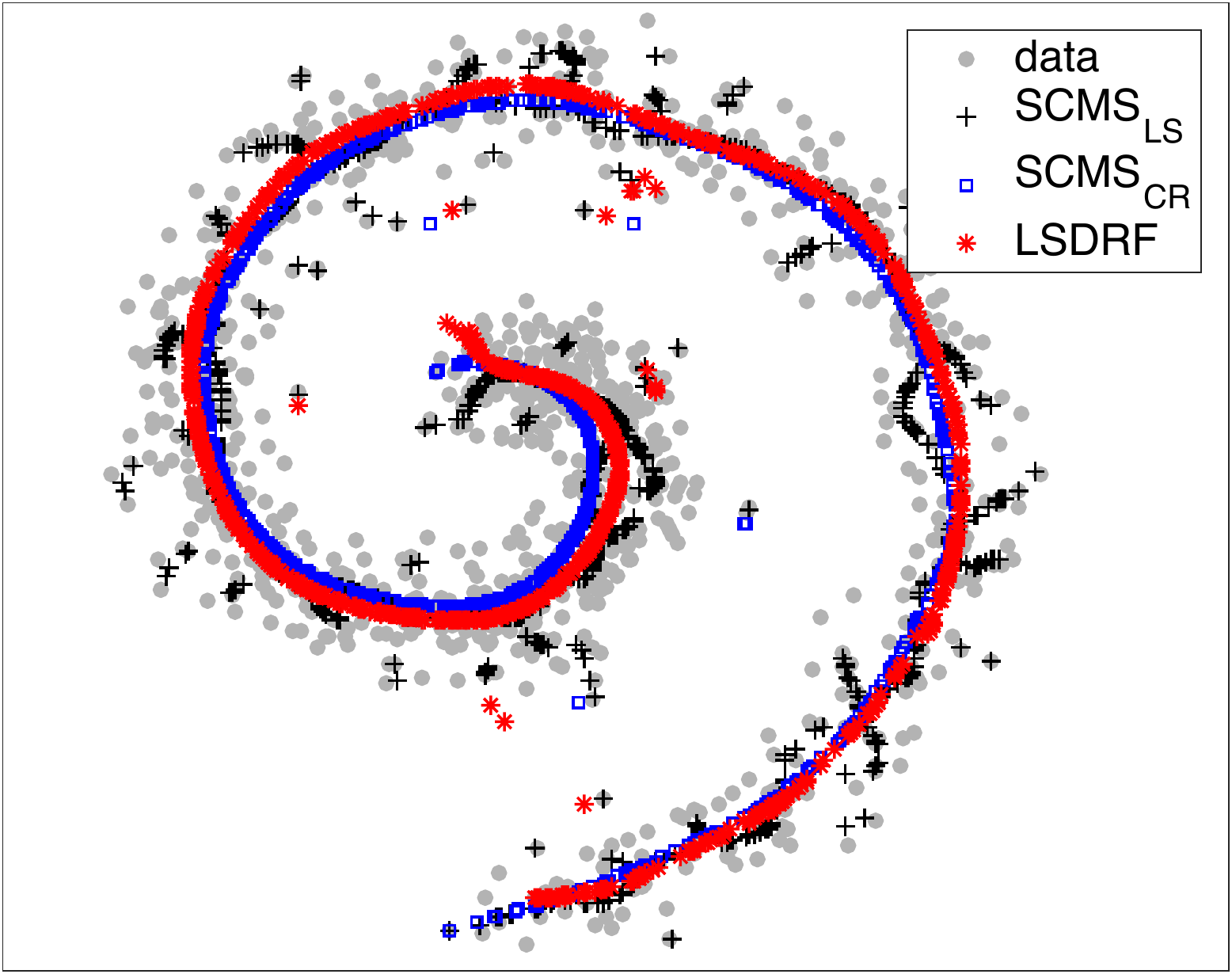}}
     \subfigure{\includegraphics[width=0.3\textwidth,clip]{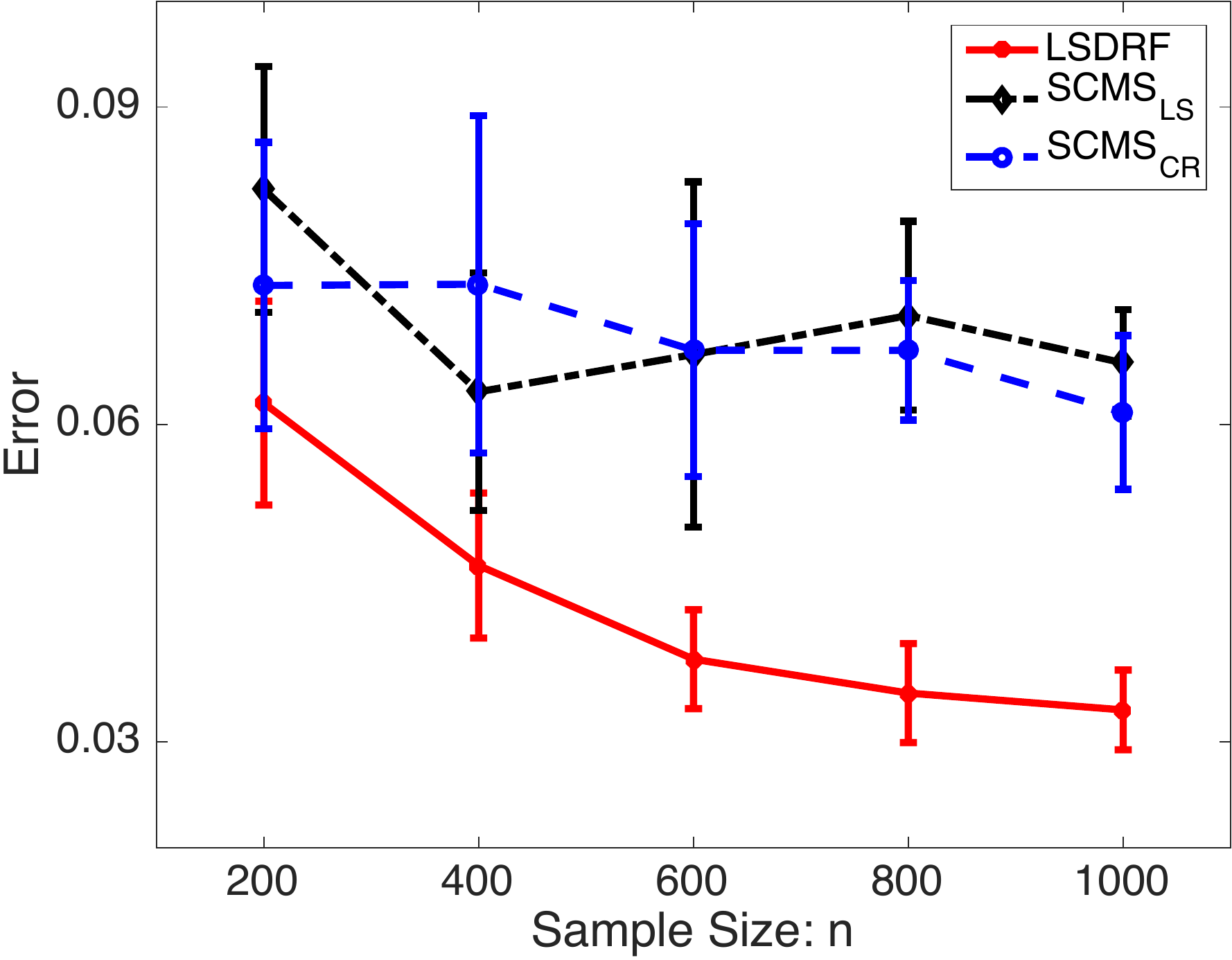}}
     \subfigure{\includegraphics[width=0.3\textwidth,clip]{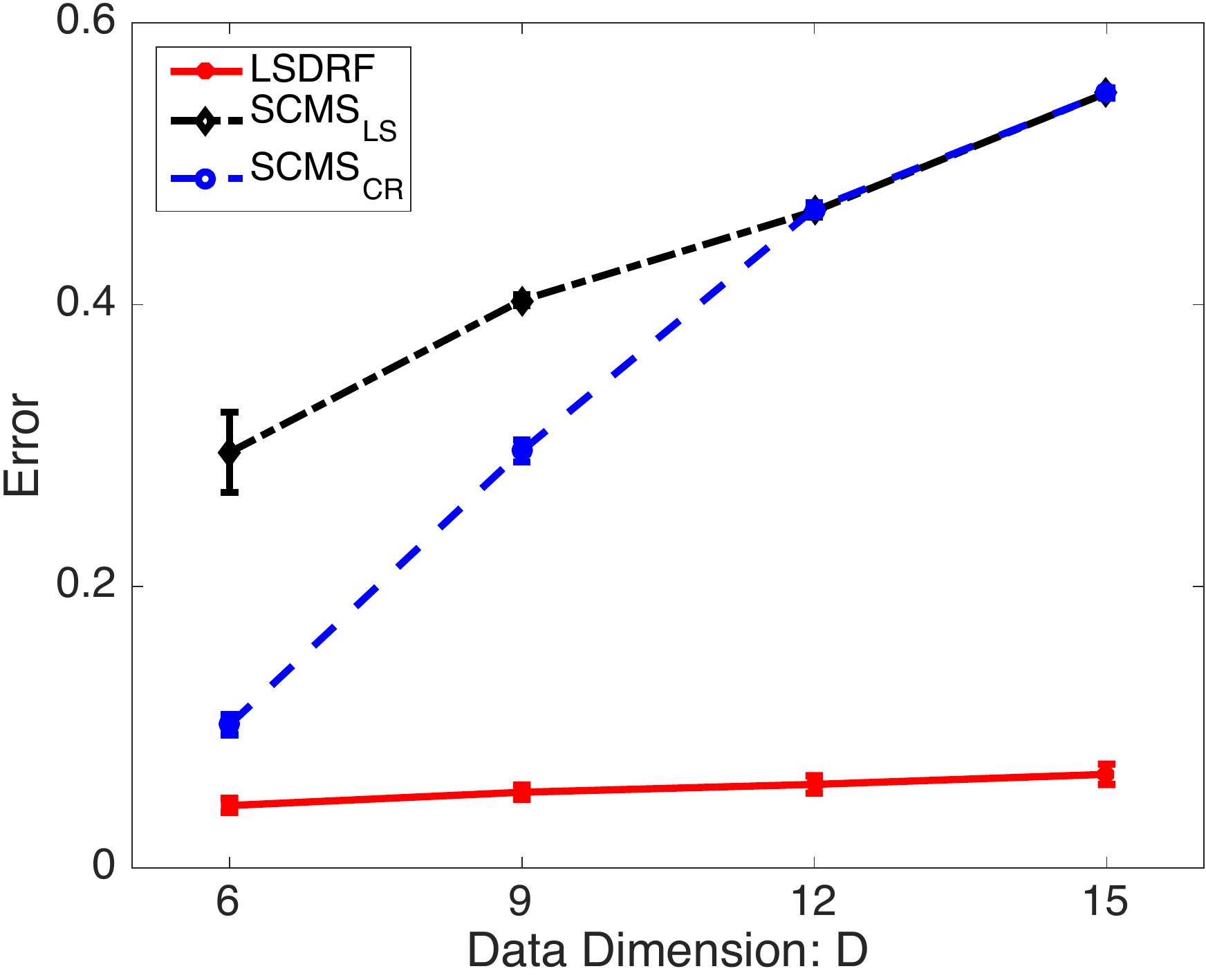}}\\
     \subfigure{\includegraphics[width=0.3\textwidth,clip]{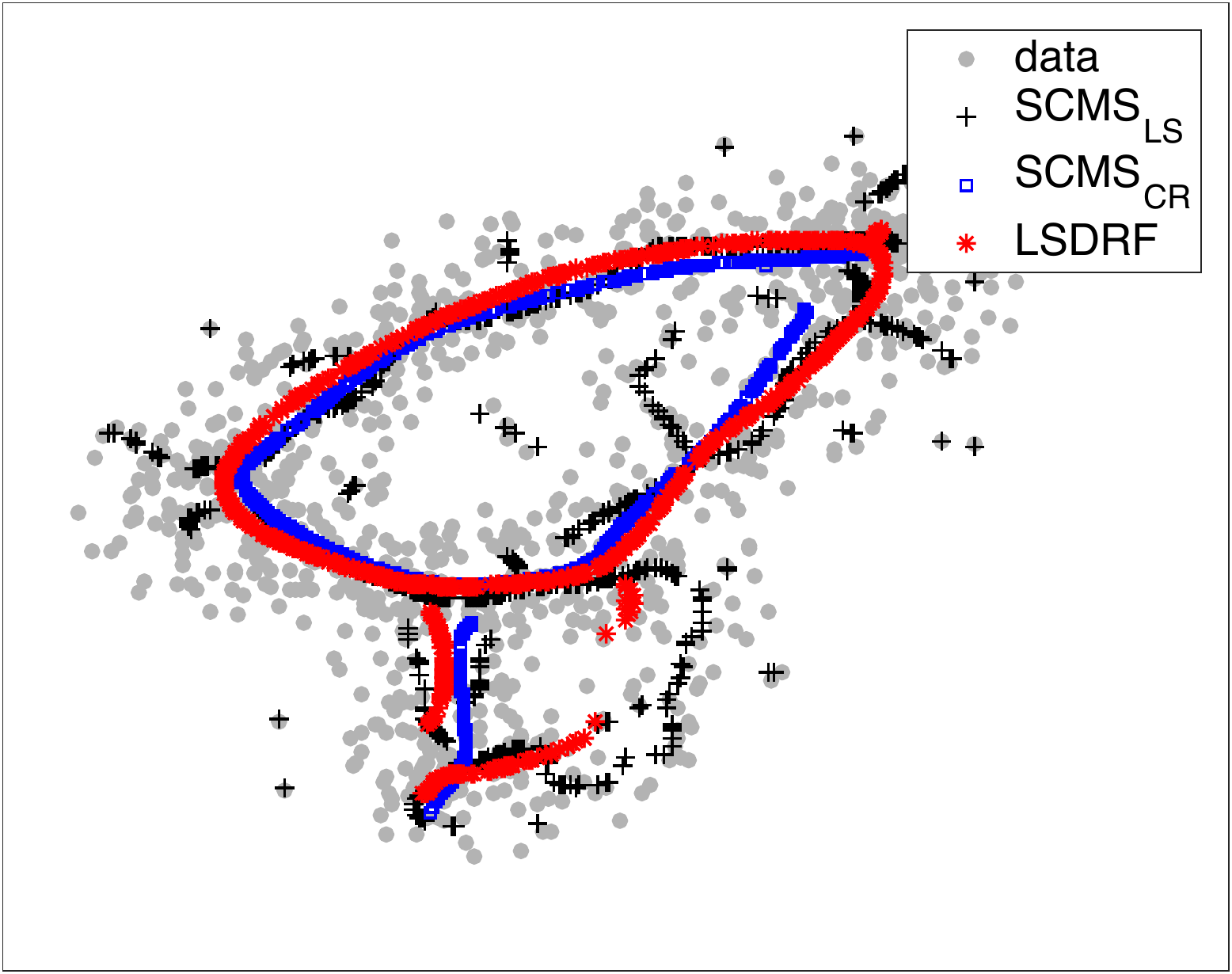}}
     \subfigure{\includegraphics[width=0.3\textwidth,clip]{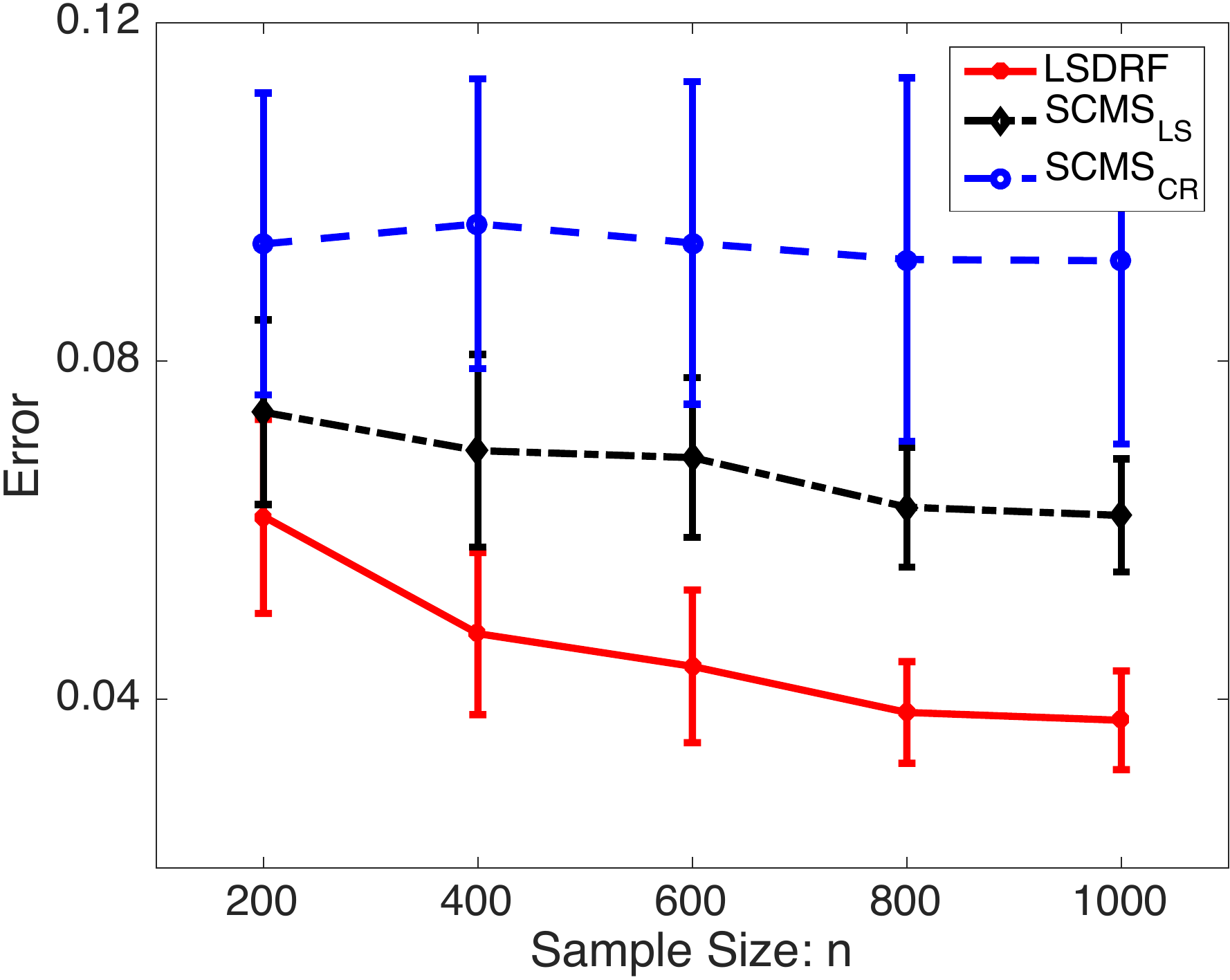}}
     \subfigure{\includegraphics[width=0.3\textwidth,clip]{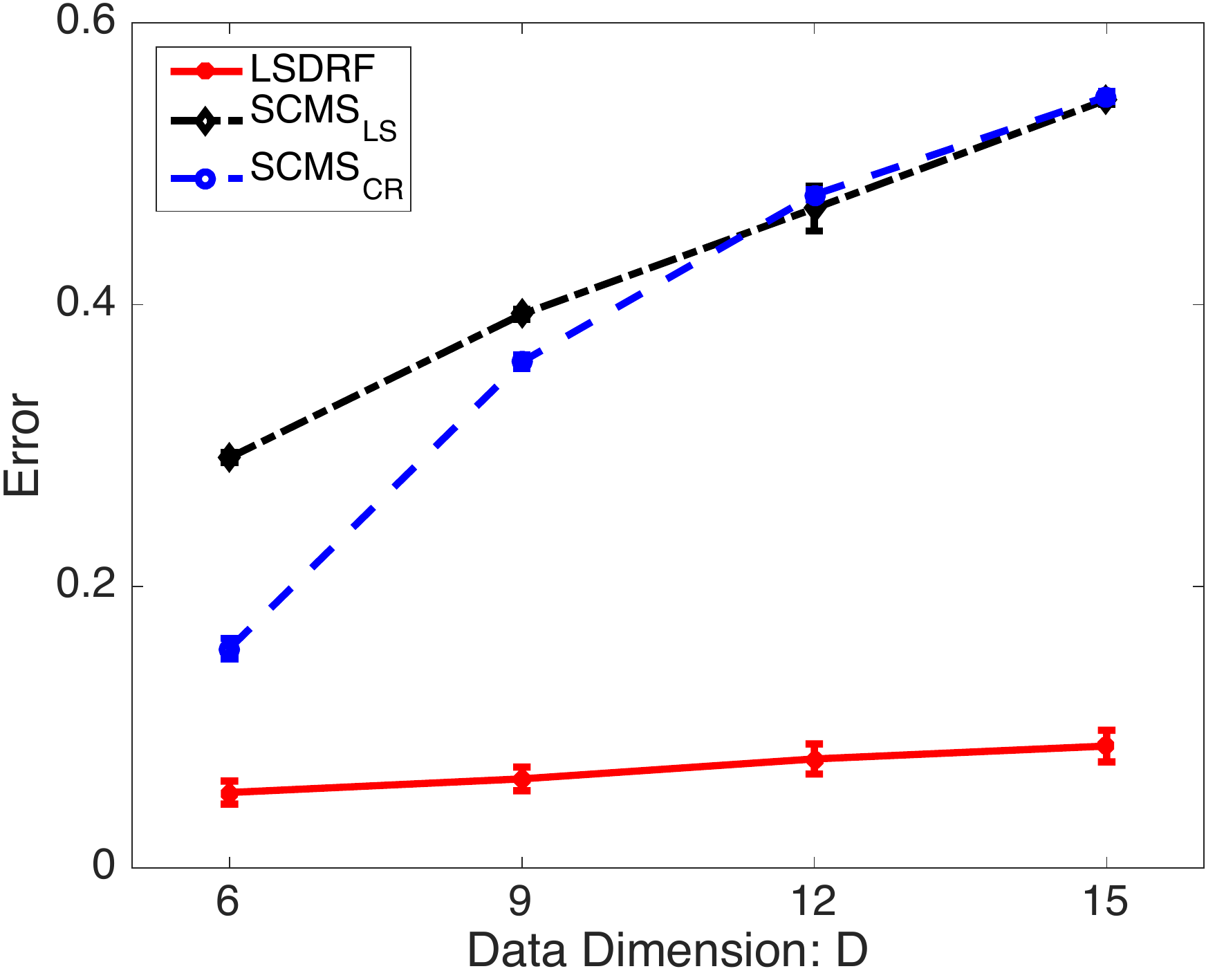}}\\
     \subfigure{\includegraphics[width=0.3\textwidth,clip]{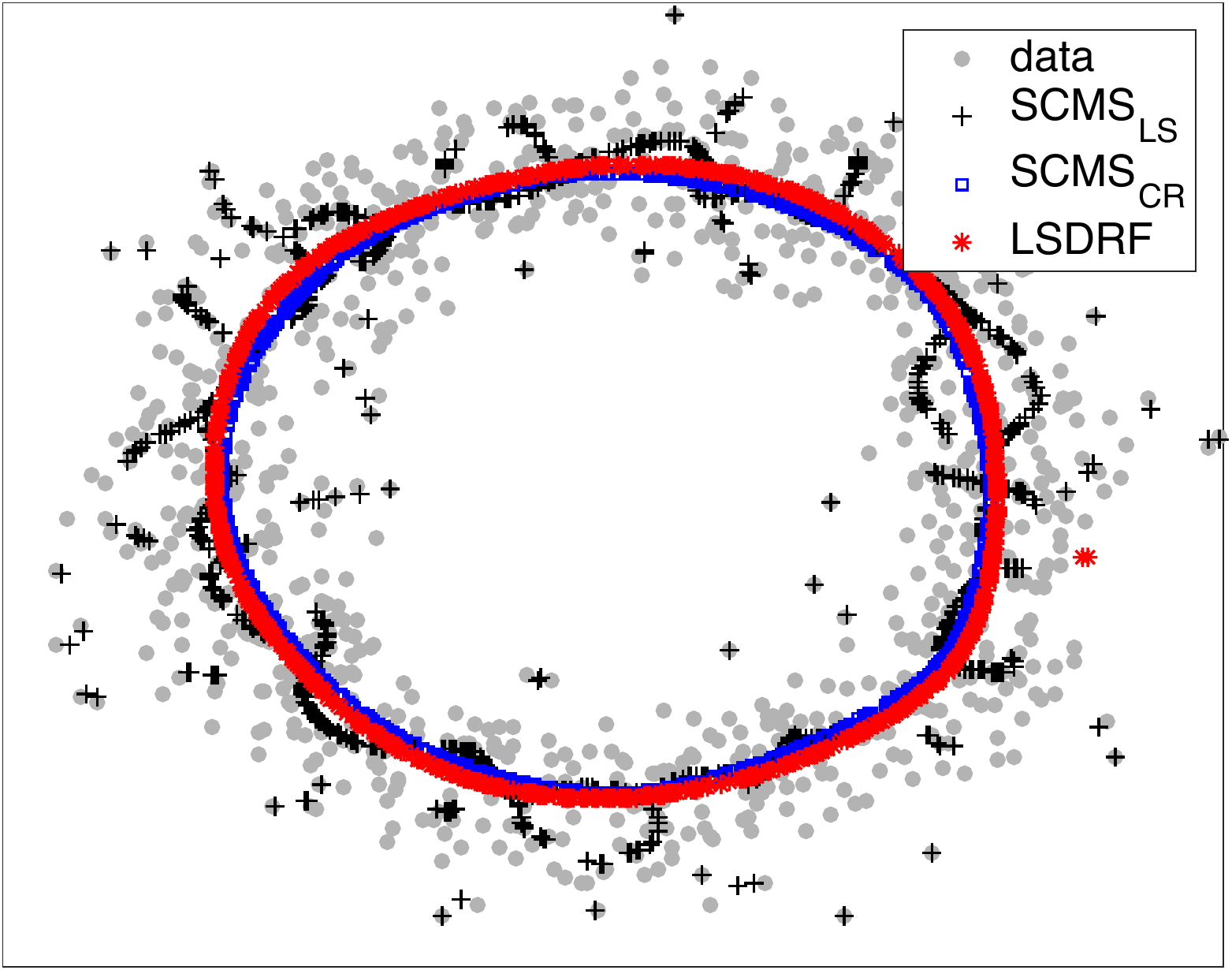}}
     \subfigure{\includegraphics[width=0.3\textwidth,clip]{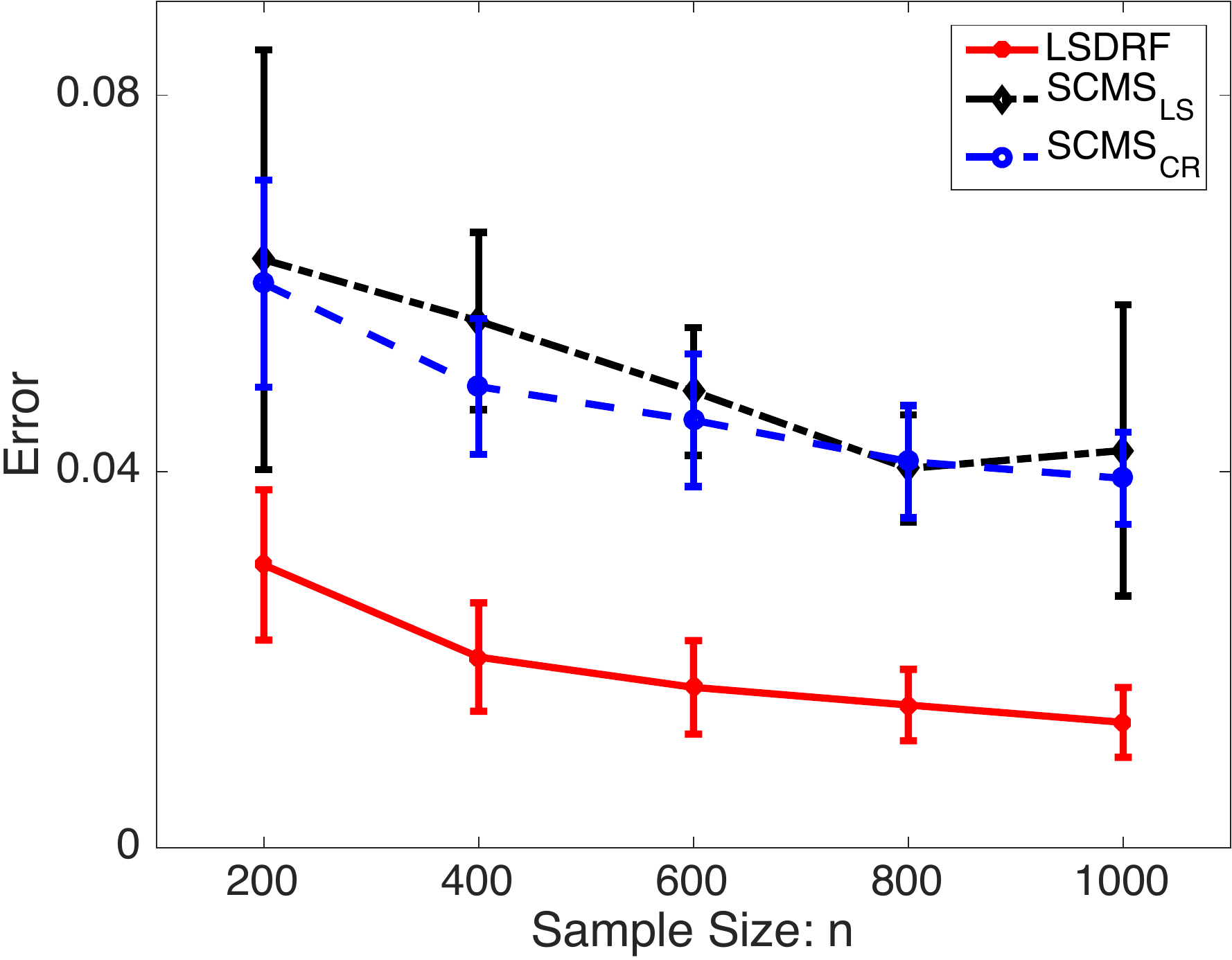}}
     \subfigure{\includegraphics[width=0.3\textwidth,clip]{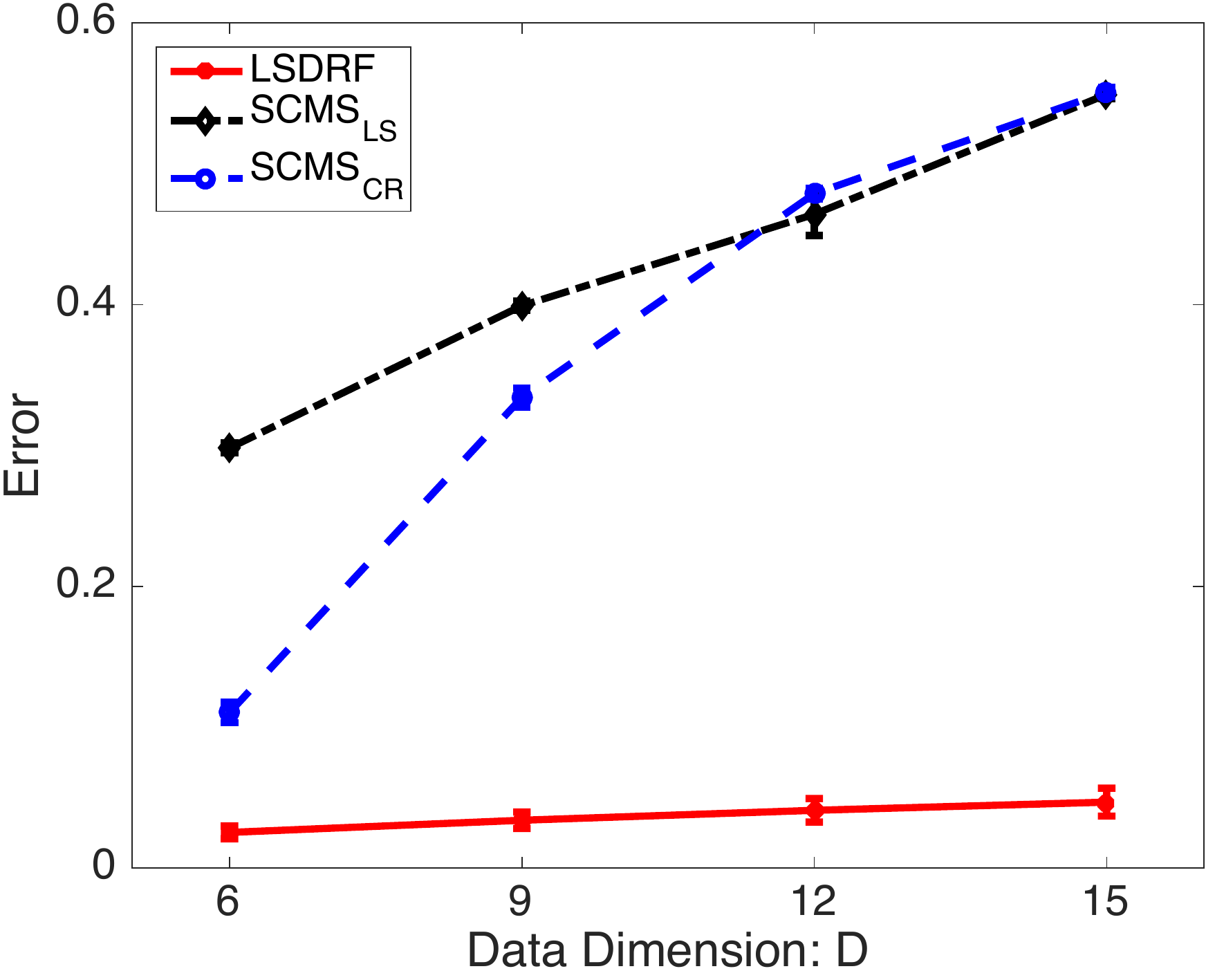}}\\
     \subfigure{\includegraphics[width=0.3\textwidth,clip]{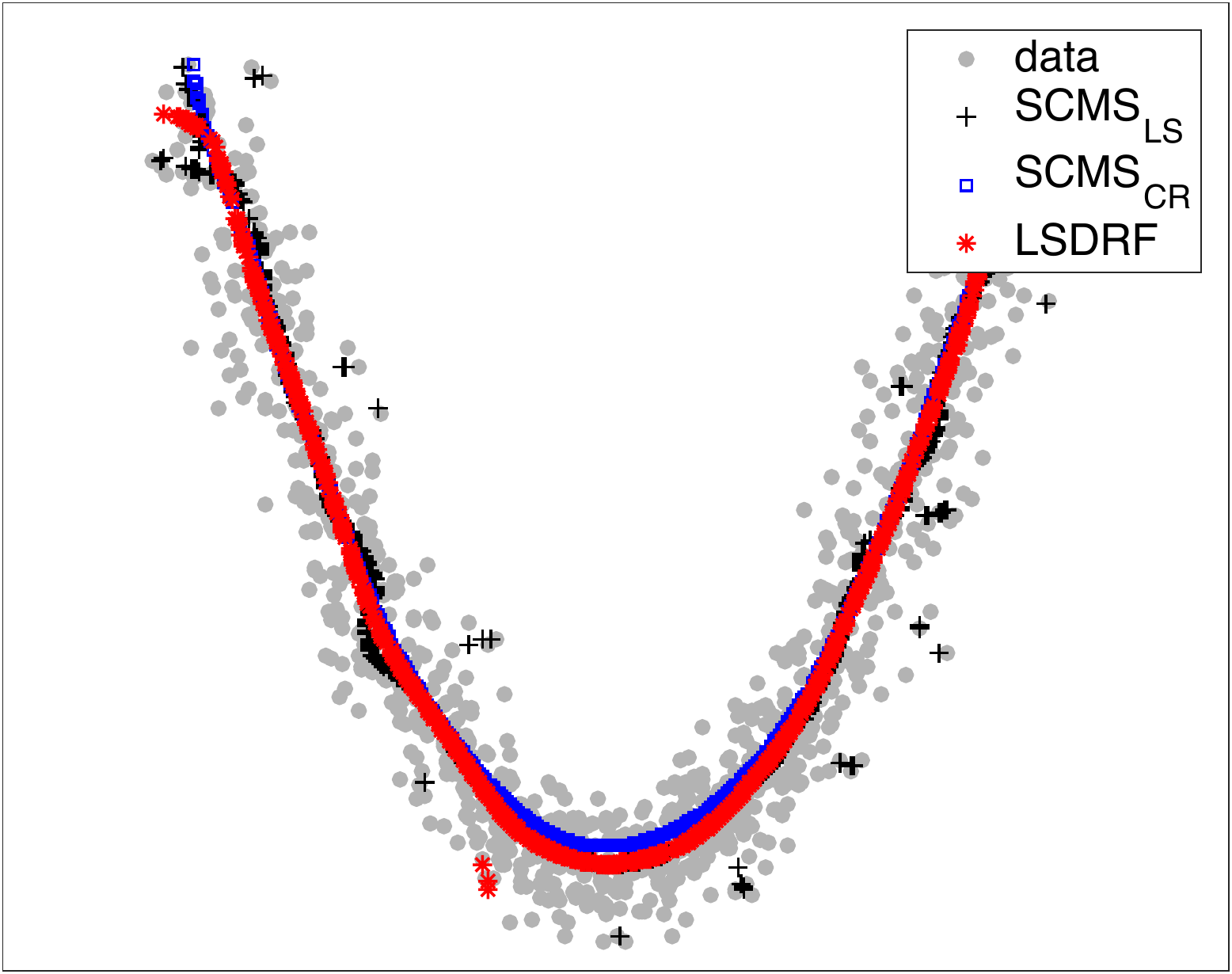}}
     \subfigure{\includegraphics[width=0.3\textwidth,clip]{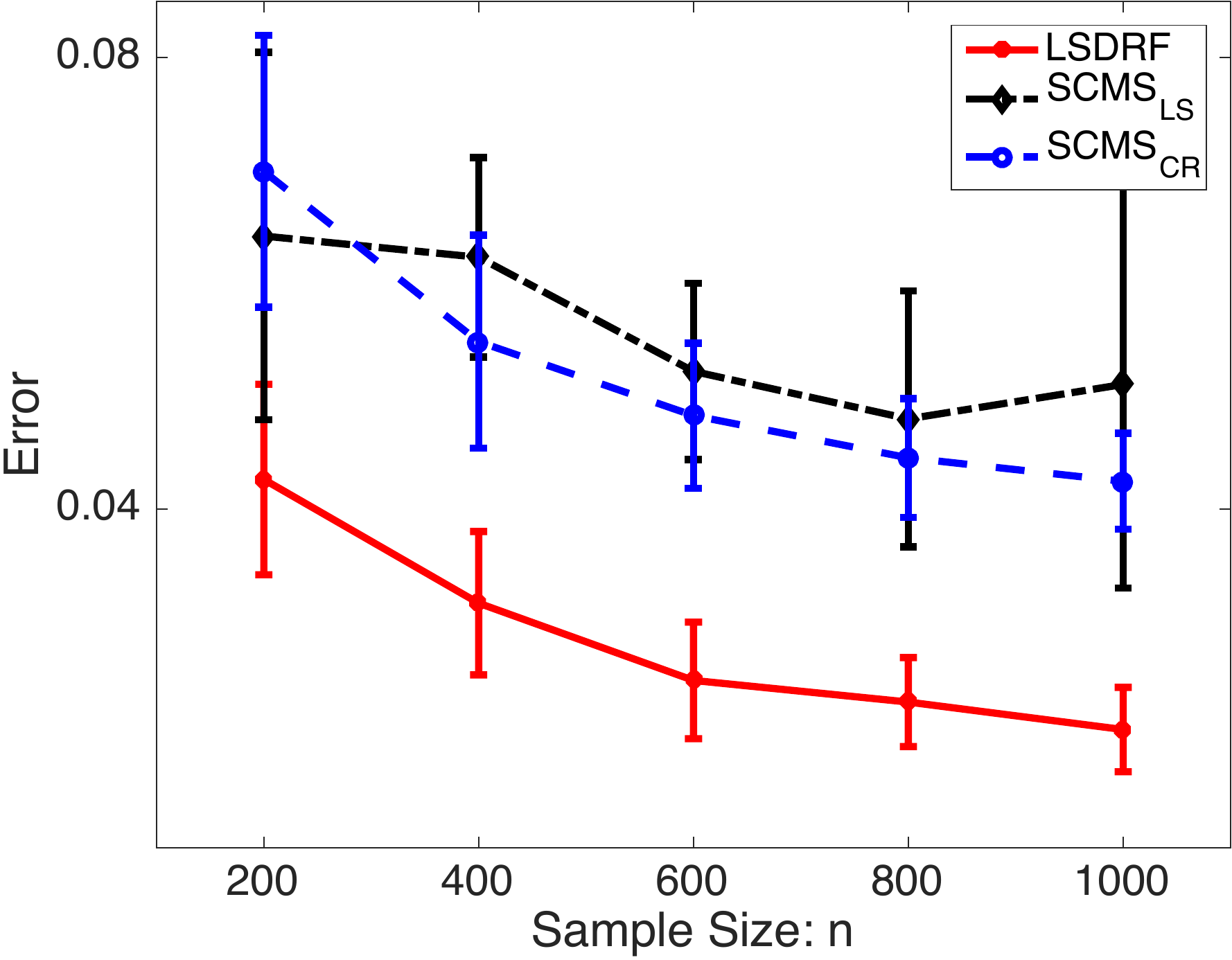}}
     \subfigure{\includegraphics[width=0.3\textwidth,clip]{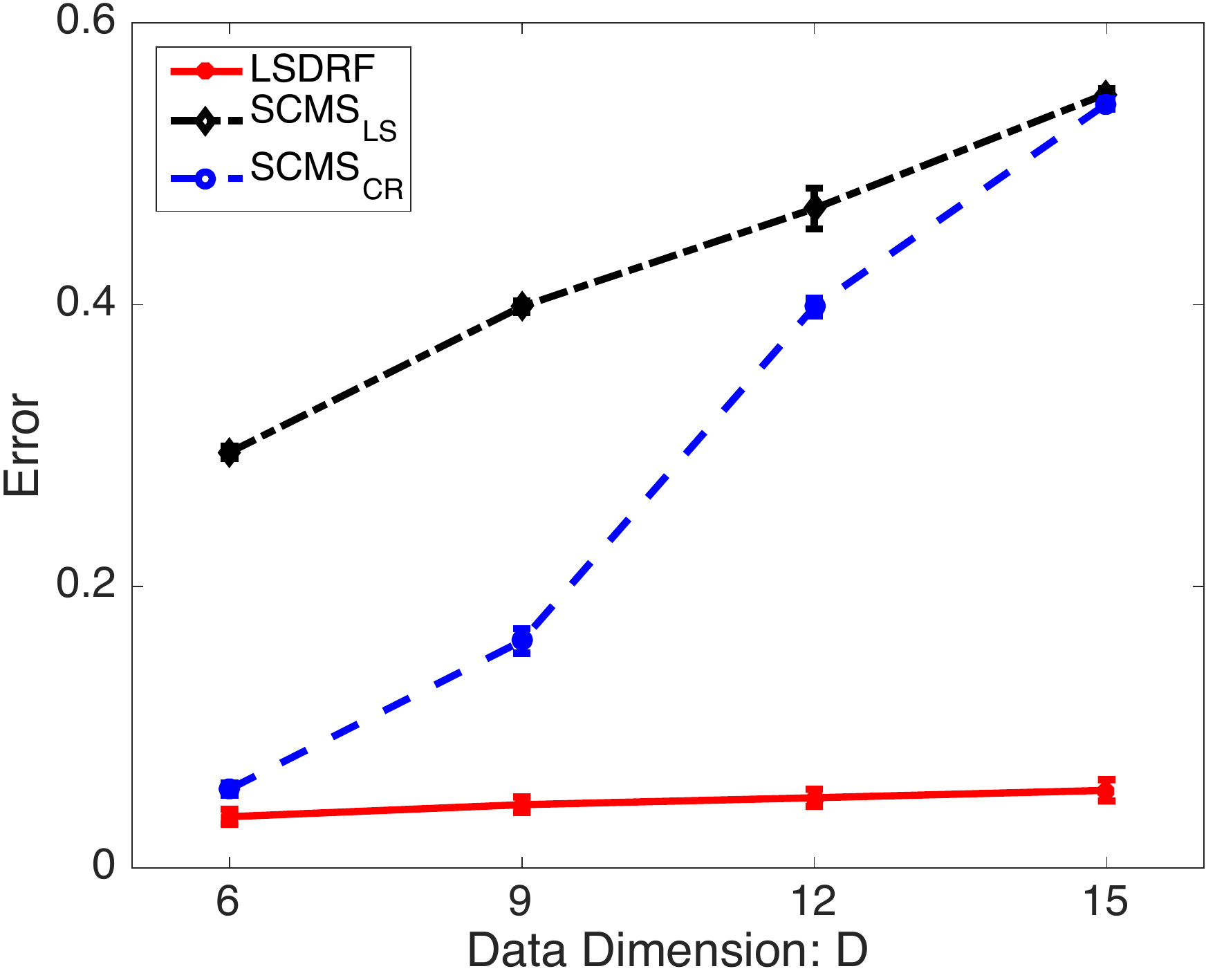}}\\
     \setcounter{subfigure}{0} \subfigure[Density ridge
     estimates]{\includegraphics[width=0.3\textwidth,clip]{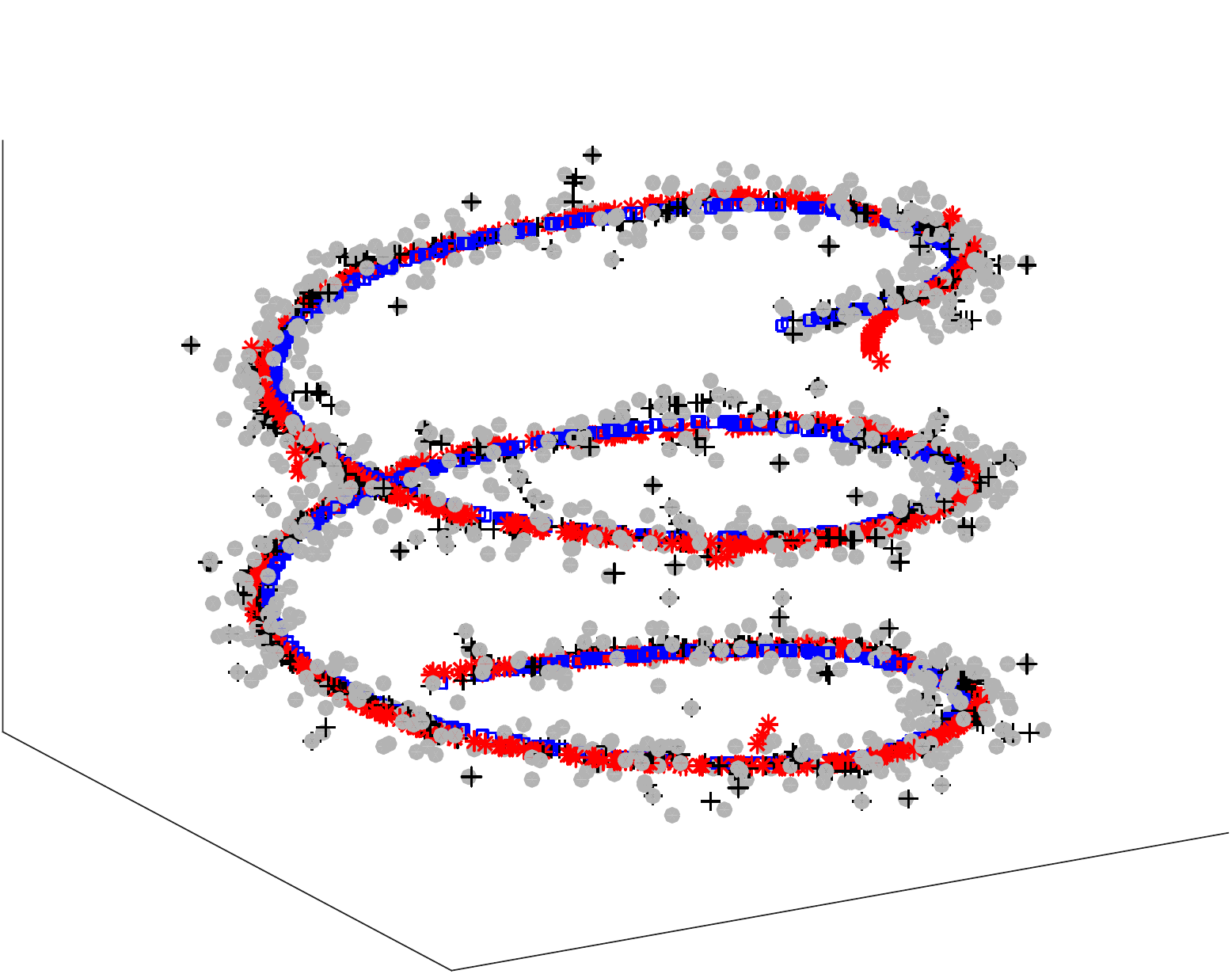}}
     \subfigure[Error against sample
     size]{\includegraphics[width=0.3\textwidth,clip]{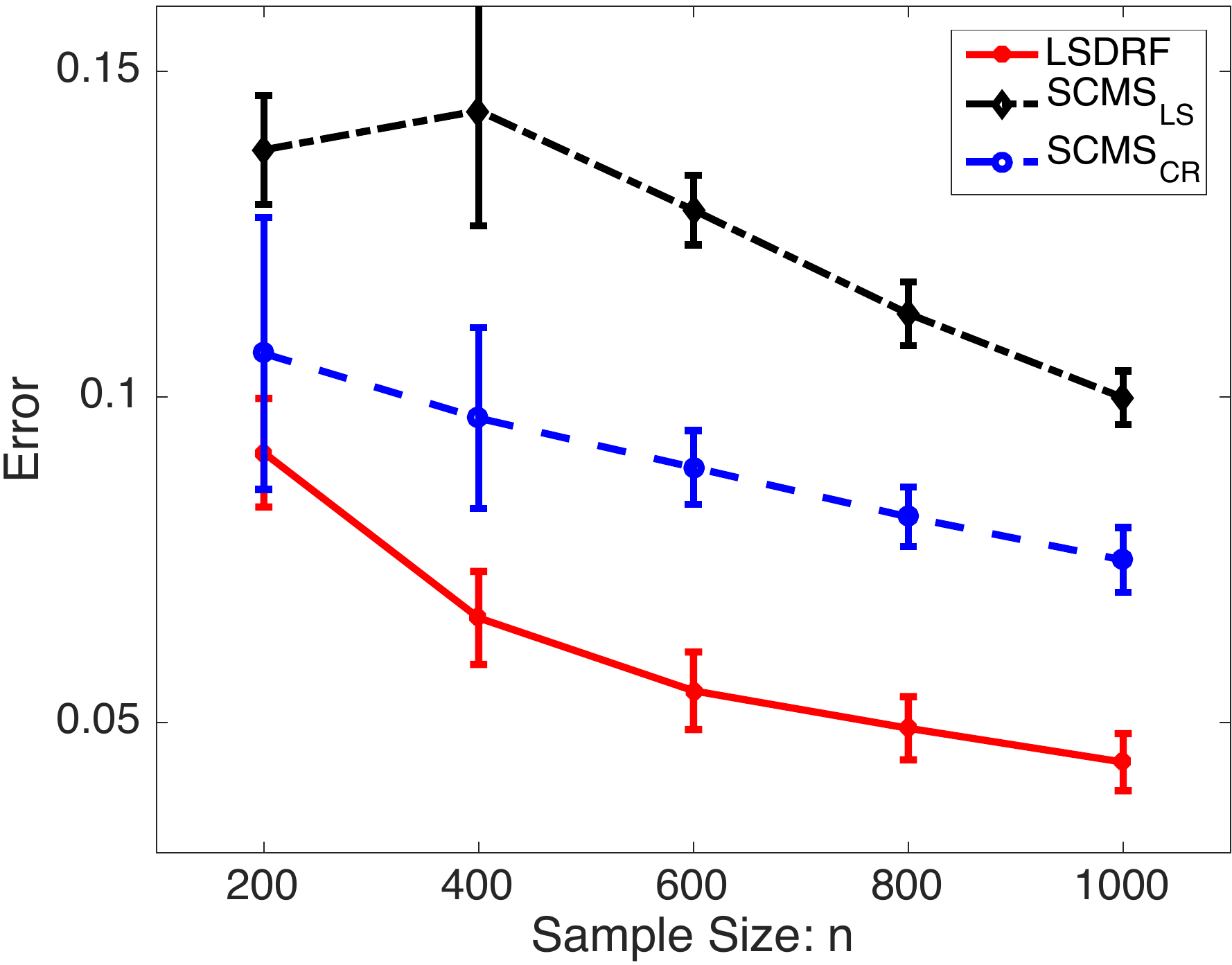}}
     \subfigure[Error against data
     dimension]{\includegraphics[width=0.3\textwidth,clip]{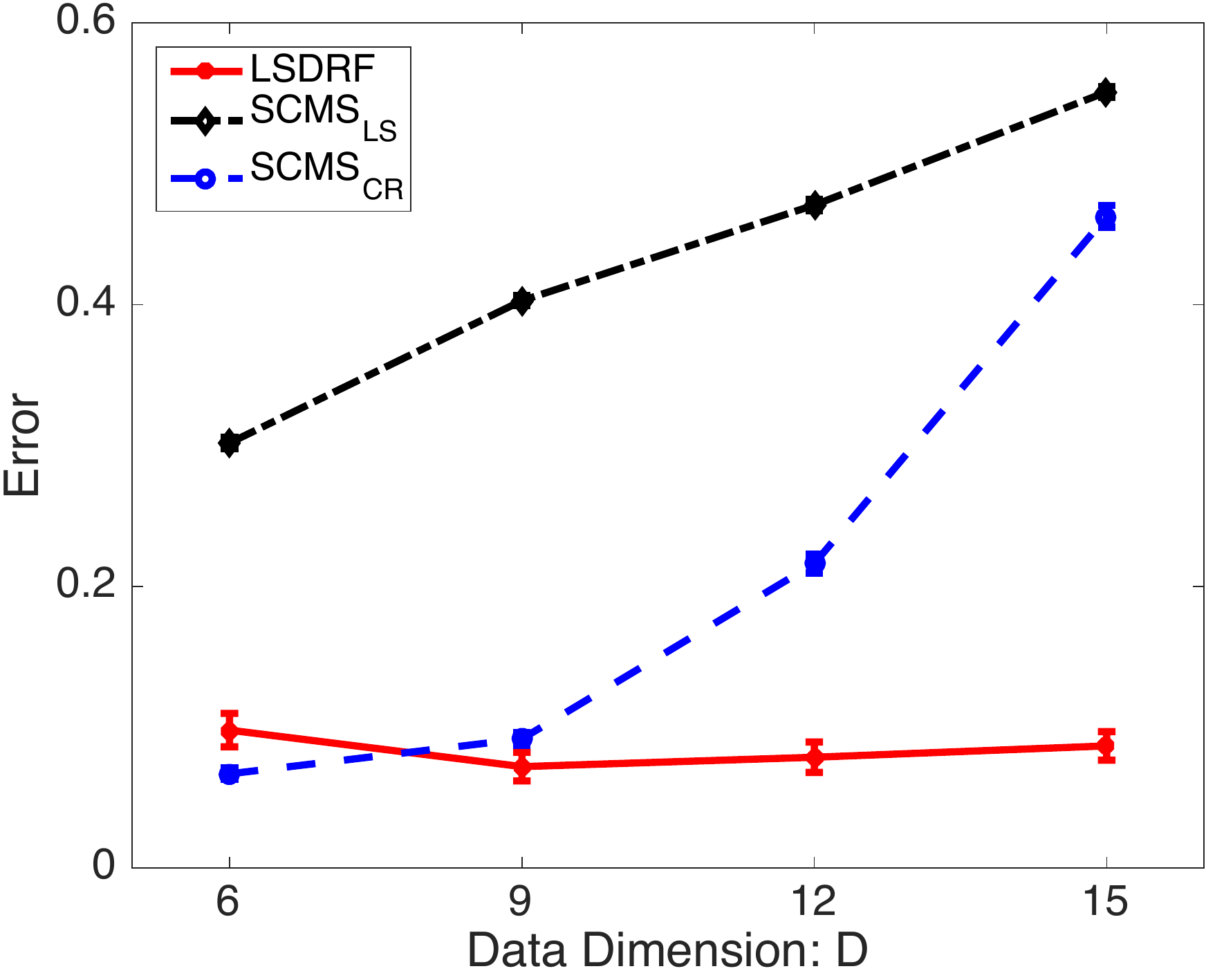}}
     \caption{\label{fig:RidgeArt} Performance of ridge estimation on
     artificial data. Each point and error bar denote the average and
     standard deviation of ARI over $50$ runs, respectively. For (c),
     $n=1000$. Errors for ridge estimation are computed according
     to~\eqref{error-ridge}.}
    \end{center}
   \end{figure}
   
   The performance of LSDRF is compared to SCMS with two different
   bandwidth selection methods:
   \begin{itemize}
    \item {\bf LSDRF}: When estimating $g_j(\bm{x})$, we selected ten
	  candidates of the width parameter in the Gaussian kernel and
	  the regularization parameter from $10^{l}\times
	  \sigma^{(j)}_{\text{med}}$ ($-0.3\leq l \leq 1$) and $10^{m}$
	  ($-4\leq m \leq 0$), respectively. When estimating
	  $[\bm{H}(\bm{x})]_{ij}$, ten candidates of the width parameter
	  in the Gaussian kernel were selected from $10^{l}\times
	  \sqrt{\sigma^{(i)}_{\text{med}}\sigma^{(j)}_{\text{med}}}$
	  ($-0.3\leq l \leq 1$). For the regularization parameter, we
	  used the same candidates as in $g_j(\bm{x})$.
	  
    \item {\bf SCMS$_{\text{LS}}$}: The bandwidth parameter $h$ was
	  cross-validated based on the standard \emph{integrated squared
	  error}. We selected ten candidates of $h$ from $10^{l}\times
	  h_{\text{med}}$ ($-1.5\leq l \leq 0$) where $h_{\text{med}}$
	  is the median value of $|x_i^{(j)}-x_k^{(j)}|$ with respect to
	  $i$, $j$ and $k$.
	 
   \item {\bf SCMS$_{\text{CR}}$}: The bandwidth parameter $h$ was
	 cross-validated based on the coverage risk proposed
	 in~\citet{chen2015optimal}. As suggested
	 in~\citet{chen2015optimal}, we selected ten candidates of $h$
	 from $10^{l}\times h_{\text{NR}}$ ($-1\leq l \leq 0$) where
	 $h_{\text{NR}}$ is the bandwidth based on the normal reference
	 rule~\citep{silverman1986density}.
   \end{itemize}
   We investigate the performance of these methods on a variety of
   simulated datasets.\footnote{Most of the datasets are generated using
   a MATLAB package made by Jakob Verbeek, which is available at
   \url{http://lear.inrialpes.fr/people/verbeek/code/kseg_soft.tar.gz}.}
   The $i$-th observation of data was generated according to
   $x_{i}^{(j)}=f^{(j)}(t_i)+n^{(j)}_{i}$, where $t_i$ was taken from
   some range at regular intervals, $f^{(j)}(\cdot)$ denotes some fixed
   function, and $n^{(j)}_i$ was the Gaussian noise with mean $0$ and
   standard deviation $0.15$. Higher-dimensional data were created by
   appending the Gaussian variables with mean $0$ and standard deviation
   $0.15$.  The estimation error was measured by
   \begin{align}
    \text{Error}=\frac{1}{n}\sum_{i=1}^{n} \min_{l}\|
    \widehat{\bm{y}}_i-\bm{f}(t_l)\|,
    \label{error-ridge}
   \end{align}
   where
   $\bm{f}(\cdot)=(f^{(1)}(\cdot),f^{(2)}(\cdot),\dots,f^{(D)}(\cdot))^{\top}$
   and $\widehat{\bm{y}}_i$ denotes an estimate of the density ridge
   point from $\bm{x}_i$.

   The estimated ridges are visualized in
   Fig.\ref{fig:RidgeEach}. SCMS$_{\text{LS}}$ provides a broken and
   non-smooth ridge estimate because the selected bandwidth by the
   least-squares cross-validation is small for density ridge estimation
   as in mode-seeking clustering. In contrast, the ridges estimated by
   LSDRF and SCMS$_{\text{CR}}$ are smooth. However, SCMS$_{\text{CR}}$
   gives a biased estimate around highly curved region in the true ridge
   (e.g., the centers of the spiral and quadratic curve in
   Fig.\ref{fig:RidgeEach}), while the bias in LSDDR seems smaller. This
   implies that LSDRF more accurately estimates density ridges. The
   accuracy of LSDRF is quantified on a variety of artificial datasets
   in Fig~\ref{fig:RidgeArt}. LSDRF produces smaller errors particularly
   when the sample size is large (Fig~\ref{fig:RidgeArt}(b)). In
   addition, as in mode-seeking clustering, the performance of LSDRF is
   even better when the dimensionality of data is higher
   (Fig.~\ref{fig:RidgeArt}(c)). This implies that our direct approach
   is useful for high(er)-dimensional data.
   \subsubsection{Density Ridge Estimation on Real-World Datasets}
   \label{ssec:application}
   \begin{figure}[!t]
   \begin{center}
    \subfigure{\includegraphics[width=0.32\textwidth,clip]{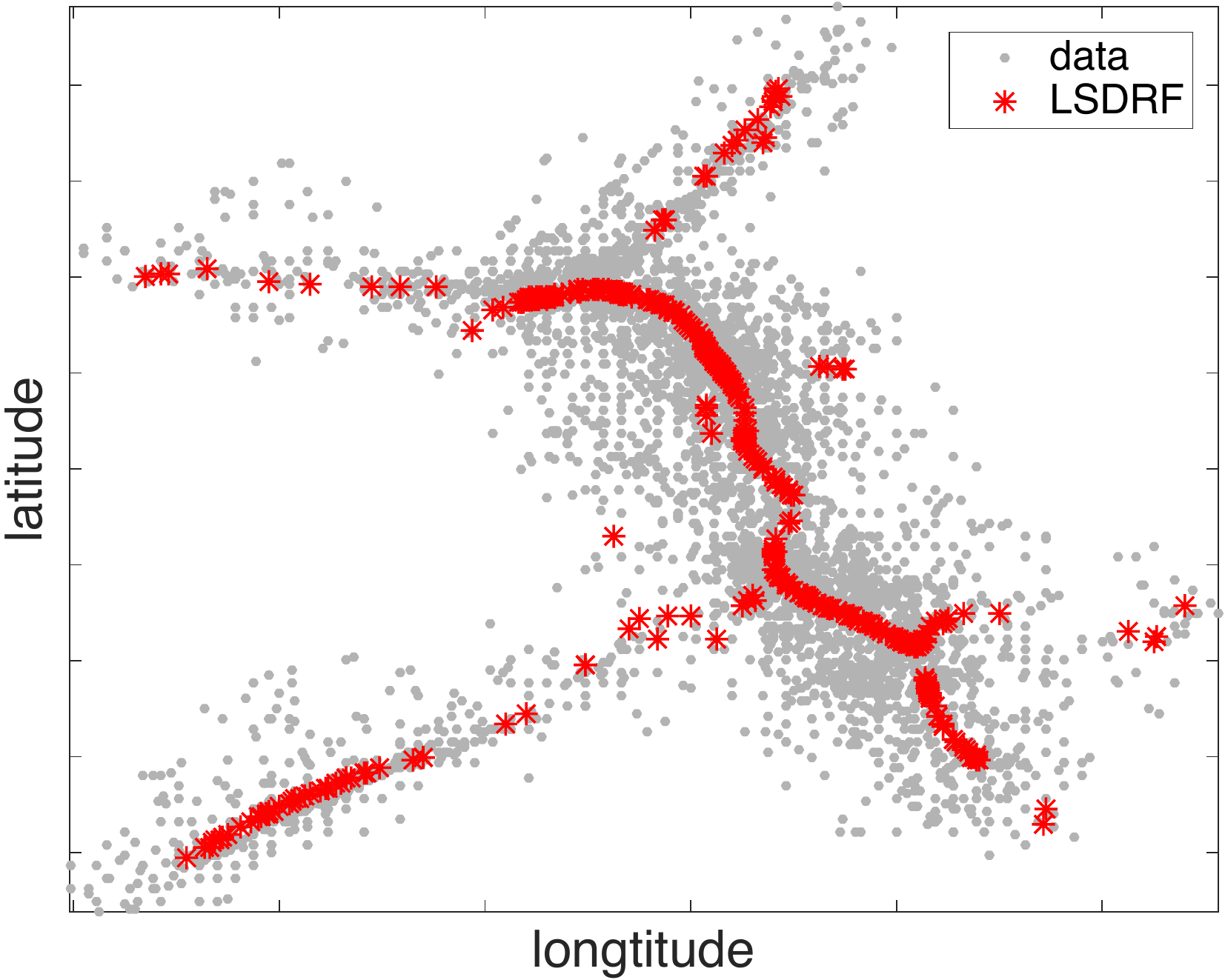}}
    \subfigure{\includegraphics[width=0.32\textwidth,clip]{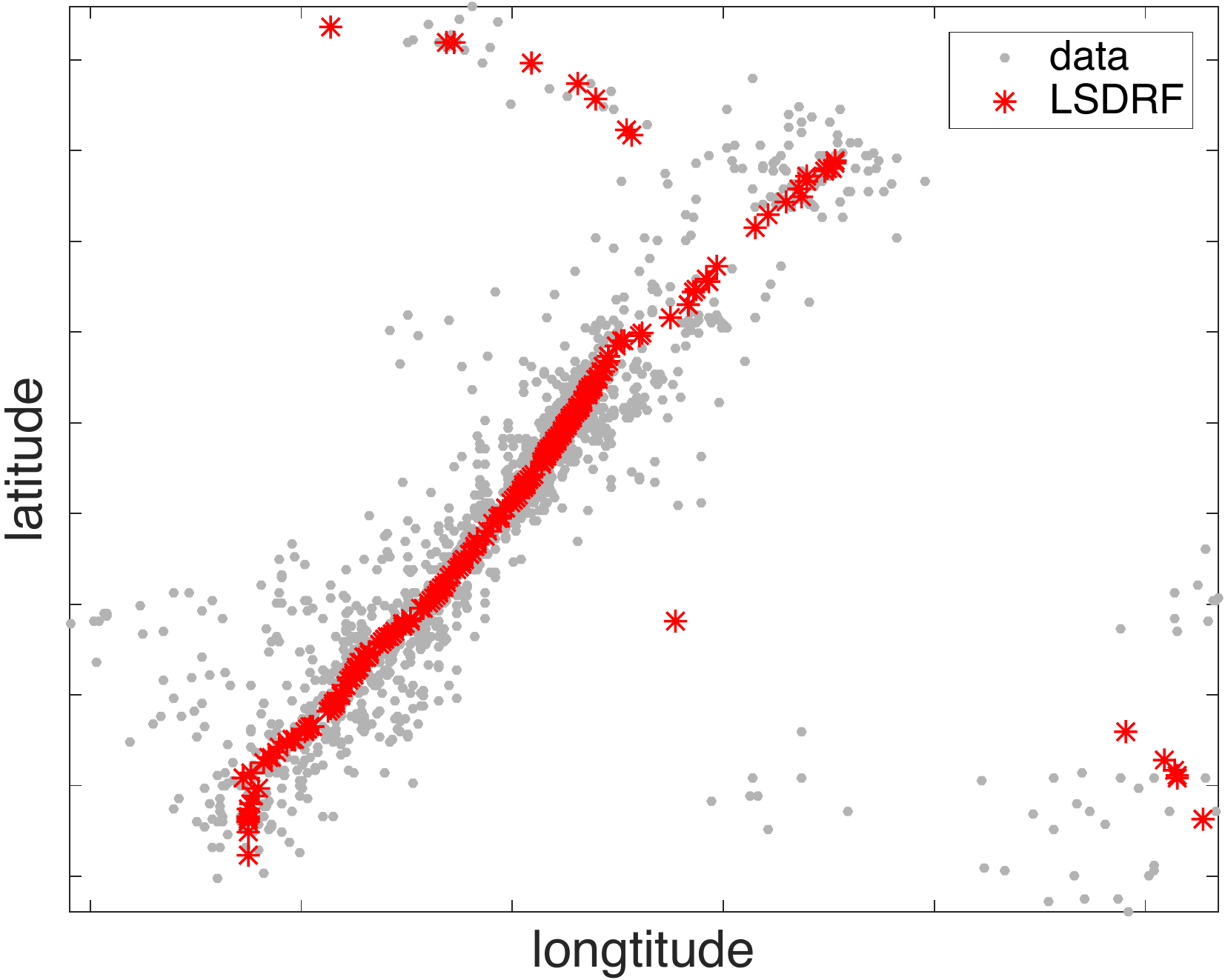}}
    \subfigure{\includegraphics[width=0.32\textwidth,clip]{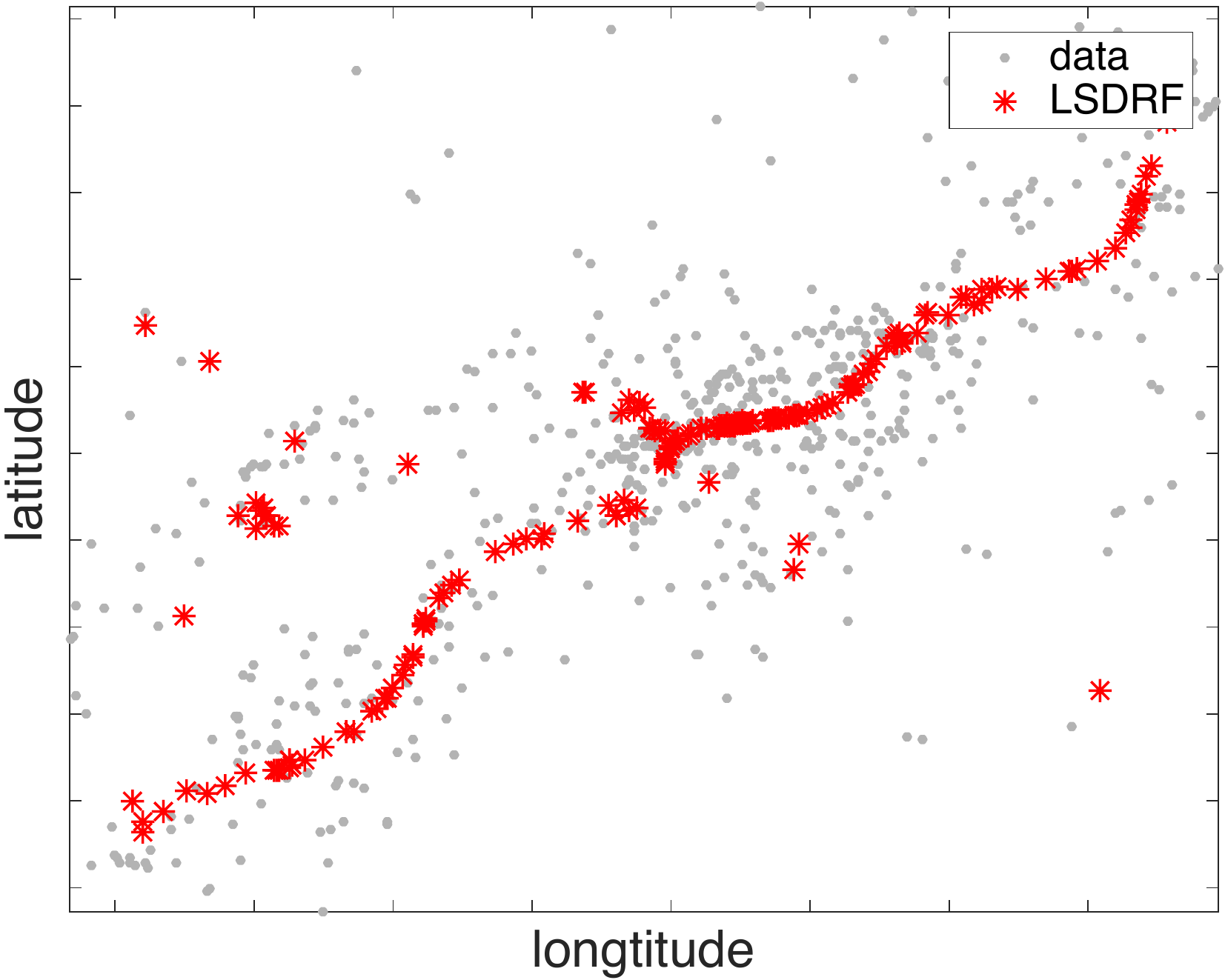}}
    \subfigure{\includegraphics[width=0.32\textwidth,clip]{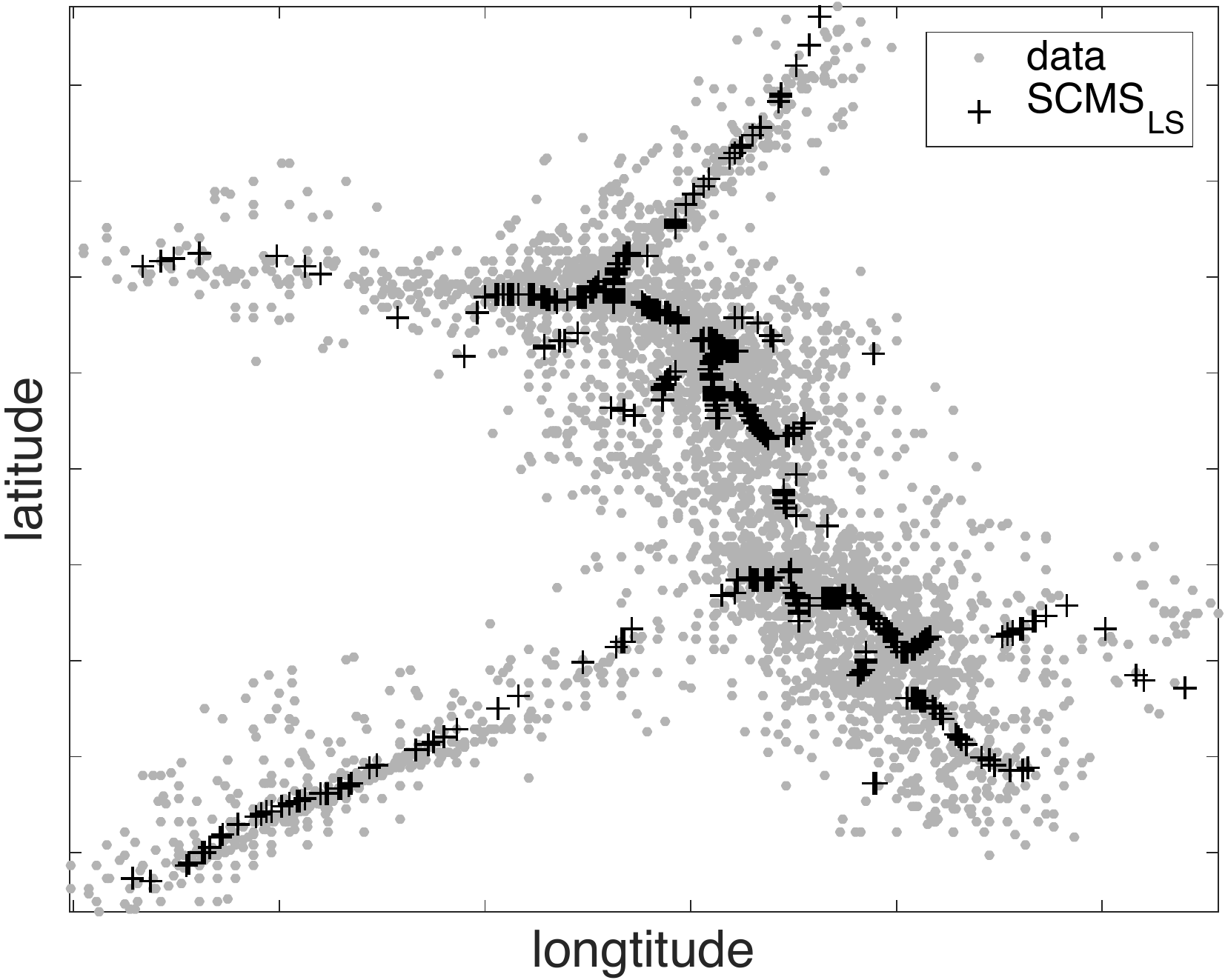}}
    \subfigure{\includegraphics[width=0.32\textwidth,clip]{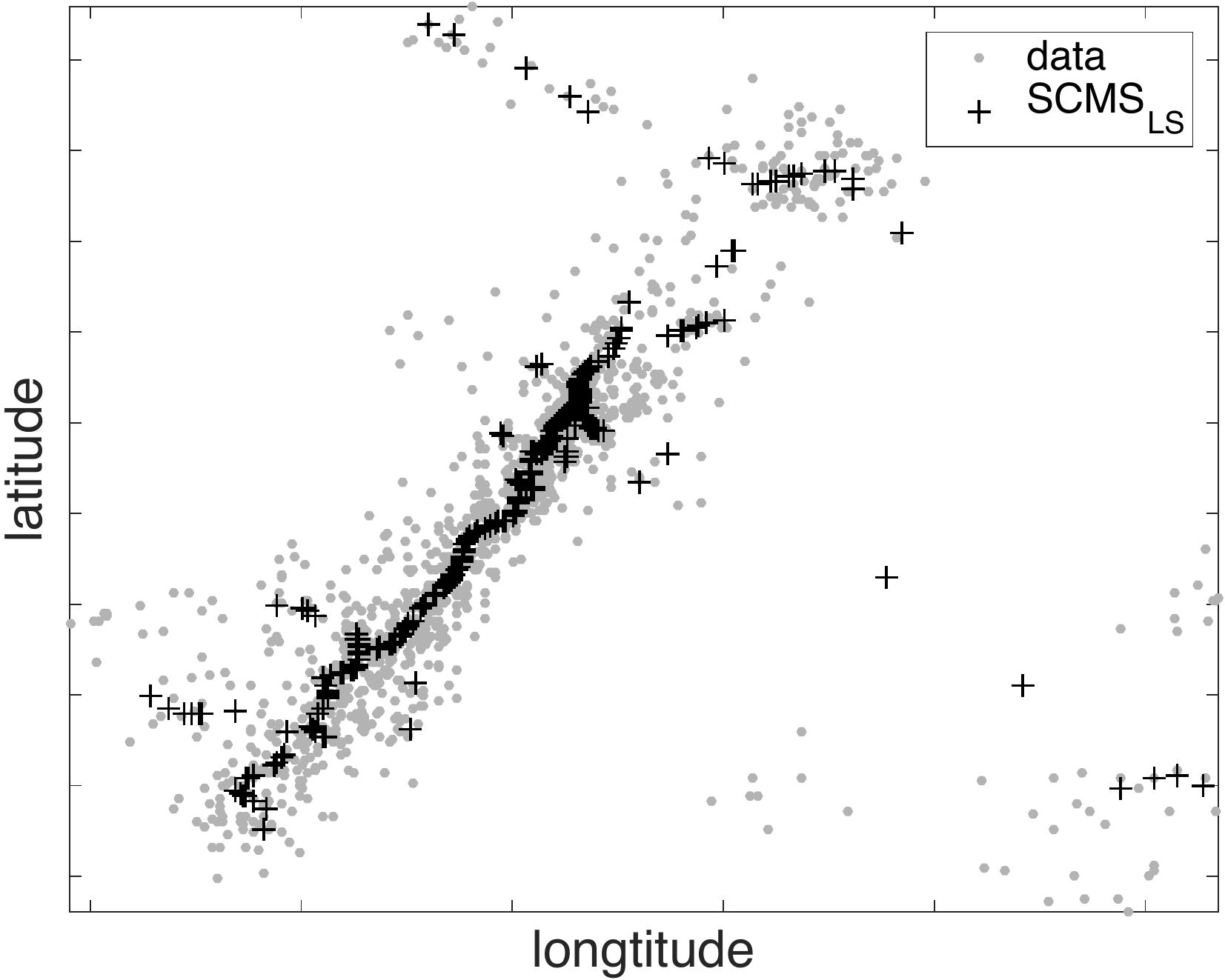}}
    \subfigure{\includegraphics[width=0.32\textwidth,clip]{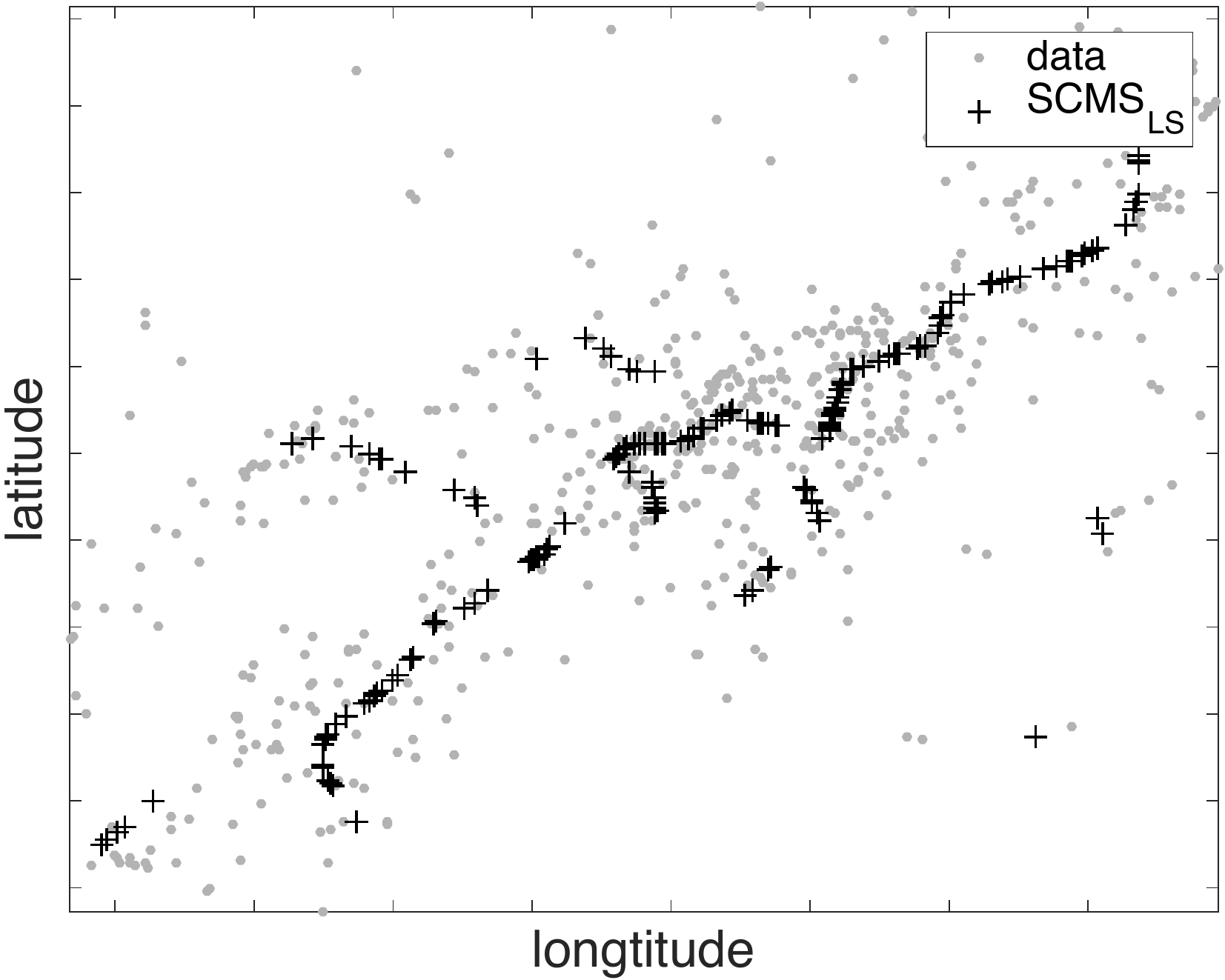}}
    \setcounter{subfigure}{0}
    \subfigure[Madrid~1]{\includegraphics[width=0.32\textwidth,clip]{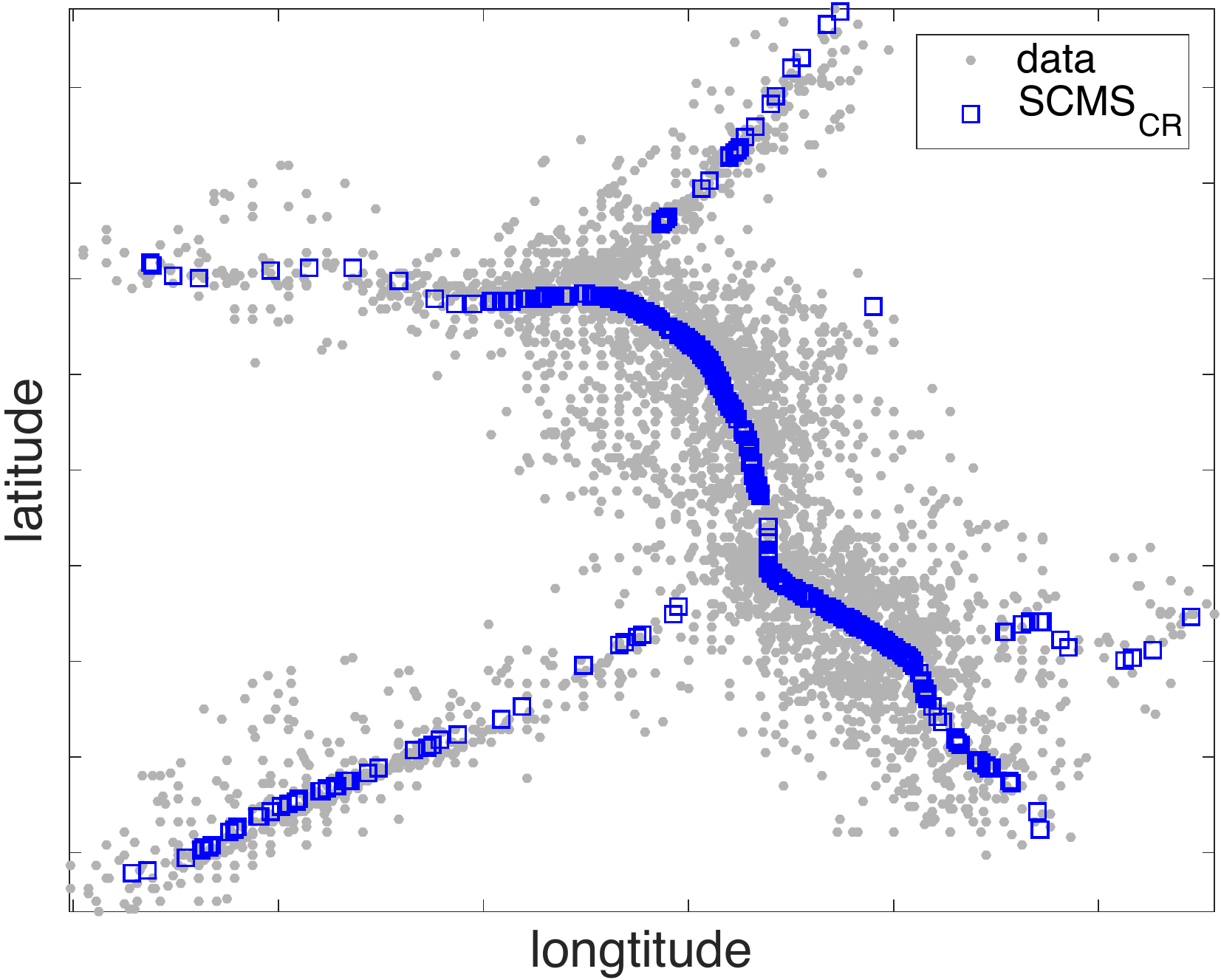}}
    \subfigure[Madrid~2]{\includegraphics[width=0.32\textwidth,clip]{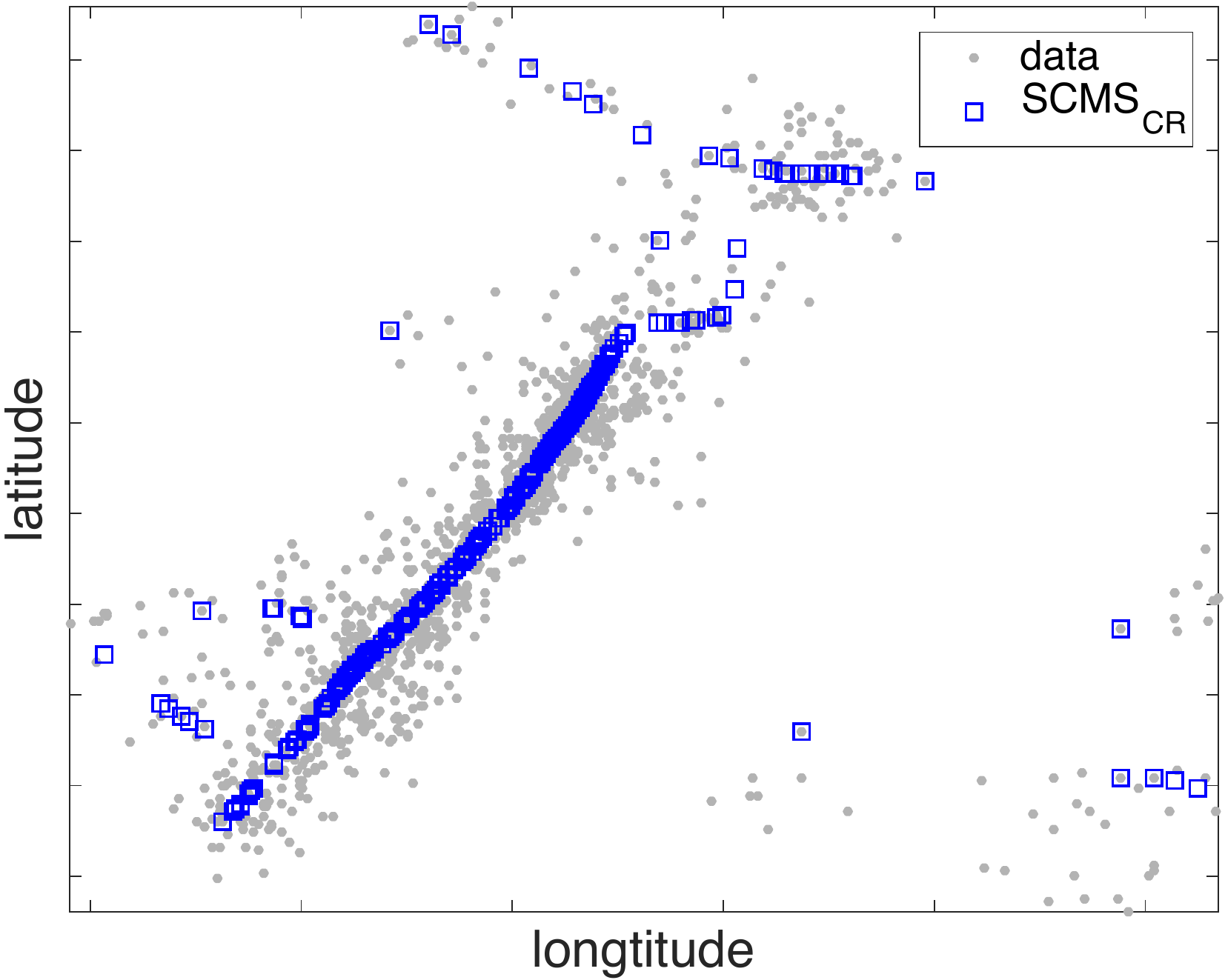}}
    \subfigure[Madrid~3]{\includegraphics[width=0.32\textwidth,clip]{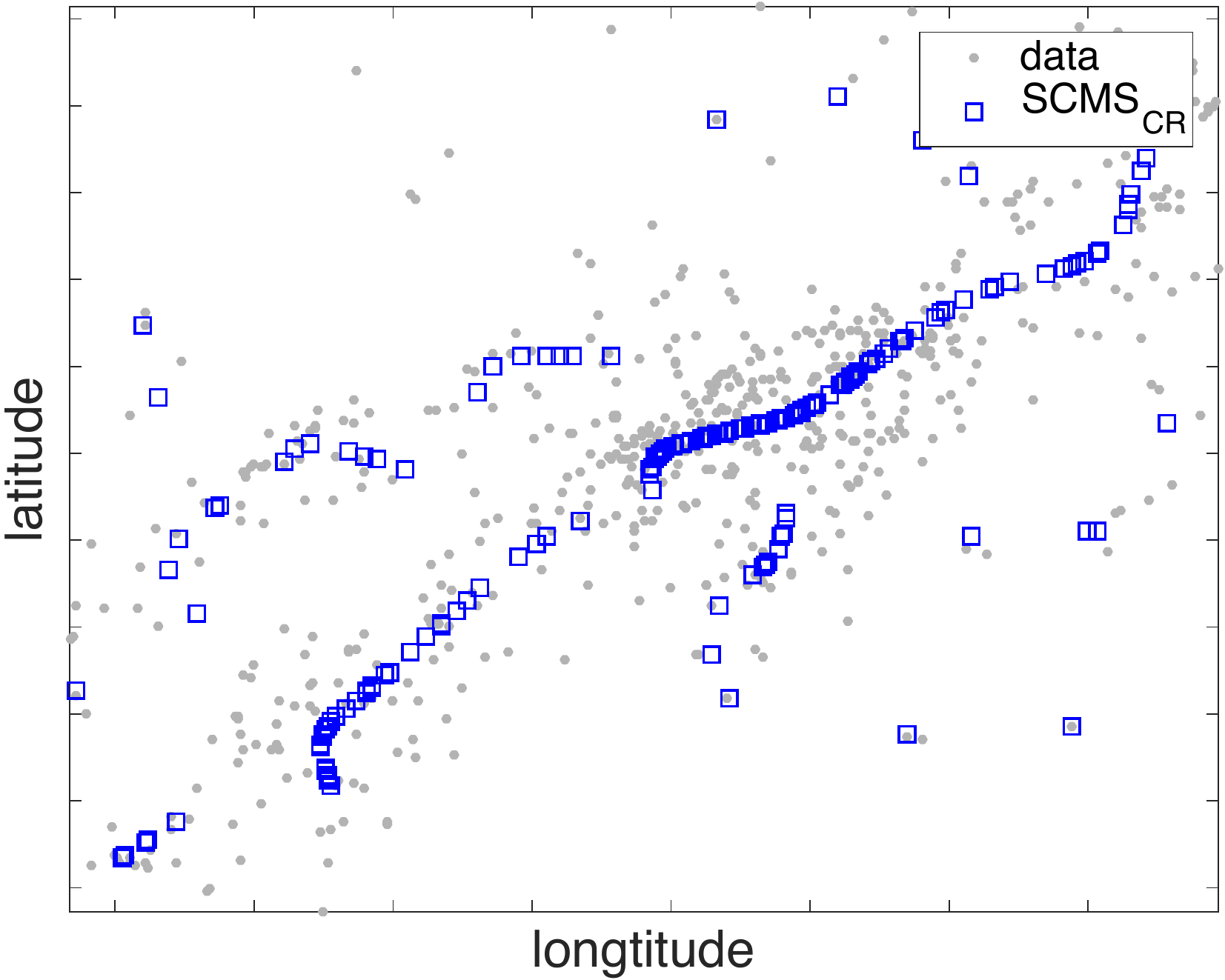}}
    \caption{\label{fig:Madrid} Density ridge estimation to three
    regions in the New Madrid earthquake dataset. The three regions
    (a,b,c) were extracted according to a range of latitude and
    longitude. The first, second and third rows correspond to results
    from LSDRF, SCMS$_{\text{LS}}$ and SCMS$_{\text{CR}}$,
    respectively.}
   \end{center}
   \end{figure}
   \begin{figure}[!t]
   \begin{center}
    \subfigure{\includegraphics[width=0.32\textwidth,clip]{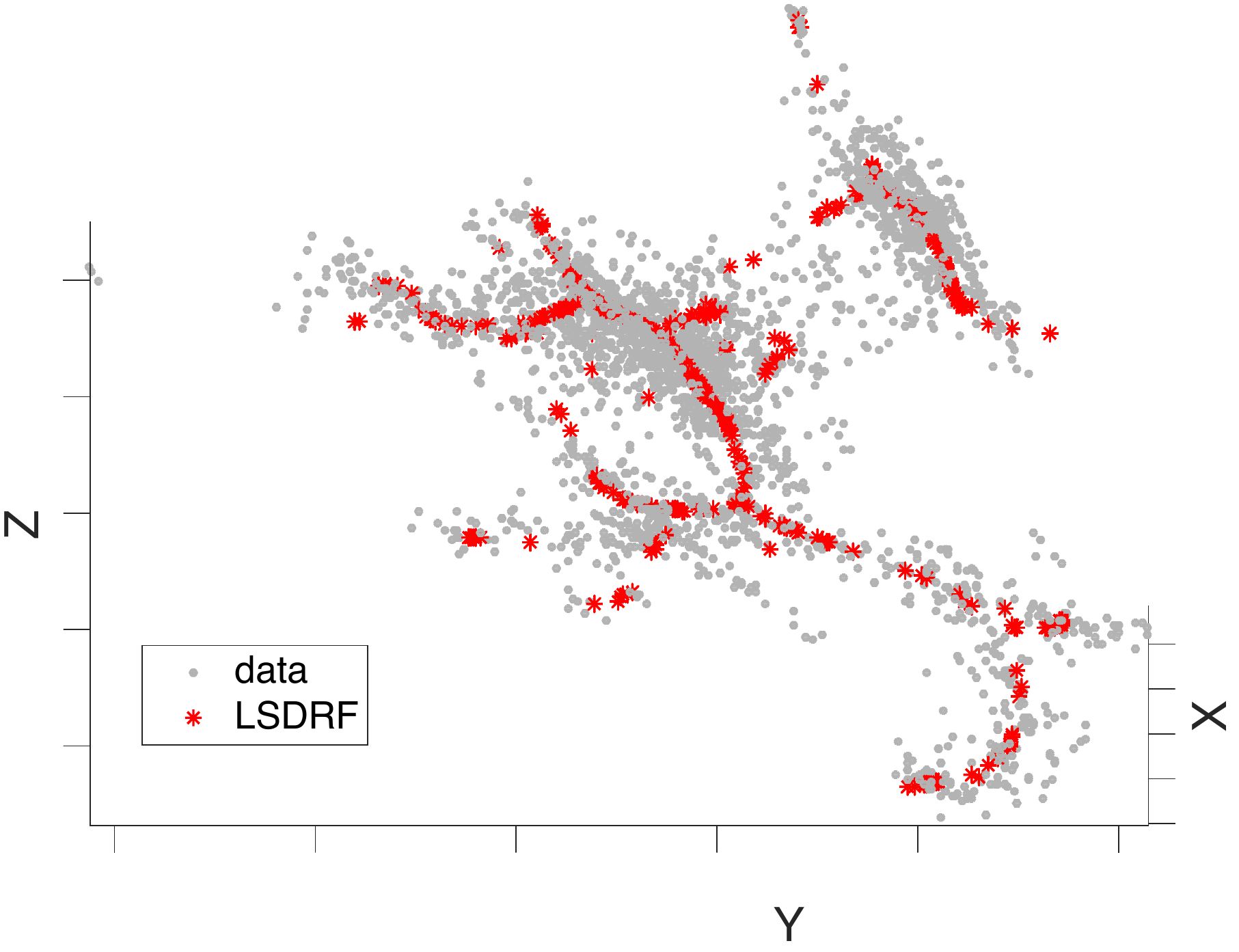}}
    \subfigure{\includegraphics[width=0.32\textwidth,clip]{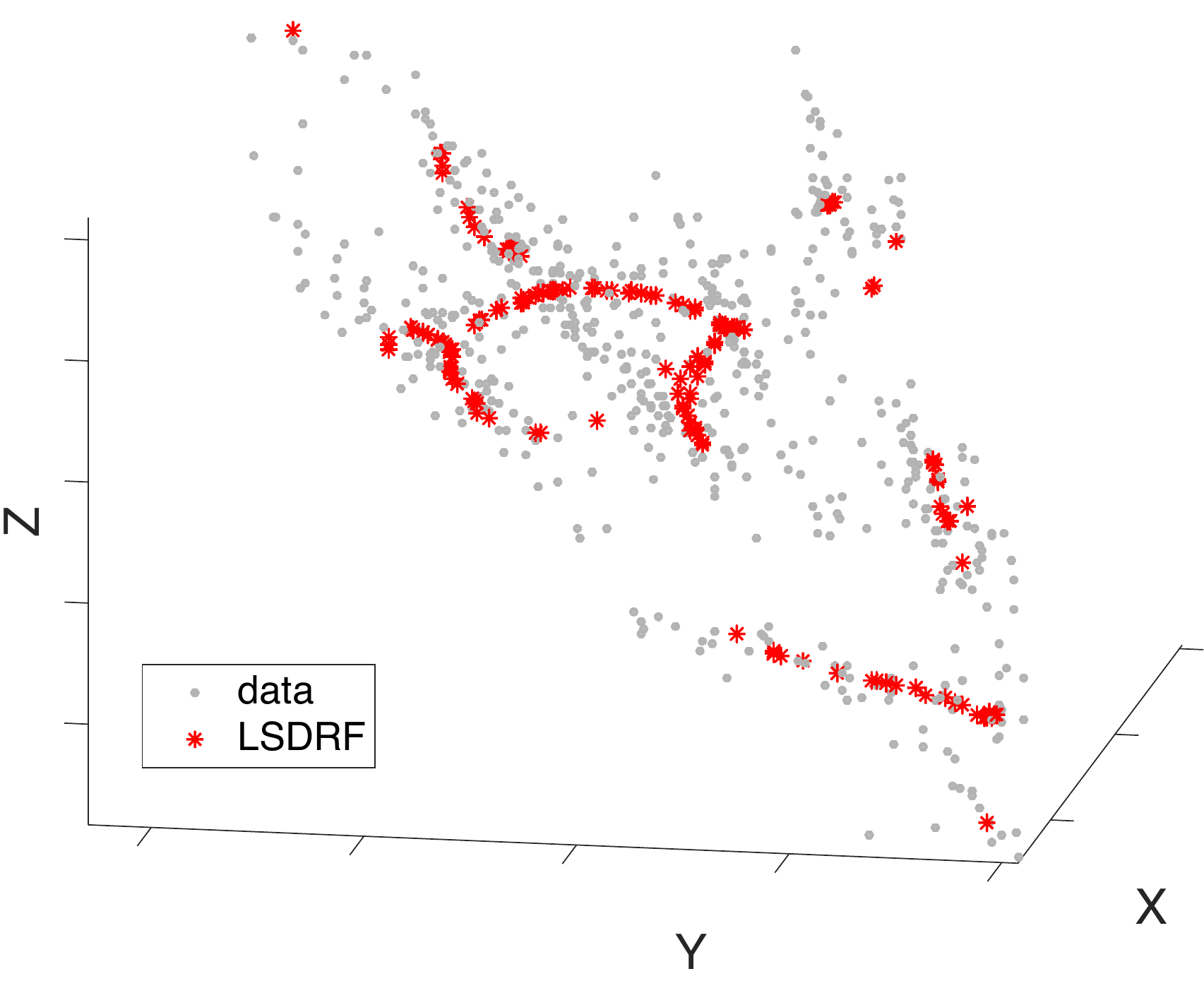}}
    \subfigure{\includegraphics[width=0.32\textwidth,clip]{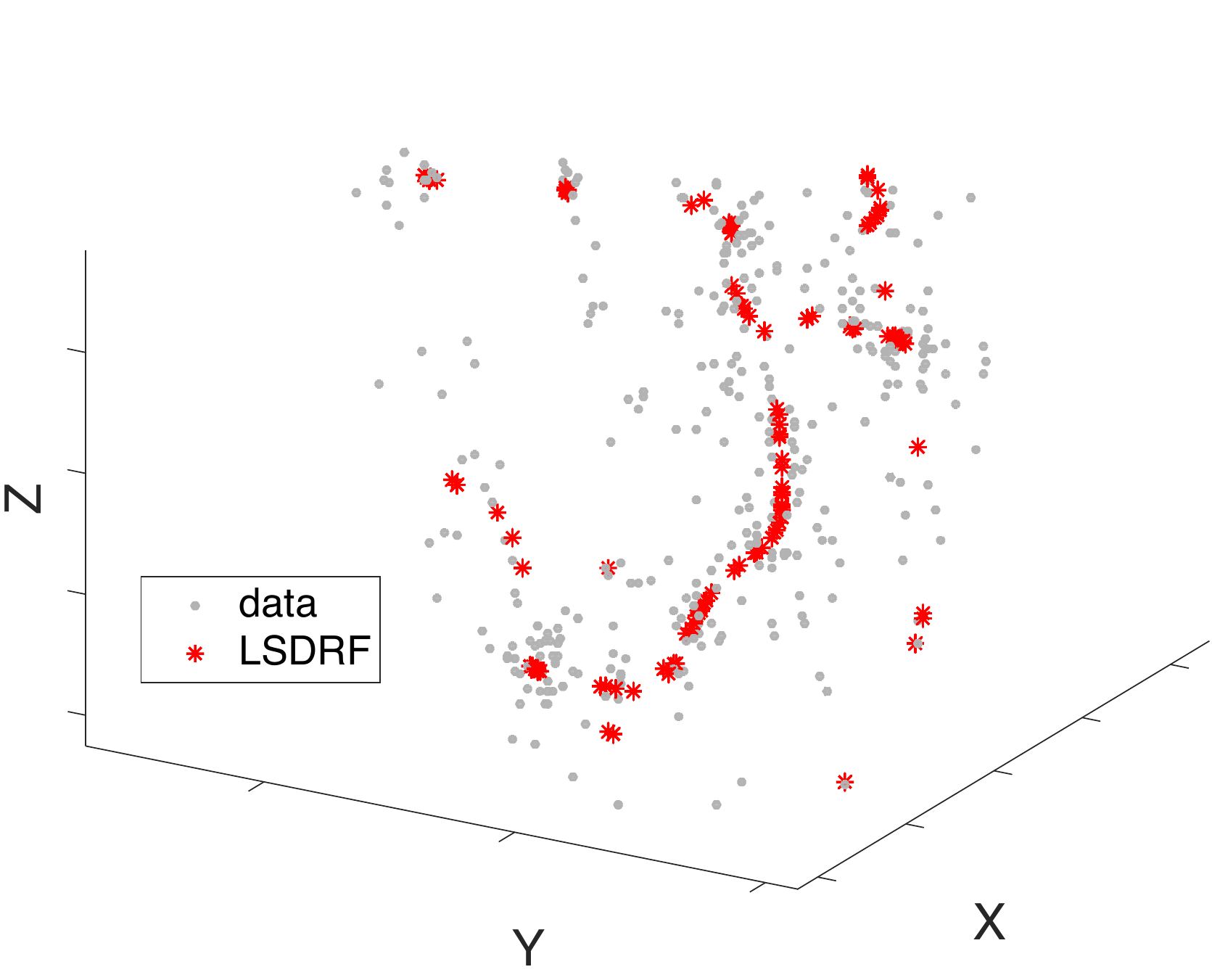}}
    \subfigure{\includegraphics[width=0.32\textwidth,clip]{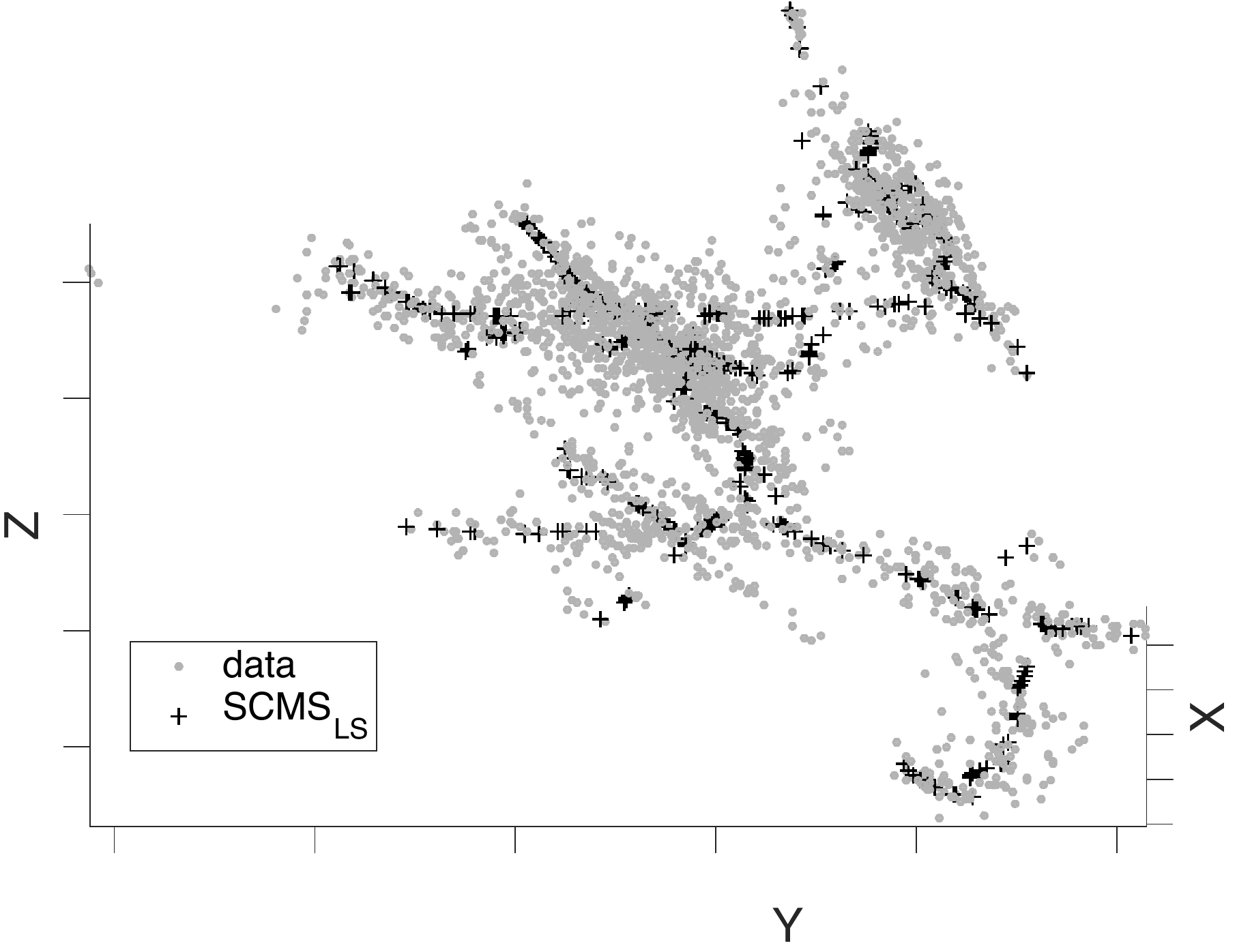}}
    \subfigure{\includegraphics[width=0.32\textwidth,clip]{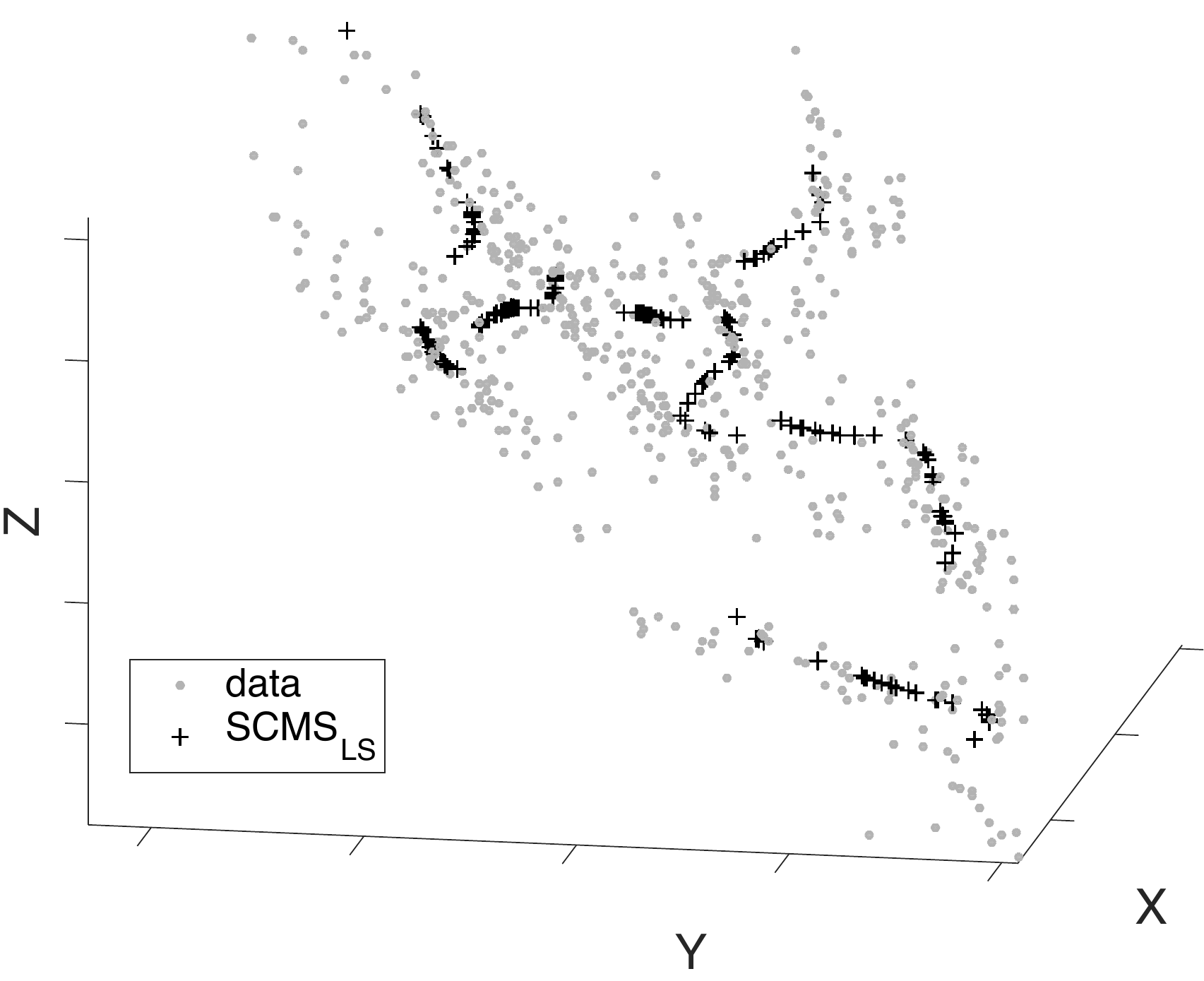}}
    \subfigure{\includegraphics[width=0.32\textwidth,clip]{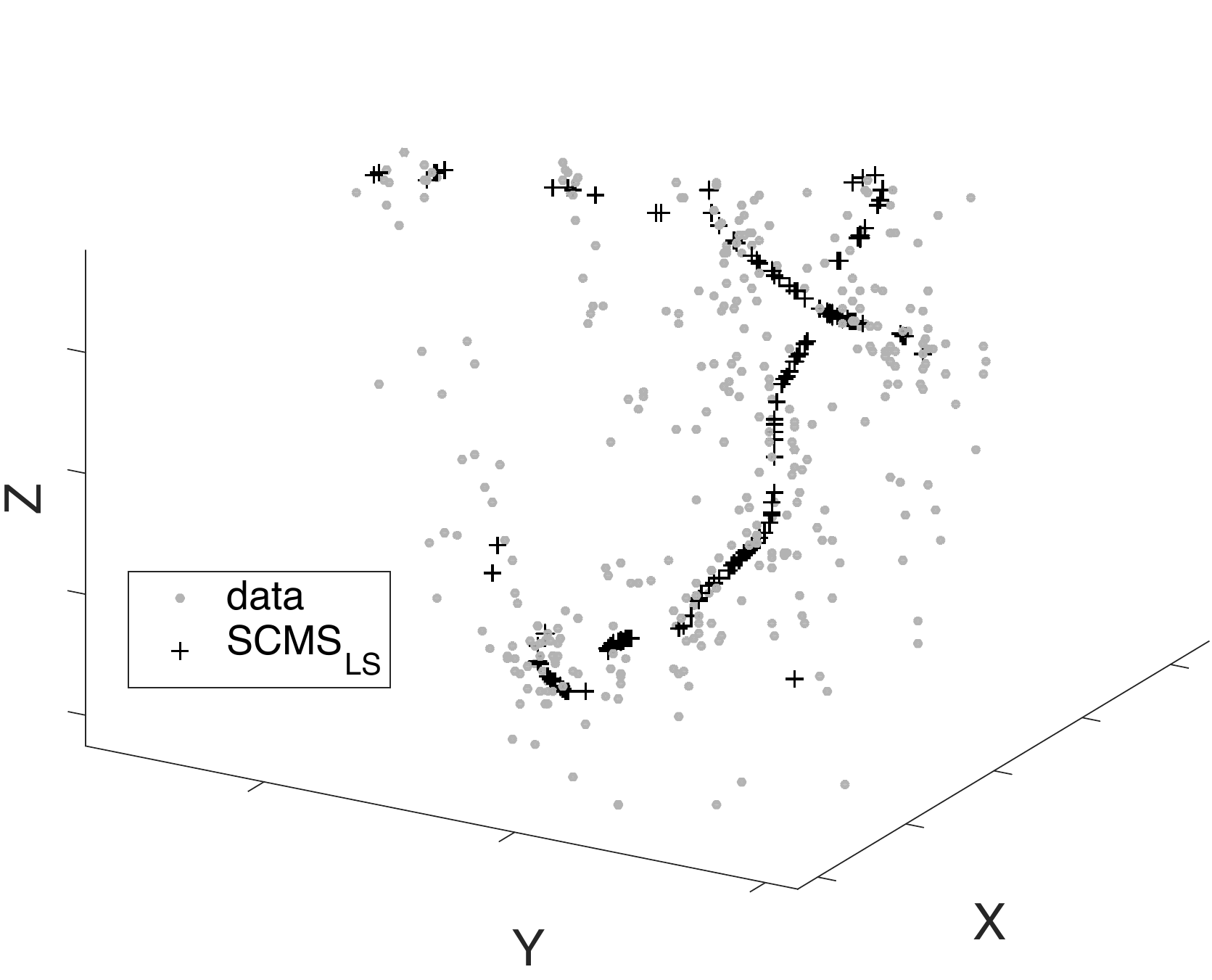}}
    \setcounter{subfigure}{0}
    \subfigure[Shapley~1]{\includegraphics[width=0.32\textwidth,clip]{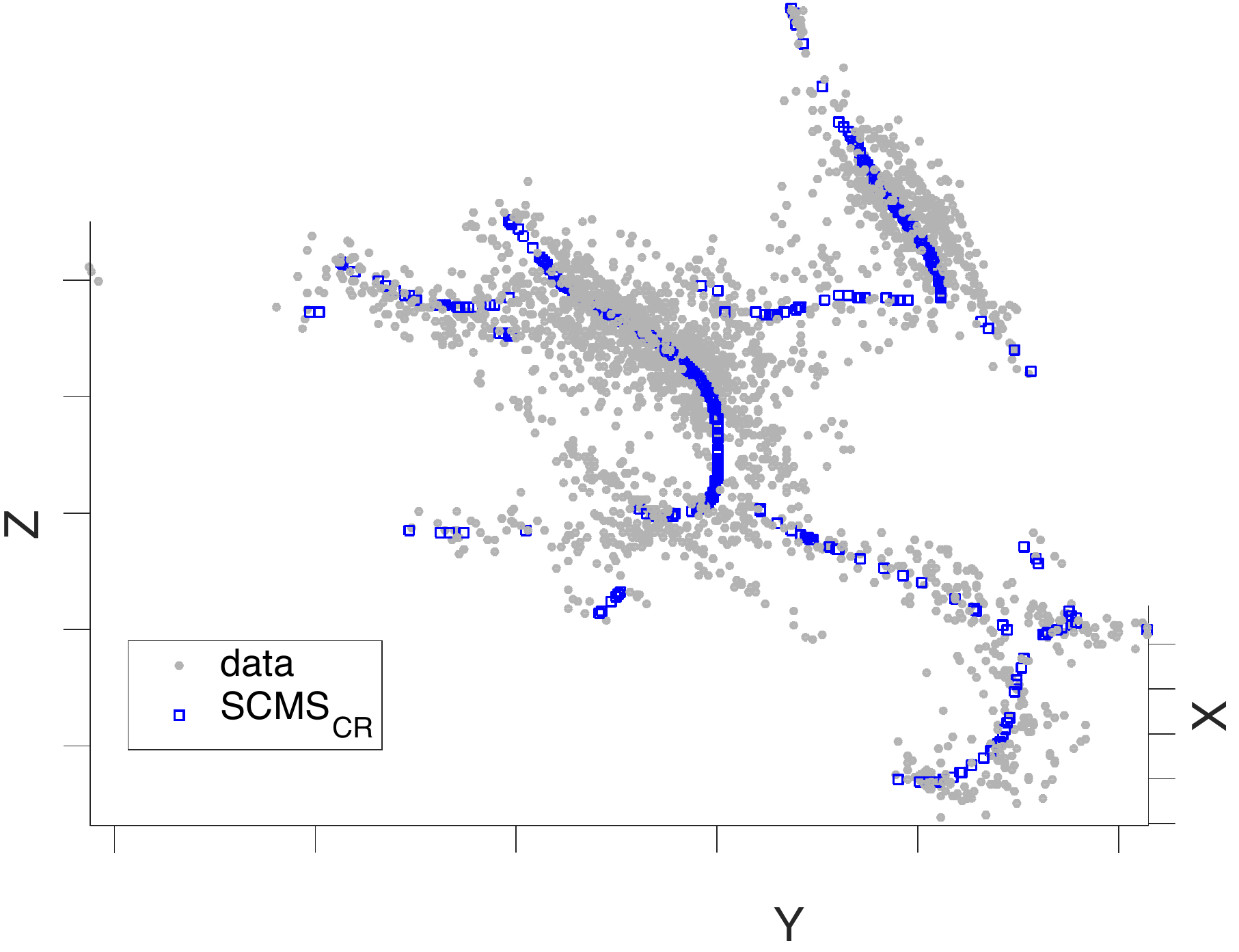}}
    \subfigure[Shapley~2]{\includegraphics[width=0.32\textwidth,clip]{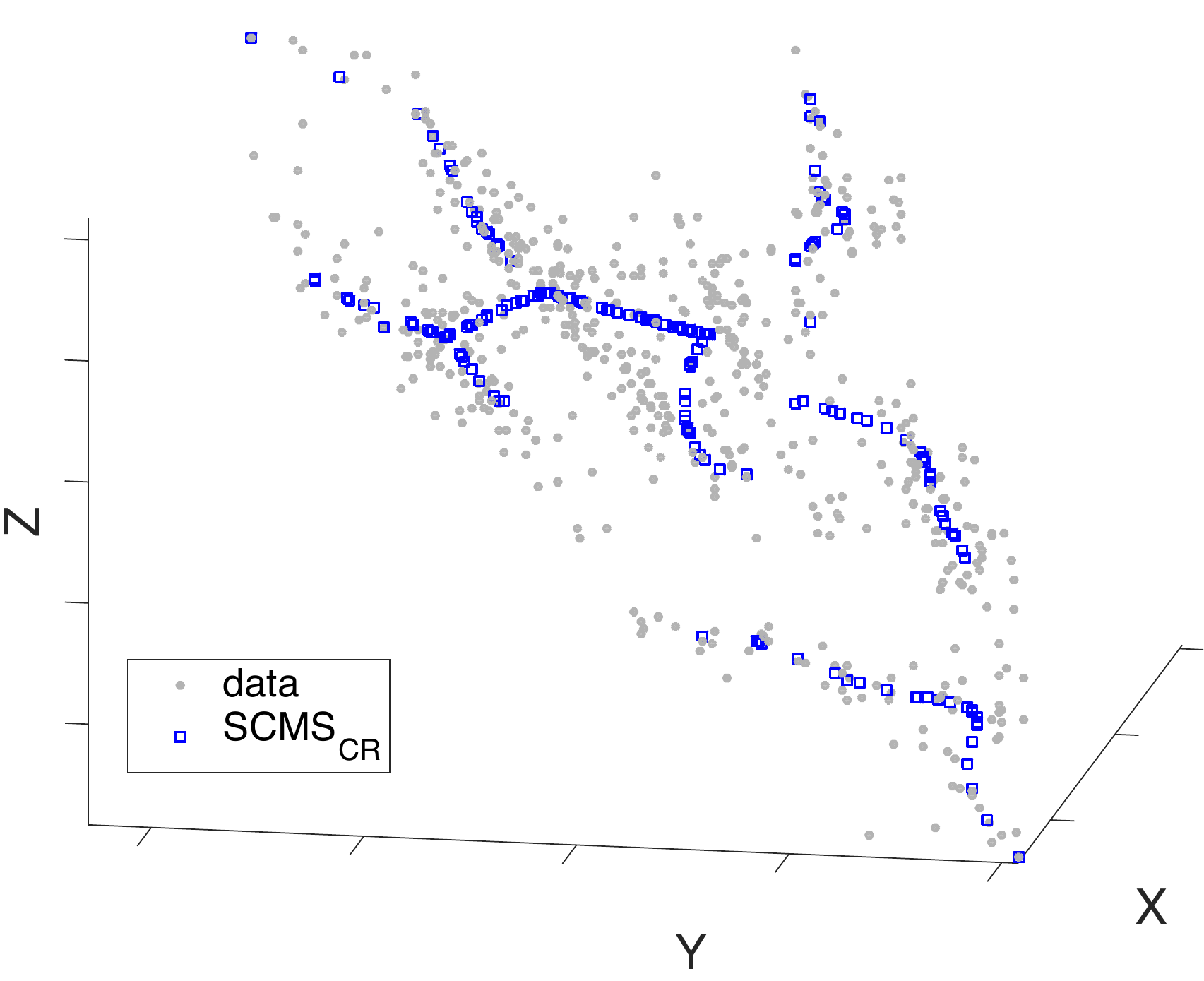}}
    \subfigure[Shapley~3]{\includegraphics[width=0.32\textwidth,clip]{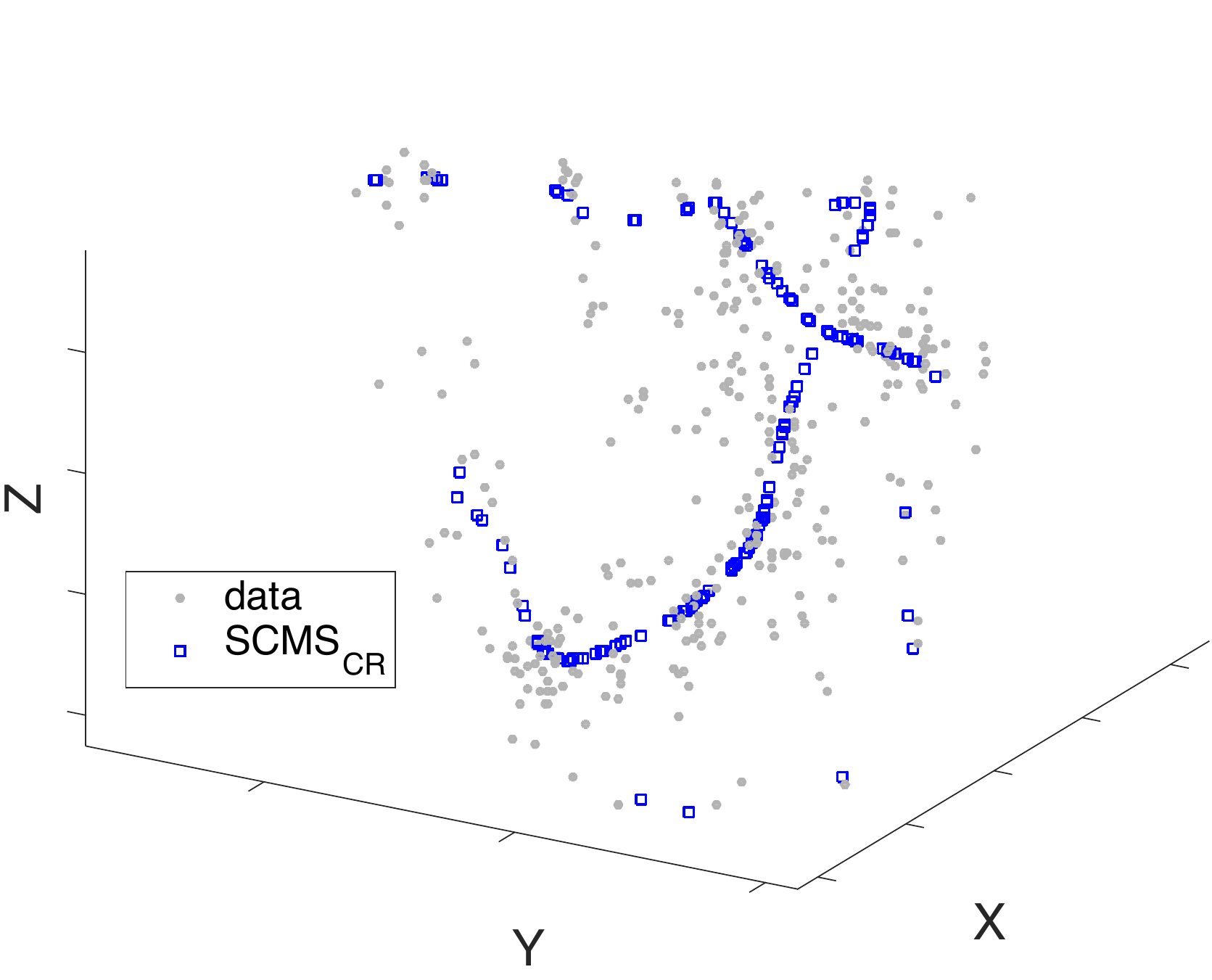}}
    \caption{\label{fig:Shapley} Density ridge estimation to three
    regions in the Shapley galaxy dataset. The three regions were
    extracted according to a range of recession velocity. The first,
    second and third rows correspond to results from LSDRF,
    SCMS$_{\text{LS}}$ and SCMS$_{\text{CR}}$, respectively.}
   \end{center}
   \end{figure}

   Next, we apply LSDRF to real-world datasets. As
   in~\citet{pulkkinen2015ridge}, we employed the following two
   datasets:
   \begin{itemize}
    \item \emph{New Madrid earthquake} dataset: This seismological
	  dataset was downloaded from the Center for Earthquake Research
	  and
	  Information.\footnote{\url{http://www.memphis.edu/ceri/seismic/}}
	  The dataset contains positional information for earthquakes
	  around the New Madrid seismic zone from $1974$ to $2016$,
	  providing $11,131$ samples.  The three regions in
	  Figs.\ref{fig:Madrid}(a,b,c) were extracted according to (a)
	  $(-90.2,-89.25)$, (b) $(-92.5,-92.15)$ and (c) $(-85.5,-83.5)$
	  degrees for the latitude range, respectively. For the
	  longitude range, (a) $(36,36.8)$, (b) $(35.2,35.4)$ and (c)
	  $(34.5,36.5)$ degrees were selected. The total numbers of the
	  original data samples and reduced data samples in each region
	  were (a) $(N,n)=(5902,500)$, (b) $(N,n)=(1548,300)$ and (c)
	  $(N,n)=(594,200)$.
	   
    \item \emph{Shapley galaxy} dataset: This dataset was downloaded
	  from the Center for Astrostatistics at Pennsylvania State
	  University.\footnote{\url{http://astrostatistics.psu.edu/datasets/Shapley_galaxy.html}}
	  The dataset contains information about the three-dimensional
	  sky angles and recession velocity of $4,215$ galaxies. As done
	  in~\citet{pulkkinen2015ridge}, we transformed the data samples
	  into the three-dimensional Cartesian coordinates based on the
	  fact that the recession velocity is proportional to the radial
	  distance~\citep{drinkwater2004large}. The three regions in
	  Figs~\ref{fig:Shapley}(a,b,c) were extracted according to a
	  velocity range: (a) $(6000,20000)$ km/s, (b) $(1500,6000)$
	  km/s and (c) $(6000,10500)$ km/s, respectively. The total
	  numbers of the original data samples and reduced data samples
	  in each region were (a) $(N,n)=(2849,500)$, (b)
	  $(N,n)=(595,200)$ and (c) $(N,n)=(351,150)$.
   \end{itemize}
   In each dataset, we focused on three regions containing prominent
   features, and standardized data samples in each region by subtracting
   the mean value and dividing by standard deviation in a dimension-wise
   manner. Here, the standardized data samples are collectively denoted
   by $\widetilde{\mathcal{D}}=\{\widetilde{\bm{x}}_i\}_{i=1}^N$.
   Before applying density ridge estimation methods, we performed
   preprocessing to remove clutter noises: KDE was applied to the
   dataset of each region, and then the data samples
   $\widetilde{\bm{x}}_i$ in each region were removed when
   $\frac{\widehat{p}_{\text{KDE}}(\widetilde{\bm{x}}_i)}
   {\max_{j}[\widehat{p}_{\text{KDE}}(\widetilde{\bm{x}}_j)]}<10^{-3}$.
   After noise removal, we randomly chose $n$ data samples from each
   region, and applied the three density ridge estimation methods to the
   sub-sampled data. The sub-sampled data are collectively expressed by
   $\mathcal{D}=\{\bm{x}_i\}_{i=1}^n$.  For performance comparison, we
   computed the logarithm of $\widehat{p}_{\text{KDE}}$ on the estimated
   density ridges, which is given by
   \begin{align*}
    \mathcal{L}=\frac{1}{n} \sum_{i=1}^{n}\log
    \widehat{p}_{\text{KDE}}(\widehat{\bm{y}}_i),
   \end{align*}
   where the centers of the kernel function in
   $\widehat{p}_{\text{KDE}}$ were set at the original data samples
   $\widetilde{\bm{x}}_i$ in each region, and $\widehat{\bm{y}}_i$
   denotes an estimated density ridge point from $\bm{x}_i$.  If
   $\mathcal{L}$ is larger, the performance can be interpreted to be
   better because ridges are defined on relatively high density
   areas. Unlike the last illustration, for SCMS$_{\LS}$, we employed
   the following adaptive-bandwidth Gaussian kernel:
   \begin{align*}
    \frac{1}{(2\pi
   h_i^2)^{D/2}}\exp\left(-\frac{\|\bm{x}-\bm{x}_i\|^2}{2h_{i}^2}\right),
   \end{align*}
   where $h_i$ denotes the bandwidth parameter. We restricted $h_{i}$ at
   the $m$-nearest neighbor Euclidean distance from $\bm{x}_i$ to
   $\bm{x}_j~(i\neq j)$, and performed cross-validation with respect to
   $m$ whose candidates were $128$, $64$, $32$, $16$, $8$ and $4$. For
   SCMS$_{\LS}$, the ten candidates of the bandwidth parameter were
   selected from $10^{l}\times h_{\text{NR}}$ ($-0.3\leq l \leq 0$). For
   LSDRF, we employed all data samples $\{\bm{x}_i\}_{i=1}^n$ as the
   centers of the Gaussian kernel, and used the median value of
   $\text{CV}(t)$ in Section~\ref{ssec:PracLSDDR} instead of the mean
   value in cross-validation.

   Ridges estimated by LSDRF are smooth and seem to qualitatively
   well-match the ridges in the underlying data, and SCMS$_{\text{CR}}$
   and SCMS$_{\text{LS}}$ also perform fairly good
   (Figs.\ref{fig:Madrid} and~\ref{fig:Shapley}).  Table~\ref{tab:ridge}
   is quantitative comparison by $\mathcal{L}$, showing that LSDRF
   compares favorably with both SCMS$_{\text{CR}}$ and
   SCMS$_{\text{LS}}$.
   
   \begin{table}[!t]
    \caption{\label{tab:ridge} The average and standard deviation of the
   performance measure $\mathcal{L}$ over $50$ runs. Madrid~1,~2 and~3
   (or Shapley~1,~2 and~3) correspond to the three regions in
   Fig.\ref{fig:Madrid} (or Fig.\ref{fig:Shapley}). A larger value means
   a better result. Numbers in the parentheses are standard deviations.
   The best and comparable methods judged by the unpaired t-test at the
   significance level $1\%$ are described in boldface.}
    \begin{center}
     \begin{tabular}{@{\ }c@{\ }||@{\ }c@{\ }|@{\ }c@{\ }|@{\ }c@{\ }c}
      \hline
      \multicolumn{4}{c}{New Madrid earthquake}\\
      \hline
      & LSDRF & SCMS$_{\text{LS}}$ & SCMS$_{\text{CR}}$ \\
      \hline
      Madrid~1 & {\bf  -0.511(0.101)} & -0.610(0.072) & -0.571(0.075)\\
      Madrid~2 & {\bf  0.001(0.175)} & {\bf 0.029(0.065)} & -0.076(0.075)\\
      Madrid~3 & {\bf -1.173(0.132)} & -1.238(0.086) & -1.238(0.098)\\
      \hline
     \end{tabular}
    \end{center}
    \begin{center}
     \begin{tabular}{@{\ }c@{\ }||@{\ }c@{\ }|@{\ }c@{\ }|@{\ }c@{\ }c}
      \hline
      \multicolumn{4}{c}{Shapley galaxy}\\
      \hline
      & LSDRF & SCMS$_{\text{LS}}$ & SCMS$_{\text{CR}}$ \\
      \hline
      Shapley~1 &  {\bf 0.188(0.093)} & 0.094(0.073) & 0.063(0.121)\\
      Shapley~2 &  {\bf -1.120(0.145)} & -1.220(0.097) & -1.462(0.223)\\  
      Shapley~3 & {\bf -1.295(0.114)} & -1.544(0.076) & -1.581(0.091)\\
      \hline
     \end{tabular}
    \end{center}
   \end{table}  
 \section{Conclusion}
 \label{sec:conclusion}
 In this paper, we proposed a novel estimator of the ratios of the
 density derivatives to the density. In stark contrast with the
 approaches in mean shift clustering and subspace constrained mean
 shift, our approach is to \emph{directly} estimate the
 \emph{density-derivative-ratios} without going through density
 estimation and computing the ratios. The proposed estimator was
 theoretically investigated, and the convergence rate was
 established. We applied the proposed estimator to mode-seeking
 clustering and density ridge estimation, and developed practical
 methods. Moreover, theoretical analysis were also performed to these
 methods , and the convergence rates to the mode and ridge of the true
 density were established. Our experimental illustration demonstrated
 that the proposed methods for mode-seeking clustering and density ridge
 estimation outperformed existing methods particularly for
 high(er)-dimensional data.
 
 This paper focused only on mode-seeking clustering and density ridge
 estimation. The proposed estimator can be useful or extended for other
 problems. For instance, making use of the global mode (the global
 maximum) of a conditional density enables us to develop a regression
 method robust against outliers~\citep{yao2012local}. Non-parametric
 estimation of the mode is also needed in functional data
 analysis~\citep{gasser1998nonparametric}. In future, we explore novel
 applications of the proposed estimator.
 
   \subsubsection*{Acknowledgements}
   The authors are grateful to Dr. Matthew James Holland for his helpful
   comments on an earlier version of this paper. Takafumi Kanamori was
   supported by KAKENHI 16K00044, 15H03636, and 15H01678.  Aapo
   Hyv{\"a}rinen was supported by the Academy of Finland. Gang Niu was
   supported by CREST JPMJCR1403. Masashi Sugiyama was supported by the
   International Research Center for Neurointelligence (WPI-IRCN) at The
   University of Tokyo Institutes for Advanced Study.
   
  \appendix
 \section{Proof of Theorem~\ref{theo:Hconvrate}}
 \label{app:Hconv}
 \begin{proof}  
  We first derive the following two lemmas by modifying the proof
  techniques in~\citet{sriperumbudur2013density}:
  \begin{lemma}
   With $\epsilon=1$ in Assumption~{\bf (D)}, the following statements hold:
   \begin{enumerate}[(i)]
    \item For $J_{\bm{j}}$ with the regularizer,
	  \begin{align*}
	   J_{\bm{j}}^{\lambda}(r_{\bm{j}}):=J_{\bm{j}}(r_{\bm{j}})
	   +\lambda_{\bm{j}}\|r_{\bm{j}}\|^2_{\rkhs},
	  \end{align*}
	  the minimizer of $J_{\bm{j}}^{\lambda}$ is given by
	  \begin{align*}
	   r_{\bm{j}}^{\lambda}
	   :=\argmin_{r_{\bm{j}}\in\rkhs} J_{\bm{j}}^{\lambda}(r_{\bm{j}})
	   =(C+\lambda_{\bm{j}} I)^{-1}\xi_{\bm{j}}
	   =(C+\lambda_{\bm{j}} I)^{-1}Cr_{\bm{j}}^*,
	  \end{align*}
	  where $C=\int_{\X} k(\cdot,\bm{x})\otimes
	  k(\cdot,\bm{x})p(\bm{x})\intd\bm{x}$, $\otimes$ is the tensor
	  product, and 
	  \begin{align*}
	   \xi_{\bm{j}}:=(-1)^{|\bm{j}|}\int_{\X} \partial_{\bm{j}}
	   k(\cdot,\bm{x})p(\bm{x})\intd\bm{x}.
	  \end{align*}
		  
    \item $\widehat{J}_{\bm{j}}(r_{\bm{j}})$ can be equivalently
	  expressed as
	  \begin{align*}
	   \widehat{J}_{\bm{j}}(r_{\bm{j}})
	   =\< r_{\bm{j}}-r_{\bm{j}}^*, \widehat{C}
	   (r_{\bm{j}}-r_{\bm{j}}^*)\>_{\rkhs},
	  \end{align*}
	  where
	  \begin{align*}
	   \widehat{C}:=\frac{1}{n}\sum_{i=1}^n
	   k(\cdot,\bm{x}_i)\otimes k(\cdot,\bm{x}_i)\quad \text{and}\quad
	   \widehat{\xi}_{\bm{j}}:=\frac{(-1)^{|\bm{j}|}}{n}\sum_{i=1}^n
	   \partial_{\bm{j}} k(\cdot,\bm{x}_i).	   
	  \end{align*}	  
	  Then, $\widehat{r}_{\bm{j}}$ is given by
	  \begin{align*}
	   \widehat{r}_{\bm{j}}
	   =\argmin_{r_{\bm{j}}\in\rkhs}\left[\widehat{J}_{\bm{j}}(r_{\bm{j}})
	   +\lambda_{\bm{j}}\|r_{\bm{j}}\|^2_{\rkhs}\right]
	   =(\widehat{C}+\lambda_{\bm{j}} I)^{-1}\widehat{\xi}_{\bm{j}}.
	  \end{align*}
   \end{enumerate}
   \label{lemma1}
  \end{lemma}
  \begin{lemma}
   With $\epsilon=2$ in Assumption~{\bf (D)},
   \begin{align}    
    \|\widehat{\xi}_{\bm{j}}-\widehat{C}r_{\bm{j}}^{*}\|_{\rkhs}
    =O_{\mathrm{P}}\left(n^{-1/2}\right).\label{firstterm}    
   \end{align}
   \label{lemma2}
  \end{lemma}
  The proofs of these lemmas can be seen in
  Appendices~\ref{app:proofoflemmaHconv} and~\ref{app:proofoflemma2},
  respectively.

  Next, we make use of the proof of Theorem~5
  in~\citet{sriperumbudur2013density} to prove
  Theorem~\ref{theo:Hconvrate}. From Lemma~\ref{lemma1},
  \begin{align*}
   \widehat{r}_{\bm{j}}-r_{\bm{j}}^{\lambda}
   &=(\widehat{C}+\lambda_{\bm{j}} I)^{-1}
   \widehat{\xi}_{\bm{j}}-r_{\bm{j}}^{\lambda}\\
   &=(\widehat{C}+\lambda_{\bm{j}} I)^{-1}\left\{
   \widehat{\xi}_{\bm{j}}-\widehat{C}r_{\bm{j}}^{\lambda}
   -\lambda_{\bm{j}} r_{\bm{j}}^{\lambda}\right\}\\
   &=(\widehat{C}+\lambda_{\bm{j}} I)^{-1}
   (\widehat{\xi}_{\bm{j}}-\widehat{C}r_{\bm{j}}^{*})
   +(\widehat{C}+\lambda_{\bm{j}} I)^{-1}
   (C-\widehat{C})(r_{\bm{j}}^{\lambda}-r_{\bm{j}}^{*}), 
  \end{align*}
  where we used
  $\lambda_{\bm{j}}r_{\bm{j}}^{\lambda}=C(r_{\bm{j}}^{*}-r_{\bm{j}}^{\lambda})$
  from Lemma~\ref{lemma1}(i). Therefore,
  \begin{align*}
   \|\widehat{r}_{\bm{j}}-r_{\bm{j}}^*\|_{\rkhs} &\leq
   \|\widehat{r}_{\bm{j}}-r_{\bm{j}}^{\lambda}\|_{\rkhs}
   +\|r_{\bm{j}}^{\lambda}-r^*_{\bm{j}}\|_{\rkhs}.\\ &\leq
   \|(\widehat{C}+\lambda_{\bm{j}} I)^{-1}\|
   (\|\widehat{\xi}_{\bm{j}}-\widehat{C}r_{\bm{j}}^{*}\|_{\rkhs}
   +\|C-\widehat{C}\|\mathcal{A}_0(\lambda_{\bm{j}}))+\mathcal{A}_0(\lambda_{\bm{j}}),
  \end{align*}
  where $\mathcal{A}_0(\lambda_{\bm{j}})
  =\|r_{\bm{j}}^{\lambda}-r_{\bm{j}}^{*}\|_{\rkhs}$. It can be shown
  that $\|(\widehat{C}+\lambda_{\bm{j}} I)^{-1}\|\leq
  1/\lambda_{\bm{j}}$ for sufficiently small $\lambda_{\bm{j}}$.  Thus,
  Lemma~\ref{lemma2} shows that the first term can be bounded by
  $O_{\mathrm{P}}\left(\frac{1}{\lambda_{\bm{j}}\sqrt{n}}\right)$. In
  addition, with the proof techniques
  in~\citet[Lemma~5]{fukumizu2007statistical},
  $\|C-\widehat{C}\|\leq\|C-\widehat{C}\|_{\text{HS}}=O_{\mathrm{P}}(n^{-1/2})$
  with $\epsilon=2$ where $\|\cdot\|_{\HS}$ denotes the Hilbert-Schmidt
  norm.  Thus, the second term is of the order
  $O_{\mathrm{P}}\left(\frac{\mathcal{A}_0(\lambda_{\bm{j}})}{\sqrt{n}\lambda_{\bm{j}}}\right)$.
  From these results,
  \begin{align}
   \|\widehat{r}_{\bm{j}}-r_{\bm{j}}^*\|_{\rkhs} &\leq
   O_{\mathrm{P}}\left(\frac{1}{\lambda_{\bm{j}}\sqrt{n}}\right)
   +O_{\mathrm{P}}\left(\frac{\mathcal{A}_0(\lambda_{\bm{j}})}{\sqrt{n}\lambda_{\bm{j}}}\right)
   +\mathcal{A}_0(\lambda_{\bm{j}}).
   \label{upper-bound-rkhs}
  \end{align}
  Propostion~A.2 in~\cite{sriperumbudur2013density} states that if
  $r_{\bm{j}}^*\in\mathcal{R}(C^{\gamma})$ and $C$ is a bounded and
  self-adjoint compact operator on a separable $\rkhs$, the following
  inequality holds:
  \begin{align}
   \mathcal{A}_0(\lambda_{\bm{j}})\leq\max(1,\|C\|^{\gamma-1})
   \lambda_{\bm{j}}^{\min(1,\gamma)}\|C^{-\gamma}r_{\bm{j}}^*\|_{\rkhs}.
   \label{bound-A0}
  \end{align}
  It can be easily verified that $C$ is a self-adjoint operator.
  Assumption~(D) with $\epsilon=2$ ensures that $C$ is a Hilbert-Schmidt
  operator and therefore compact because it is bounded in terms of the
  Hilbert-Schmidt norm. Thus, applying~\eqref{bound-A0}
  to~\eqref{upper-bound-rkhs} completes the proof when choosing
  $\lambda_{\bm{j}}=
  O\left(n^{-\max\left\{\frac{1}{4},\frac{1}{2(\gamma+1)}\right\}}\right)$
  as $n\rightarrow\infty$.
 \end{proof}
  \subsection{Proof of Lemma~\ref{lemma1}}
  \label{app:proofoflemmaHconv}
  \begin{proof}
   {\it (i)} From the definition of $J_{\bm{j}}$,
   \begin{align*}
    J_{\bm{j}}(r_{\bm{j}})
    &=\int_{\X}\left\{r_{\bm{j}}(\bm{x})-r_{\bm{j}}^*(\bm{x})\right\}^2
    p(\bm{x})\intd\bm{x}\\
    &=\int_{\X} \< r_{\bm{j}}-r_{\bm{j}}^*, k(\cdot,\bm{x})\>_{\rkhs}^2
    p(\bm{x})\intd\bm{x}\\
    &=\int_{\X} \<r_{\bm{j}}-r_{\bm{j}}^*, (k(\cdot,\bm{x})\otimes
    k(\cdot,\bm{x}))(r_{\bm{j}}-r_{\bm{j}}^*)\>_{\rkhs} p(\bm{x})\intd\bm{x}\\
    &=\int_{\X} \< r_{\bm{j}}-r_{\bm{j}}^*,
    C_{\mathrm{x}}(r_{\bm{j}}-r_{\bm{j}}^*)\>_{\rkhs}p(\bm{x})\intd\bm{x}\\
    &=\< r_{\bm{j}}-r_{\bm{j}}^*, C(r_{\bm{j}}-r_{\bm{j}}^*)\>_{\rkhs},
   \end{align*}
   where $C_{\mathrm{x}}:=k(\cdot,\bm{x})\otimes
   k(\cdot,\bm{x})$. Expanding the right-hand side above transforms
   $J_{\bm{j}}^{\lambda}$ as
   \begin{align}
    J_{\bm{j}}^{\lambda}(r_{\bm{j}})
    &=\< r_{\bm{j}},Cr_{\bm{j}}\>_{\rkhs}
    -2\< r_{\bm{j}},Cr_{\bm{j}}^*\>_{\rkhs}
    +\< r_{\bm{j}}^*,Cr_{\bm{j}}^*\>_{\rkhs}
    +\lambda_{\bm{j}}\< r_{\bm{j}},r_{\bm{j}}\>_{\rkhs}\nonumber\\
    &=\< r_{\bm{j}},(C+\lambda_{\bm{j}}I)r_{\bm{j}}\>_{\rkhs}
    -2\< r_{\bm{j}},Cr_{\bm{j}}^*\>_{\rkhs}
    +\< r_{\bm{j}}^*,Cr_{\bm{j}}^*\>_{\rkhs}.
    \label{eqn:J_rkhs}
   \end{align}
   For the second term in~\eqref{eqn:J_rkhs}, we compute
   \begin{align}
    \< r_{\bm{j}}, Cr_{\bm{j}}^*\>_{\rkhs} &=\< r_{\bm{j}}, \int_{\X}
    k(\cdot,\bm{x})r_{\bm{j}}^*(\bm{x}) p(\bm{x})\intd\bm{x}\>_{\rkhs}
    \nonumber\\
    &=\< r_{\bm{j}}, \int_{\X} k(\cdot,\bm{x}) \partial_{\bm{j}}
    p(\bm{x})\intd\bm{x}\>_{\rkhs}\nonumber\\ &=\< r_{\bm{j}}, 
    (-1)^{|\bm{j}|}\int_{\X}
    \partial_{\bm{j}} k(\cdot,\bm{x})p(\bm{x})\intd\bm{x}\>_{\rkhs}\nonumber\\
    &=\< r_{\bm{j}}, \xi_{\bm{j}}\>_{\rkhs},
    \label{eqn:secondterm}
   \end{align}
   where we applied Assumption {\bf (C)}, and
   \begin{align*}
    \xi_{\bm{j}}=(-1)^{|\bm{j}|}\int_{\X}
    \partial_{\bm{j}} k(\cdot,\bm{x})p(\bm{x})\intd\bm{x}.
   \end{align*}
   Comparing the left-hand side with the right-hand side at the last
   line in~\eqref{eqn:secondterm} gives
   \begin{align}
    Cr_{\bm{j}}^*=\xi_{\bm{j}}. \label{Cr*xi}
   \end{align}
   Eq.\eqref{Cr*xi} is valid because \eqref{eqn:secondterm} holds for
   arbitrary $r_{\bm{j}}\in\rkhs$.
   
   Simple calculation after substituting~\eqref{eqn:secondterm}
   into~\eqref{eqn:J_rkhs} provides
   \begin{align*}
    J_{\bm{j}}^{\lambda}(r_{\bm{j}})
    &=\|(C+\lambda_{\bm{j}} I)^{1/2}r_{\bm{j}}-(C+\lambda_{\bm{j}}
    I)^{-1/2}\xi_{\bm{j}}\|_{\rkhs}^2 -\<\xi_{\bm{j}},
    (C+\lambda_{\bm{j}} I)^{-1}\xi_{\bm{j}}\>_{\rkhs}
    +\< r_{\bm{j}}^*,Cr_{\bm{j}}^*\>_{\rkhs},
   \end{align*}
   Since the second and third terms in the right-hand side above do not
   include $r_{\bm{j}}$, the minimizer of
   $J_{\bm{j}}^{\lambda}(r_{\bm{j}})$ is given by
   $r_{\bm{j}}^{\lambda}=(C+\lambda_{\bm{j}} I)^{-1}\xi_{\bm{j}}
   =(C+\lambda_{\bm{j}} I)^{-1}Cr_{\bm{j}}^*$ where \eqref{Cr*xi} was
   applied.
   
   {\it (ii)} It follows from (i) by substituting $C$ and $\xi_{\bm{j}}$
   with $\widehat{C}$ and $\widehat{\xi}_{\bm{j}}$, respectively.
  \end{proof}
  \subsection{Proof of Lemma~\ref{lemma2}}
  \label{app:proofoflemma2}
  \begin{proof}
   We first compute the expectation of
   $\|\widehat{\xi}_{\bm{j}}-\widehat{C}r_{\bm{j}}^{*}\|_{\rkhs}^2$ as
   \begin{align}
    E \|\widehat{\xi}_{\bm{j}}-\widehat{C}r_{\bm{j}}^{*}\|_{\rkhs}^2
    =\frac{n-1}{n} \|\xi_{\bm{j}}-Cr_{\bm{j}}^{*}\|^2_{\rkhs}
    +\frac{1}{n}\int_{\X} \|(-1)^{|\bm{j}|}\partial_{\bm{j}}k(\cdot,\bm{x})
    +C_{\mathrm{x}}r_{\bm{j}}^{*}\|^2_{\rkhs}p(\bm{x})\intd\bm{x},
    \label{Explemma2}
   \end{align}
   where $C_{\mathrm{x}}=k(\cdot,\bm{x})\otimes k(\cdot,\bm{x})$.
   Eq.\eqref{Cr*xi} indicates that the first term in the right-hand side
   of~\eqref{Explemma2} vanishes, i.e.,
   $\|\xi_{\bm{j}}-Cr_{\bm{j}}^{*}\|_{\rkhs}=0$. From
   \begin{align*}
    \|(-1)^{|\bm{j}|}\partial_{\bm{j}}k(\cdot,\bm{x})
    +C_{\mathrm{x}}r_{\bm{j}}^{*}\|^2_{\rkhs}
    \leq 2\|\partial_{\bm{j}}k(\cdot,\bm{x})\|^2_{\rkhs}
    +2\|C_{\mathrm{x}}\|^2_{\text{HS}}\|r_{\bm{j}}^{*}\|^2_{\rkhs},
   \end{align*}
   Assumption~(D) with $\epsilon=2$ ensures that the second term in the
   right-hand side of~\eqref{Explemma2} is finite. Thus, applying the
   Chebyshev's inequality proves the lemma because
   $E(\widehat{\xi}_{\bm{j}}-\widehat{C}r_{\bm{j}}^{*})=\xi_{\bm{j}}-Cr_{\bm{j}}^{*}=0$
   from~\eqref{Cr*xi}.
  \end{proof}
 \section{Connection to the Minimax Theory}
 \label{app:minimax}

  This appendix provides details for the connections to the minimax
  theory discussed in the remark after
  Theorem~\ref{theo:Hconvrate}. First, we introduce the following
  results:
  \begin{itemize}
   \item By the minimax theory~\citep{tsybakov2008introduction}, Eq.(10)
	 in~\citet{JMLR:v18:16-011} shows the minimax rate: For any
	 $\alpha>\delta\geq0$,
	 \begin{align}
	  \inf_{\widehat{r}_{\bm{j},n}}\sup_{r_{\bm{j}}^*\in H^{\alpha}_2}
	  \|\widehat{r}_{\bm{j},n}-r_{\bm{j}}^*\|_{H^{\delta}_2} \asymp
	  n^{-\frac{\alpha-\delta}{2(\alpha-\delta)+D}}. \label{minimax-rate}
	 \end{align}
	 
   \item The following proposition provides necessary conditions for
	 $r_{\bm{j}}^*\in R(C)$:
	 \begin{proposition}
	  Suppose that $\psi,\phi\in C(\R{D})\cap L^1(\R{D})$ are
	  real-valued, shift-invariant and positive definite kernel
	  functions. Let $\rkhs$ and $\mathcal{G}$ be RKHSs associated
	  with $\psi(\bm{x}-\bm{y})$ and $\phi(\bm{x}-\bm{y})$,
	  respectively. For $2\leq r\leq\infty$, assume that the
	  followings hold,
	 \begin{align*}
	  p\in L^{\frac{r}{r-1}}(\R{D}),\quad
	  \left\|\frac{\phi^{\wedge}}{\psi^{\wedge}}\right\|_{\infty}<\infty
	  \quad\text{and}\quad
	  \left\|\frac{\psi^{\wedge2}}{\phi^{\wedge}}\right\|_{\frac{r}{r-2}}<\infty.
	 \end{align*}
	  Then, $r_{\bm{j}}^*\in R(C)$ implies that $r_{\bm{j}}^*\in
	  \mathcal{G}\subset\rkhs$, where
	  $k(\bm{x},\bm{y})=\psi(\bm{x}-\bm{y})$ in the operator
	  $C$. \label{prop:necesaryconditions}
	 \end{proposition}
	 The proof of Proposition~\ref{prop:necesaryconditions} is
	 deferred to Section~\ref{ssec:proof-necessaryconditions}.  The
	 conditions are necessary ones for $r_{\bm{j}}^*\in
	 R(C^{\beta})$ with $\beta>1$ as well because
	 $R(C^{\beta_1})\subset R(C^{\beta_2})$ for
	 $0<\beta_1<\beta_2<\infty$~\citep[Section~4.2 and
	 Appendix~B.3]{JMLR:v18:16-011}.
  \end{itemize}
  
  Recall that when the \emph{Mat{\'e}rn kernel},
  $k(\bm{x},\bm{y})=\psi(\bm{x}-\bm{y})=\frac{2^{1-s}}{\Gamma(s)}
  \|\bm{x}-\bm{y}\|^{s-D/2}\mathfrak{K}_{D/2-s}(\|\bm{x}-\bm{y}\|)$, is
  employed, the corresponding RKHS $\rkhs$ is the Sobolev space
  $H^{s}_2$ with $s>D/2$~\citep[Chapter~10]{wendland2004scattered}:
  \begin{align*}
   \rkhs=H^{s}_2:=\left\{f\in L^2(\R{D})\cap C(\R{D}) :
   \int(1+\|\bm{\omega}\|^2)^{s}|f^{\wedge}(\bm{\omega})|^2
   \intd\bm{\omega}<\infty\right\}.
  \end{align*}
  Theorem~6.13 in~\citet{wendland2004scattered} gives the Fourier
  transform of $\psi$ as
  \begin{align*}
   \psi^{\wedge}(\bm{\omega})=(1+\|\bm{\omega}\|^2)^{-s}.
  \end{align*}
  When $p\in L^{1}(\R{D})$, applying
  Proposition~\ref{prop:necesaryconditions} ensures that $r_{\bm{j}}^*\in
  R(C)$ implies $r_{\bm{j}}^*\in H^{s'}_2\subset H^{s}_2$ with
  $\frac{D}{2}<s\leq s'< 2s+\frac{1}{2}-\frac{D}{2}$. Thus,
  $r_{\bm{j}}^*\in H^{2s+\frac{1}{2}-\frac{D}{2}-\epsilon}_2$ for
  arbitrarily small $\epsilon>0$. Then, if we chose
  $\rkhs=H^{D-\frac{1}{2}+\epsilon}_2$, the rate $n^{-\frac{1}{4}}$ in
  Theorem~\ref{theo:Hconvrate} is minimax optimal (Set
  $\alpha=2s+\frac{1}{2}-\frac{D}{2}-\epsilon$ and $\delta=s$
  in~\eqref{minimax-rate}, equate the exponent in the right-hand side
  of~\eqref{minimax-rate} with $-\frac{1}{4}$, and solve it with respect
  to $s$). Similar discussion is possible when $p\in L^{2}(\R{D})$: The
  rate is minimax optimal under the choice of
  $\rkhs=H^{\frac{D}{2}+\epsilon}_2$.
  \subsection{Proof of Proposition~\ref{prop:necesaryconditions}}
  \label{ssec:proof-necessaryconditions}
  Here, we modify the proof of Proposition~8 in~\citet{JMLR:v18:16-011}.
  \begin{proof}
   To characterize RKHSs induced by shift-invariant kernels, we employ
   the following lemma:
   \begin{lemma}[Theorem~10.12 in~\citet{wendland2004scattered}]
    Let $\psi(\bm{x}-\bm{y})$ be a real-valued, symmetric and positive
    definite kernel. When $\psi\in C(\R{D})\cap L^1(\R{D})$, it induces
    the following Hilbert space,
    \begin{align*}
     \rkhs:=\left\{f\in C(\R{D})\cap L^2(\R{D}) :
     \frac{f^{\wedge}}{\sqrt{\psi^{\wedge}}}\in L^2(\R{D})\right\},
    \end{align*}
    with the reproducing kernel $\psi(\bm{x}-\bm{y})$ and inner product,
    \begin{align*}
     \<f,g\>_{\rkhs}:=
     \frac{1}{(2\pi)^{D/2}}\int\frac{f^{\wedge}(\bm{\omega})\overline{g^{\wedge}(\bm{\omega})}}
     {\psi^{\wedge}(\bm{\omega})}\intd\bm{\omega}.
    \end{align*}
    $\overline{g^{\wedge}(\bm{\omega})}$ above denotes the complex
    conjugate of $g^{\wedge}(\bm{\omega})$. In particular, every $f$ in
    $\rkhs$ can be recovered from its Fourier transform $f^{\wedge}\in
    L^1(\R{D})\cap L^2(\R{D})$ as
    \begin{align}
     f(\bm{x})=\frac{1}{(2\pi)^{D/2}}\int
     f^{\wedge}(\bm{\omega})e^{i\bm{x}^{\top}\bm{\omega}}
     \intd\bm{\omega}.  \label{inverseFourier}
    \end{align} \label{RKHS-Shift-Invariant}
   \end{lemma}  
   
   Let us express an RKHS $\mathcal{G}$ induced by another real-valued,
   symmetric and positive definite kernel $\phi(\bm{x}-\bm{y})$. We
   first show that $\mathcal{G}\subset\rkhs$ if
   $\left\|\frac{\phi^{\wedge}(\bm{\omega})}{\psi^{\wedge}(\bm{\omega})}\right\|_{\infty}<\infty$.
   From Lemma~\ref{RKHS-Shift-Invariant}, for $g\in\mathcal{G}$, the
   norm in $\rkhs$ is computed as
   \begin{align*}
    \|g\|_{\rkhs}^2=
    \frac{1}{(2\pi)^{D/2}}\int\frac{|g^{\wedge}(\bm{\omega})|^2}
    {\psi^{\wedge}(\bm{\omega})}\intd\bm{\omega}
    =\frac{1}{(2\pi)^{D/2}}\int\frac{|g^{\wedge}(\bm{\omega})|^2\phi^{\wedge}(\bm{\omega})}
    {\phi^{\wedge}(\bm{\omega})\psi^{\wedge}(\bm{\omega})}\intd\bm{\omega}
    \leq\|g\|_{\mathcal{G}}^2
    \left\|\frac{\phi^{\wedge}(\bm{\omega})}{\psi^{\wedge}(\bm{\omega})}\right\|_{\infty}<\infty.
   \end{align*}
   Thus, $g\in\rkhs$, which indicates that $\mathcal{G}\subset\rkhs$.
   
   Next, we show that $r_{\bm{j}}^*\in R(C)$ indicates
   $r_{\bm{j}}^*\in\mathcal{G}$. Since $r_{\bm{j}}^*\in R(C)$, there
   exists $f\in\rkhs$ such that $r_{\bm{j}}^* =Cf$, i.e.,
   \begin{align}
    r_{\bm{j}}^*(\bm{y})&=\int k(\bm{x},\bm{y})f(\bm{x})
    p(\bm{x})\intd\bm{x}\nonumber\\ &=\int\psi(\bm{x}-\bm{y})f(\bm{x})
    p(\bm{x})\intd\bm{x}\nonumber\\ &=\int\left[\frac{1}{(2\pi)^{D/2}}\int
    \psi^{\wedge}(\bm{\omega})e^{i(\bm{x}-\bm{y})^{\top}\bm{\omega}}
    \intd\bm{\omega}\right]f(\bm{x}) p(\bm{x})\intd\bm{x}\nonumber\\
    &=\int\left[\frac{1}{(2\pi)^{D/2}}\int f(\bm{x})
    p(\bm{x})e^{i\bm{x}^{\top}\bm{\omega}}
    \intd\bm{x}\right]\psi^{\wedge}(\bm{\omega})e^{-i\bm{y}^{\top}\bm{\omega}}\intd\bm{\omega}
    \nonumber\\&=\int (f^{\wedge}*p^{\wedge})(-\bm{\omega})\psi^{\wedge}(\bm{\omega})
    e^{-i\bm{y}^{\top}\bm{\omega}}\intd\bm{\omega},\label{Fourier-r*}
   \end{align}
   where we applied~\eqref{inverseFourier} to $\psi(\bm{x}-\bm{y})$
   on the third line and Fubini's theorem on the fourth line, and $*$
   denotes the convolution such that
   \begin{align*}
    (f*p)(\bm{x}):=\int f(\bm{y})p(\bm{x}-\bm{y})\intd\bm{y}.
   \end{align*}
   Eq.\eqref{Fourier-r*} indicates that the Fourier transform of
   $r_{\bm{j}}^{*}$ is given by
   \begin{align*}
    r_{\bm{j}}^{*\wedge}(\bm{\omega})
    =(f^{\wedge}*p^{\wedge})(-\bm{\omega})\psi^{\wedge}(\bm{\omega}).
   \end{align*}
   Computing the norm of $r^*_{\bm{j}}$ in $\mathcal{G}$ yields
   \begin{align*}
    \|r^*_{\bm{j}}\|_{\mathcal{G}}^2
    =\int |(f^{\wedge}*p^{\wedge})(-\bm{\omega})|^2
    \frac{\psi^{\wedge}(\bm{\omega})^2}{\phi^{\wedge}(\bm{\omega})}\intd\bm{\omega}
    &\leq\|(f^{\wedge}*p^{\wedge})^2\|_{r/2}
    \left\|\frac{\psi^{\wedge2}}{\phi^{\wedge}}\right\|_{\frac{r}{r-2}}
    =\|f^{\wedge}*p^{\wedge}\|_{r}^2
    \left\|\frac{\psi^{\wedge2}}{\phi^{\wedge}}\right\|_{\frac{r}{r-2}},
   \end{align*}
   where H{\"o}lder inequality was applied with $2\leq
   r\leq\infty$. Then, Young's convolution and Hausdorff-Young
   inequalities~\citep{beckner1975inequalities} yield
   \begin{align*}
    \|f^{\wedge}*p^{\wedge}\|_{r}
    \leq\|f^{\wedge}\|_{1}\cdot\|p^{\wedge}\|_{r}
    \leq\|f^{\wedge}\|_{1}\cdot\|p\|_{\frac{r}{r-1}}<\infty
   \end{align*}
   Thus, by Lemma~\ref{RKHS-Shift-Invariant}, $r_{\bm{j}}^{*}\in R(C)$
   indicates $r_{\bm{j}}^{*}\in\mathcal{G}$.
  \end{proof}
 \section{Proof of Theorem~\ref{theo:LSLDGCconv}}
 \label{app:LSLDGCconv}
 \begin{proof}
  Suppose that $\widehat{\alpha}_{j}^{(i)}=0$ and
  $\widetilde{\beta}_{j}^{(i)}\left(=-\widehat{\beta}_{j}^{(i)}\right)\geq0$
  for all $i$ and $j$. Computing the integral in~\eqref{mypathintegral}
  shows that
  \begin{align}
   \widehat{D}_{\widehat{\bm{g}}}[\bm{x}|\bm{y}]&=
   \sum_{j=1}^D \int_{y^{(j)}}^{x^{(j)}}
   \widehat{g}_{j}(x^{(1)},\dots,x^{(j-1)},
   z^{(j)},y^{(j+1)},\dots,y^{(D)})\intd
   z^{(j)}\nonumber\\
   &=\sum_{j=1}^D\sum_{i=1}^n \widetilde{\beta}^{(i)}_{j}\left[
   \phi\left(\frac{\|\bm{z}_x^{j}-\bm{x}_i\|^2}{2\sigma_j^2}\right)
   -\phi\left(\frac{\|\bm{z}_y^{j}-\bm{x}_i\|^2}{2\sigma_j^2}\right)\right],
   \label{one-integral}
  \end{align}
  where we used the relation
  $\partial_j'\phi\left(\|\bm{x}-\bm{x}'\|^2\right)=-\partial_j\phi\left(\|\bm{x}-\bm{x}'\|^2\right)$,
  and
  \begin{align}
   \bm{z}_y^j&=(x^{(1)},\dots,x^{(j-1)},y^{(j)},y^{(j+1)},\dots,y^{(D)})^{\top}
   \nonumber\\
   \bm{z}_x^j&=(x^{(1)},\dots,x^{(j-1)},x^{(j)},y^{(j+1)},\dots,y^{(D)})^{\top}.
   \label{twopoints}
  \end{align}
  Note that the $j$-th elements in $\bm{z}_y^{j}$ and $\bm{z}_x^{j}$
  only differ. To ensure that the right-hand side
  in~\eqref{one-integral} is non-negative, we need to show that for all
  $j$,
  \begin{align}
   \sum_{i=1}^n \widetilde{\beta}^{(i)}_{j}\left[
   \phi\left(\frac{\|\bm{z}_x^{j}-\bm{x}_i\|^2}{2\sigma_j^2}\right)
   -\phi\left(\frac{\|\bm{z}_y^{j}-\bm{x}_i\|^2}{2\sigma_j^2}\right)\right]
   \geq0.\label{conv-goal}   
  \end{align}

  To obtain a lower bound of the left-hand side in~\eqref{conv-goal}, we
  use the following inequality, which comes from the convexity of
  $\phi$:
  \begin{align}
   \phi\left(\frac{\|\bm{z}_x^{j}-\bm{x}_i\|^2}{2\sigma_j^2}\right)
   &-\phi\left(\frac{\|\bm{z}_y^{j}-\bm{x}_i\|^2}{2\sigma_j^2}\right)
   \nonumber\\
   &\geq\frac{1}{2\sigma_j^2}
   \varphi\left(\frac{\|\bm{z}_y^{j}-\bm{x}_i\|^2}{2\sigma_j^2}\right)
   \left[(y^{(j)}-x_i^{(j)})^2-(x^{(j)}-x_i^{(j)})^2\right].\label{convex}
  \end{align}
  Since all $\widetilde{\beta}^{(i)}_{j}$ are assumed to be
  non-negative, \eqref{convex} provides
  \begin{align}
   &\sum_{i=1}^n \widetilde{\beta}^{(i)}_{j} \left[
   \phi\left(\frac{\|\bm{z}_x^{j}-\bm{x}_i\|^2}{2\sigma_j^2}\right)
   -\phi\left(\frac{\|\bm{z}_y^{,j}-\bm{x}_i\|^2}{2\sigma_j^2}\right)   
   \right]\nonumber\\
   &\geq\frac{1}{2\sigma_j^2}\sum_{i=1}^n \widetilde{\beta}^{(i)}_{j}
   \varphi\left(\frac{\|\bm{z}_y^{j}-\bm{x}_i\|^2}{2\sigma_j^2}\right)
   \left[(y^{(j)}-x_i^{(j)})^2-(x^{(j)}-x_i^{(j)})^2\right]\nonumber\\
   &=\frac{1}{2\sigma_j^2}\sum_{i=1}^n\widetilde{\beta}^{(i)}_{j}
   \varphi\left(\frac{\|\bm{z}_y^{j}-\bm{x}_i\|^2}{2\sigma_j^2}\right)
   \left[(y^{(j)})^2-(x^{(j)})^2\right]
   -\underbrace{\sum_{i=1}^n\widetilde{\beta}^{(i)}_{j} x_i^{(j)}
   \varphi\left(\frac{\|\bm{z}_y^{j}-\bm{x}_i\|^2}{2\sigma_j^2}\right)}_{(\star)}
   \frac{(y^{(j)}-x^{(j)})}{\sigma_j^2}.\label{convex-lower-bound}
  \end{align}
  Finally, we set $\bm{y}=\bm{z}_k^{\tau}$ and
  $\bm{x}=\bm{z}_k^{\tau+1}$ in
  $\widehat{D}_{\widehat{\bm{g}}}[\bm{x}|\bm{y}]$, and therefore
  \begin{align*}
   \bm{z}_y^j&=(z_k^{(\tau+1,1)},\dots,z_k^{(\tau+1,j-1)},z_k^{(\tau,j)},z_k^{(\tau,j+1)},\dots,z_k^{(\tau,D)})^{\top}=\tilde{\bm{z}}_k^{\tau}
   \nonumber\\
   \bm{z}_x^j&=(z_k^{(\tau+1,1)},\dots,z_k^{(\tau+1,j-1)},z_k^{(\tau+1,j)},z_k^{(\tau,j+1)},\dots,z_k^{(\tau,D)})^{\top}.
  \end{align*}
  Applying the coordinate-wise update rule~\eqref{CWupdate} to
  $(\star)$, the right-hand side in~\eqref{convex-lower-bound} becomes
  \begin{align*}
   &\frac{1}{2\sigma_j^2}\sum_{i=1}^n\widetilde{\beta}^{(i)}_{j}
   \varphi\left(\frac{\|\tilde{\bm{z}}_k^{\tau}-\bm{x}_i\|^2}{2\sigma_j^2}\right)
   (z_k^{(\tau,j)}-z_k^{(\tau+1,j)})^2\geq0.
  \end{align*}
  This proves~\eqref{conv-goal}, and thus the proof was completed.
 \end{proof}

 \section{Proof of Theorem~\ref{theo:nonintegrable-field}}
 \label{app:proofnonintegrable-field}
 \begin{proof}
  Under the path~\eqref{path}, 
  \begin{align*}
   D_{\g}[\x|\y]-\widehat{D}_{\widehat{\g}}[\x|\y]=
   \sum_{j=1}^D\int_{j-1}^j
   \<\g(\bm{\gamma}_j(t))-\widehat{\g}(\bm{\gamma}_j(t)),\,\dot{\bm{\gamma}}_j(t)\>\intd t
  \end{align*}
  where the curve $\bm{\gamma}_j(t),\,t\in[j-1,j]$ connects $\bm{z}_y^j$
  and $\bm{z}_x^j$ by the line segment whose definition is given
  in~\eqref{twopoints}.  Then, we obtain
  \begin{align*}
   \left|
   \int_{j-1}^j 
   \<\g(\bm{\gamma}_j(t))-\widehat{\g}(\bm{\gamma}_j(t)),\,\dot{\bm{\gamma}}_j(t)\>\intd t
   \right|
   \leq
   \|\g-\widehat{\g}\|_\infty |y^{(j)}-x^{(j)}|. 
  \end{align*}
  Therefore,
  \begin{align*}
   |D_{\g}[\x|\y]-D_{\widehat{\g}}[\x|\y]|\leq \|\g-\widehat{\g}\|_\infty\|\y-\x\|_1.
  \end{align*}
  Finally, with Lemma~\ref{Epsiconv}, the theorem was proved.
 \end{proof}
  
 \section{Proof of Theorem~\ref{theo:Modeconv}}
 \label{app:proofModeconv}
 We modify the proof of Theorem~1 in~\citet{chen2016comprehensive}, and
 apply Lemma~\ref{Epsiconv}. 
 \begin{proof}
  Suppose that a mode point $\bm{\mu}_j\in\mathcal{M}$ is uniquely
  approximated by an estimated mode point
  $\widehat{\bm{\mu}}_j\in\widehat{\mathcal{M}}$. Then, the Taylor
  expansion gives
  \begin{align}
   \widehat{\bm{g}}(\bm{\mu}_j)&=\widehat{\bm{g}}(\widehat{\bm{\mu}}_j)
   +\nabla\widehat{\bm{g}}(\bm{\mu}_j)
   (\bm{\mu}_j-\widehat{\bm{\mu}}_j)+o(\|\bm{\mu}_j-\widehat{\bm{\mu}}_j\|)
   \nonumber\\&=\nabla\widehat{\bm{g}}(\bm{\mu}_j)
   (\bm{\mu}_j-\widehat{\bm{\mu}}_j)+o(\|\bm{\mu}_j-\widehat{\bm{\mu}}_j\|),
   \label{g-hat-at-mode-no1}
  \end{align}
  where $\widehat{\bm{g}}(\widehat{\bm{\mu}}_j)=\bm{0}$. On the other
  hand, from Lemma~\ref{Epsiconv},
  \begin{align}
   \widehat{g}_j(\bm{\mu}_j)=\widehat{g}_j(\bm{\mu}_j)-
   g_j(\bm{\mu}_j)=O_{\mathrm{P}}\left(n^{-\min\left\{\frac{1}{4},
   \frac{\gamma}{2(\gamma+1)}\right\}}\right),
   \label{g-hat-at-mode-no2}
  \end{align}
  where $g_j(\bm{\mu}_j)=0$. Since all eigenvalues of
  $\nabla\bm{g}(\bm{\mu}_j)$ are strictly negative by the definition
  in~\eqref{definition-of-mode-set}, the following relation and
  Lemma~\ref{Epsiconv} ensures that $\nabla\widehat{\bm{g}}(\bm{\mu}_j)$
  is invertible with a high probability: By the derivative reproducing
  property~\citep{zhou2008derivative},
  \begin{align*}
   |\partial_i g_j(\bm{\x})-\partial_i \widehat{g}_j(\bm{\x})|
   =|\<g_j-\widehat{g}_j,\partial_i k(\bm{x},\cdot)\>_{\rkhs}|
   \leq\|g_j-\widehat{g}_j\|_{\rkhs}
   |\partial_i^{\prime}\partial_ik(\bm{x}^{\prime},\bm{x})|_{\bm{x}^{\prime}=\bm{x}}|
   =O\left(\|g_j-\widehat{g}_j\|_{\rkhs}\right),
  \end{align*}
  where the Cauchy-Schwarz inequality was applied, $\partial_i^{\prime}$
  denotes the partial derivative with respect to the $i$-th element in
  $\bm{x}^{\prime}$, and $\partial_i^{\prime}\partial_ik$ is assumed to
  be uniformly bounded. Thus, combining~\eqref{g-hat-at-mode-no1}
  with~\eqref{g-hat-at-mode-no2} yields
  \begin{align*}
   \|\bm{\mu}_j-\widehat{\bm{\mu}}_j\|
   =O_{\mathrm{P}}\left(n^{-\min\left\{\frac{1}{4},
   \frac{\gamma}{2(\gamma+1)}\right\}}\right).
  \end{align*}
  The fact,
  \begin{align*}
   \text{Haus}(\widehat{\mathcal{M}},\mathcal{M})
   =\max_{j} \|\bm{\mu}_j-\widehat{\bm{\mu}}_j\|,
  \end{align*}
  proves the theorem.
 \end{proof}

 \section{Reducing the Kernel Centers}
 \label{app:kernelcenter}
 \begin{figure}[!t]
  \begin{center}
   \subfigure[ARI]{\includegraphics[width=0.32\textwidth]{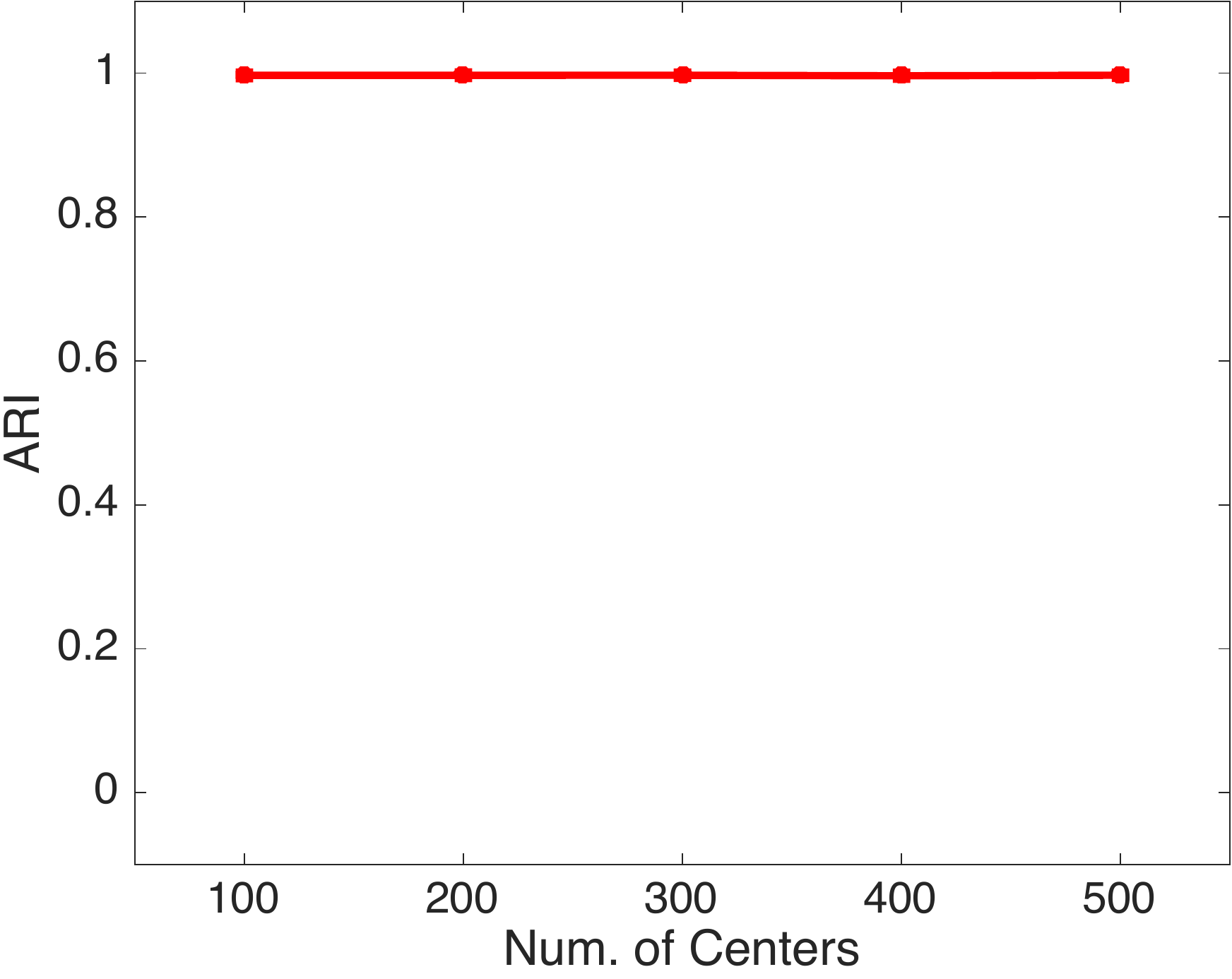}}
   \subfigure[CPU
   time]{\includegraphics[width=0.32\textwidth]{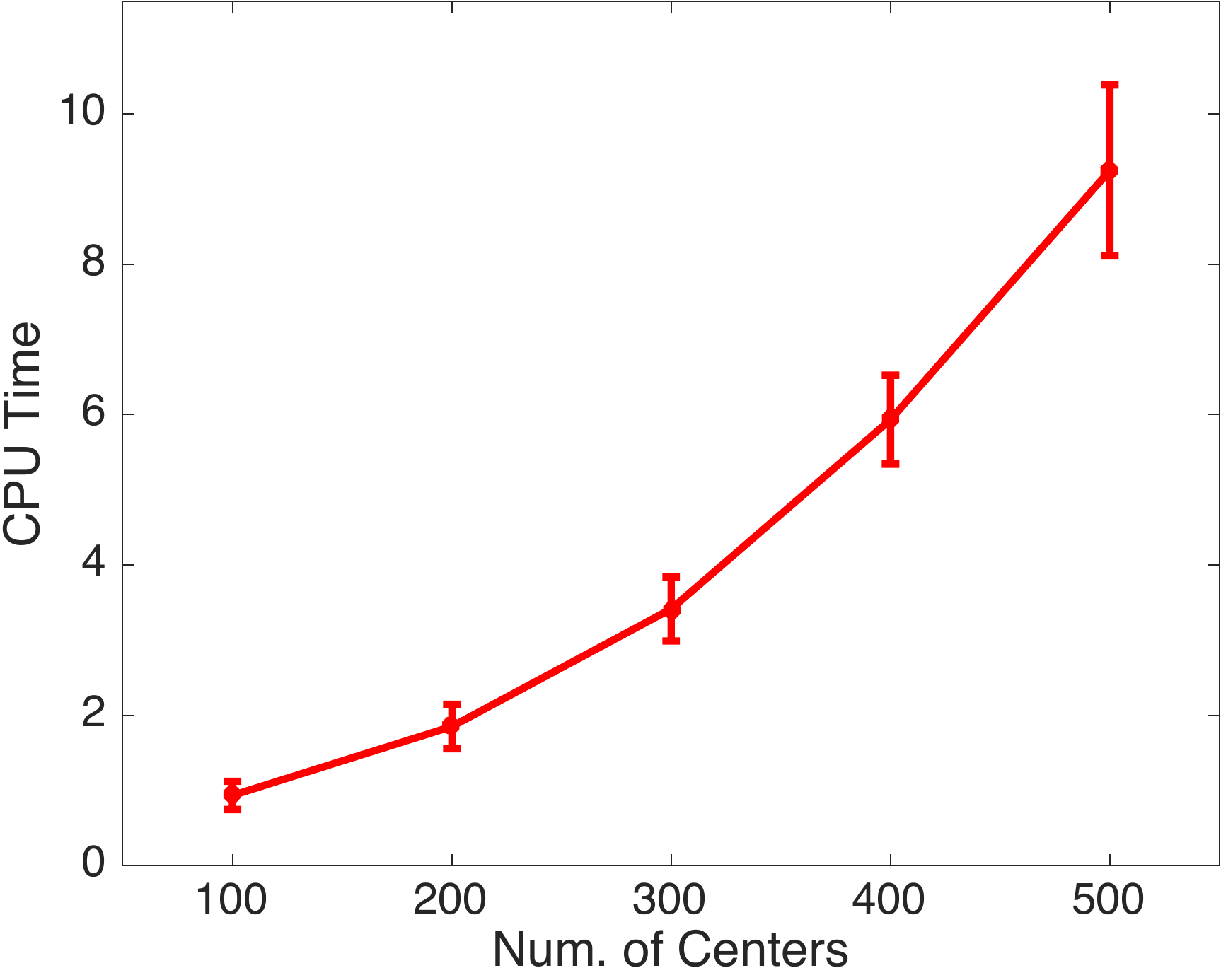}}
   \caption{\label{fig:kernelcenter} Clustering performance and CPU time
   against the number of kernel centers. Each point and error bar denote
   the average and standard deviation of (a) ARI and (b) CPU time over
   $50$ runs, respectively. The dataset used in this figure is the three
   Gaussian blobs in Section~\ref{ssec:Illust-Clust} when
   $(D,n)=(5,500)$}
  \end{center}
 \end{figure}
 This appendix investigates clustering performance and computational
 costs of LSLDGC when the number of kernel centers is changed. We
 performed similar experiments in Section~\ref{ssec:Illust-Clust}. In
 the experiments, datasets with the three Gaussian blobs
 (Fig.\ref{fig:ClustArt}(g)) were used.

 Fig.\ref{fig:kernelcenter} shows that LSLDGC with a small number of
 kernel centers significantly reduces the computation costs without
 scarifying the clustering performance.
 \section{Proof of Lemma~\ref{Epsiconv}}
 \label{app:Epsiconv}
 \begin{proof}
  For $\epsilon^{\prime}$, the Cauchy-Schwarz inequality gives
  \begin{align*}
   |\widehat{g}_j(\bm{x})-g_j(\bm{x})| =|\langle \widehat{g}_j-g_j,
   k(\cdot, \bm{x})\rangle_{\rkhs}| \leq \|\widehat{g}_j-g_j\|_{\rkhs}
   |k(\bm{x},\bm{x})|.
  \end{align*}
  Since $k(\bm{x},\bm{x})$ is assumed to be finite,
 \begin{align}
  \epsilon^{\prime}= \max_{j}
  \|\widehat{g}_j(\bm{x})-g_j(\bm{x})\|_{\infty} \leq
  O\left(\|\widehat{g}_j-g_j\|_{\rkhs}\right)
  &=O_{\mathrm{P}}\left(n^{-\min\left\{\frac{1}{4},
  \frac{\gamma}{2(\gamma+1)}\right\}}\right), \label{TMPepsi}
 \end{align}
 where we applied Theorem~\ref{theo:Hconvrate}.
 
 For $\epsilon^{\prime\prime}$, similar computation yields
 \begin{align*}
  &|[\widehat{\Sig}^{-1}(\bm{x})]_{ij}-[\Sig^{-1}(\bm{x})]_{ij}|\\
  &=|-(\widehat{g}_i(\bm{x})\widehat{g}_j(\bm{x})-g_i(\bm{x})g_j(\bm{x}))
  +[\widehat{\bm{H}}(\bm{x})]_{ij}-[{\bm{H}}(\bm{x})]_{ij}|\\
  &\leq
  |\widehat{g}_i(\bm{x})\widehat{g}_j(\bm{x})-g_i(\bm{x})g_j(\bm{x})|
  +|[\widehat{\bm{H}}(\bm{x})]_{ij}-[{\bm{H}}(\bm{x})]_{ij}|\\
  &\leq
  |\widehat{g}_i(\bm{x})|\cdot |\widehat{g}_j(\bm{x})-g_j(\bm{x})|
  +|g_j(\bm{x})|\cdot |\widehat{g}_i(\bm{x})-g_i(\bm{x})|
  +|[\widehat{\bm{H}}(\bm{x})]_{ij}-[{\bm{H}}(\bm{x})]_{ij}|\\
  &\leq \left\{|g_i(\bm{x})|\cdot \|\widehat{g}_j-{g}_j\|_{\rkhs}
  +|\widehat{g}_j(\bm{x})|\cdot \|\widehat{g}_i-{g}_i\|_{\rkhs}
  +\|[\widehat{\bm{H}}]_{ij}-[{\bm{H}}]_{ij}\|_{\rkhs}\right\}
  |k(\bm{x},\bm{x})|,   
 \end{align*}
  where we applied the following inequality on the fourth line:
  \begin{align*}
   |\widehat{g}_i(\bm{x})\widehat{g}_j(\bm{x})-g_i(\bm{x})g_j(\bm{x})|
   &=|\widehat{g}_i(\bm{x})\widehat{g}_j(\bm{x})
   -\widehat{g}_i(\bm{x})g_j(\bm{x})
   +\widehat{g}_i(\bm{x})g_j(\bm{x})-g_i(\bm{x})g_j(\bm{x})|\\
   &=|\widehat{g}_i(\bm{x})(\widehat{g}_j(\bm{x})-g_j(\bm{x}))
   +g_j(\bm{x})(\widehat{g}_i(\bm{x})-g_i(\bm{x}))|\\
   &\leq |\widehat{g}_i(\bm{x})|\cdot |\widehat{g}_j(\bm{x})-g_j(\bm{x})|
   +|g_j(\bm{x})|\cdot |\widehat{g}_i(\bm{x})-g_i(\bm{x})|.
  \end{align*}
 Thus, we obtain
 \begin{align}
  \epsilon^{\prime\prime}&=\max_{ij}\max_{\bm{x}}
  |[\widehat{\Sig}^{-1}(\bm{x})]_{ij}-[\Sig^{-1}(\bm{x})]_{ij}|
  \nonumber\\
  &\leq \max_{ij} O\left(\max(\|\widehat{g}_j-{g}_j\|_{\rkhs},
  \|\widehat{g}_i-{g}_i\|_{\rkhs},
  \|[\widehat{\bm{H}}]_{ij}-[{\bm{H}}]_{ij}\|_{\rkhs})\right)\nonumber\\
  &=O_{\mathrm{P}}\left(n^{-\min\left\{\frac{1}{4},
  \frac{\gamma}{2(\gamma+1)}\right\}}\right),
  \label{TMPepsi'}
 \end{align}
 where it follows from Theorem~\ref{theo:Hconvrate}.
 
 For $\epsilon^{\prime\prime\prime}$, we resort to the derivative
 reproducing property proved in~\citet{zhou2008derivative}: For all
 $f\in\rkhs$,
 \begin{align*}    
  \partial_j f(\bm{x})=\<f, \partial_j k(\cdot,\bm{x})\>_{\rkhs}.
 \end{align*}
 Using this relation, we obtain
 \begin{align*}
  |[\widehat{\Sig}^{-1\prime}(\bm{x})]_{ij}-[\Sig^{-1\prime}(\bm{x})]_{ij}|
  &=|\partial_j [\vecope(\widehat{\Sig}^{-1}(\bm{x}))]_{i} -\partial_j
  [\vecope(\Sig^{-1}(\bm{x}))]_{i}|\\ &= |\langle
  [\vecope(\widehat{\Sig}^{-1})]_{i} -[\vecope(\Sig^{-1})]_{i},
  \partial_j k(\cdot, \bm{x})\rangle_{\rkhs}|\\ &\leq
  \|\vecope(\widehat{\Sig}^{-1})]_{i}
  -[\vecope(\Sig^{-1})]_{i}\|_{\rkhs}
  |\partial_i^{\prime}\partial_ik(\bm{x}^{\prime},\bm{x})|_{\bm{x}^{\prime}=\bm{x}}|,
 \end{align*}
 where $\partial_j^{\prime}$ denote the derivative with respect to the
 $j$-th coordinate in $\bm{x}^{\prime}$.  Since
 $|\partial_j^{\prime}\partial_j k|$ is assumed to be finite,
  \eqref{TMPepsi'} provides
 \begin{align*}
  \epsilon^{\prime\prime\prime} &=\max_{ij}\max_{\bm{x}}
 |[\widehat{\Sig}^{-1\prime}(\bm{x})]_{ij}-[\Sig^{-1\prime}(\bm{x})]_{ij}|\\
 &\leq\max_{i}O\left(\|[\vecope(\widehat{\Sig}^{-1})]_{i}
  -[\vecope(\Sig^{-1})]_{i}\|_{\rkhs}\right)\\
  &=O_{\mathrm{P}}\left(n^{-\min\left\{\frac{1}{4},
  \frac{\gamma}{2(\gamma+1)}\right\}}\right), 
 \end{align*}
  where the last equation comes from~\eqref{TMPepsi'} because
  $[\vecope(\Sig^{-1})(\bm{x})]_{i}$ denotes a single element in
  $\Sig^{-1}(\bm{x})$.
 \end{proof}
 \bibliographystyle{abbrvnat} 
 \bibliography{../../../papers}
\end{document}